\documentclass[twoside,11pt]{article}

\usepackage{blindtext}

\usepackage{mymacros}
\allowdisplaybreaks[4] 
\usepackage{parskip} 
\usepackage[ruled,linesnumbered]{algorithm2e}
\SetKwComment{Comment}{/* }{ */} 
\let\oldnl\nl
\newcommand{\nonl}{\renewcommand{\nl}{\let\nl\oldnl}}
\usepackage{subcaption}  

\usepackage{enumitem}

\usepackage[preprint]{jmlr2e}

\usepackage{lastpage}
\jmlrheading{XY}{2025}{1-\pageref{LastPage}}{1/25; Revised
-/-}{-/-}{21-0000}{Filippo Lazzati, Mirco Mutti, Alberto Maria Metelli}

\ShortHeadings{A Framework for Inverse RL}{Reward Compatibility}
\firstpageno{1}

\begin{document}

\setlength{\abovedisplayskip}{5.5pt}
\setlength{\belowdisplayskip}{5.5pt}
\setlength{\textfloatsep}{14pt}

\title{Reward Compatibility: A Framework for Inverse RL}

\author{\name Filippo Lazzati \email filippo.lazzati@polimi.it \\
       \addr Politecnico di Milano\\
       Milan, Italy
       \AND
       \name Mirco Mutti \email mirco.m@technion.ac.il \\
       \addr Technion\\
        Haifa, Israel
       \AND
       \name Alberto Maria Metelli \email albertomaria.metelli@polimi.it \\
       \addr Politecnico di Milano\\
       Milan, Italy}

\editor{My Editor}

\maketitle

\begin{abstract}
    We provide an original theoretical study of Inverse Reinforcement Learning
    (IRL) through the lens of \emph{reward compatibility}, a novel framework to
    quantify the compatibility of a reward with the given expert's
    demonstrations.
    Intuitively, a reward is more \emph{compatible} with the demonstrations the
    closer the performance of the expert's policy computed with that reward is
    to the optimal performance for that reward.
    This generalizes the notion of \emph{feasible reward set}, the most common
    framework in the theoretical IRL literature, for which a reward is either
    compatible or not compatible. The \emph{grayscale} introduced by the reward
    compatibility is the key to extend the realm of provably efficient IRL far
    beyond what is attainable with the feasible reward set: from tabular to
    \emph{large-scale} MDPs.
    We analyze the IRL problem across various settings, including optimal and
    \emph{suboptimal} expert's demonstrations and both online and \emph{offline}
    data collection.
    For all of these dimensions, we provide a tractable algorithm and
    corresponding sample complexity analysis, as well as various insights on
    reward compatibility
    and how the framework can pave the way to yet more general problem settings.
\end{abstract}

\begin{keywords}
  Inverse Reinforcement Learning, Linear MDPs, Sample Complexity, Reward-Free
  Exploration, Identifiability
\end{keywords}

\section{Introduction}

Inverse Reinforcement Learning (IRL) is the problem of inferring the reward
function of an agent, named the expert agent, from demonstrations of behavior
\citep{russell1998learning,ng2000algorithms}. Since its formulation, much
research effort has been put into the design of efficient algorithms for solving
the IRL problem \citep{arora2018survey,adams2022survey}, with the promise to
open the door to a variety of interesting applications, including Apprenticeship
Learning \citep[AL,][]{abbeel2004apprenticeship,abbeel2006helicopter}, reward
design \citep{hadfieldmenell2017inverserewarddesign}, interpretability of the
expert's behavior \citep{hadfieldmenell2016cooperativeIRL}, and transferability
of behavior to new environments \citep{Fu2017LearningRR}.

Despite the relevance of these applications and the abundance of practical
solutions, IRL has long escaped a formal and coherent theoretical
characterization to support the empirical research. Indeed, the \emph{vanilla}
formulation of IRL, coarsely ``given a set of demonstrations from an expert's
policy recover the reward function the expert is maximizing'', is notoriously
ill-posed \citep{ng2000algorithms,metelli2021provably,metelli2023towards}, as
several rewards (infinitely many) are \emph{feasible} to explain the evidence
provided by expert's demonstrations, no matter the number of available
demonstrations.
To resolve this ambiguity, it has been proposed to consider \emph{additional}
demonstrations from multiple environments
\citep{amin2016resolving,cao2021identifiability} or from multiple experts
\citep{rolland2022identifiabilitymultipleexperts,poiani2024inverse}, or to
consider \emph{additional} human feedback
\citep{jeon2020rewardrational,skalse2023invariance}. However, this extra
information is not always available.

Only recently, the notion of \emph{feasible reward set}, introduced by
\cite{metelli2021provably}, has convincingly sorted out the issue, emerging as a
common theoretical framework for the study of IRL~\citep{lindner2022active,
metelli2023towards, lazzati2024offline, zhao2023inverse, poiani2024inverse} and
related problems
\citep{lazzati2024learningutilitiesdemonstrationsmarkov,yue2024provablyefficientexplorationinverse,freihaut2024multiagentinversereinforcementlearning}.
In the absence of conclusive information coming from expert's demonstrations,
\cite{metelli2021provably} propose to extract not just one feasible reward, but
the set of all of the rewards that make the expert's policy optimal, called the
\emph{feasible reward set}. This formulation lead to the development of provably
efficient IRL algorithms and clarified the statistical barriers of the
problem~\citep{metelli2023towards}. However, the application of the feasible
reward set framework is essentially limited to settings where the space of
states and actions is finite and small, falling into the \emph{tabular} Markov
Decision Process (MDP) formalism. As we shall see, this constraint is inherent
to the feasible set and cannot be overcome.

Unfortunately, important potential applications of IRL do not comply with this
property. Think about the problem of extracting a reward from demonstrations
collected by an expert driver~\citep{ziebart2008maximum} or an helicopter
control policy~\citep{abbeel2006helicopter}. These settings typically involve
large or continuous \emph{state spaces}, thus invalidating the tabular
representation of the space. Previous works have studied IRL beyond tabular
MDPs~\citep{michini2013scalable,finn2016guided,Fu2017LearningRR,barnes2024massively},
but none of them from the theoretical viewpoint of the feasible set.

From these considerations, the need for a unifying framework that incorporates
the important traits of the feasible reward set on the one hand, i.e., provides
a formal and coherent theoretical characterization of the sample complexity of
IRL, yet extends the notion to MDPs with large state spaces naturally arises.

In order to be useful and general, such framework should be easily applicable
not only to the common IRL problem of \citet{ng2000algorithms}, where $(i)$ the
expert's demonstrations are collected with an \emph{optimal} policy, and $(ii)$
the IRL algorithm can \emph{actively explore} the environment to collect the
data it needs, but also to other settings in which these properties do not hold.
In fact, note that in all the settings in which
demonstrations are provided by humans, property $(i)$ may also be violated.
Humans are characterized by a bounded rationality and may not be able to
optimally solve the demonstrated task. This is why a more realistic setting, in
which demonstrations of behavior come from suboptimal experts, has been widely
studied in the
literature~\citep{ziebart2008maximum,jing2019reinforcementlearningimperfectdemonstrations,kurenkov2019acteachbayesianactorcriticmethod,cheng2020policyimprovementimitationmultiple,liu2023activepolicyimprovementmultiple,poiani2024inverse}.
Moreover, while letting the IRL algorithm actively interact with the environment
helps improve its performance in many cases (e.g., in simulation
problems~\citealt{neu2007apprenticeship}), it may not be possible in settings
where safety is critical. Think about the possibility of damaging the helicopter
in~\citep{abbeel2006helicopter} or harm other people in an autonomous driving
scenario~\citep{arora2018survey}. This is why previous works considered access
to a batch of demonstrations and no interaction with the
environment~\citep{jarboui2021offlineinversereinforcementlearning,zhao2023inverse,lazzati2024offline}
invalidating property $(ii)$.

In this paper, we build on these premises as follows.

\textbf{Contributions.}~~We propose the \emph{reward compatibility} framework as a
unifying learning framework for the development of efficient
algorithms for IRL. The main contributions of the current work are summarized
as follows:
\begin{itemize}[noitemsep, topsep=0pt]
  \item We demonstrate that, without additional assumptions, the notion of
  feasible set can \emph{not} be learned efficiently in Markov Decision
  Processes (MDPs) when the state space is large/continuous, even under the
  structure enforced by Linear MDPs~\citep{jin2020provablyefficient,yang2019sampleoptimal}, in which we assume 
  that the reward function and the transition model can be expressed as linear 
  combinations of known features. This fact entails that the learning
  framework of the feasible reward set is not powerful enough to manage all the
  meaningful IRL problem settings (Section \ref{sec: limitations fs}).
  \item As a unifying learning framework for IRL, we propose \emph{Reward
  Compatibility}, a novel scheme that formalizes the intuitive notion of
  compatibility of a reward function with expert demonstrations. It generalizes
  the feasible set and allows us to define an original problem, \emph{IRL
  classification}. We devise a provably-efficient algorithm, \caty (\catylong),
  for solving the IRL classification problem in both the online settings with
  optimal and suboptimal expert, and in both tabular and Linear MDPs (see
  Section \ref{sec: the rewards compatibility framework} and Table \ref{table:
  problem settings}).
  \item In addition, we focus on the offline IRL setting in tabular MDPs. We
  show that, given the partial coverage of the space induced by the batch
  dataset, the offline setting requires suitable ``robust'' notions of reward
  compatibility, that we provide. Then, we develop a provably-efficient
  algorithm, \catyoff (\catyofflong), for solving the problem in both the
  optimal and suboptimal expert settings (see Section \ref{sec: offline} and
  Table \ref{table: problem settings}).
  \item We conclude with a discussion on reward compatibility that
  ranges from Reward Learning (ReL) \citep{jeon2020rewardrational}, to other IRL
  problem settings (e.g., \citep{ziebart2008maximum}), to the practical usage of
  learned reward functions, and we identify a variety of interesting research
  directions for future works.
\end{itemize}

This paper unifies and extends the previous conference papers
\citep{lazzati2024offline,lazzati2024scaleinverserllarge}. The former
demonstrates the need of suitable ``robust'' notions of feasible set in the
offline setting, and studies them. Instead, the latter demonstrates the
limitations of the feasible set in Linear MDPs, formalizes the reward
compatibility framework, and presents a version of \caty for the setting with
optimal expert. In this paper, we extend the contributions of
\citep{lazzati2024offline,lazzati2024scaleinverserllarge} as follows:
\begin{itemize}[noitemsep, topsep=0pt]
    \item We extend the results on the limitations of the feasible set in
    \citep{lazzati2024scaleinverserllarge} to the setting with suboptimal expert
    (Theorem \ref{thr: fs suboptimal not learnable}), and to other function
    approximation settings beyond Linear MDPs (Theorem \ref{thr: fs not
    learnable linear mixture} and a discussion in Section \ref{sec: limitations
    fs}).
    \item We extend the reward compatibility framework and the \caty algorithm
    in \citep{lazzati2024scaleinverserllarge} to the setting with suboptimal
    experts, and we provide theoretical guarantees of sample complexity for the
    algorithm (Theorem \ref{thr: upper bound caty online}, the claim concerning
    suboptimal experts). Moreover, we provide sufficient conditions to
    extend \caty to any function approximation setting beyond tabular and
    linear MDPs (see Section \ref{subsec: algorithm caty online}).
    \item We demonstrate the need of suitable ``robust'' notions of reward
    compatibility if we want to extend the framework to the offline setting
    (Theorem \ref{thr: comp non learnable offline}), and we devise them
    (Definition \ref{def: best worst comp}), similarly to what derived in
    \citep{lazzati2024offline} for the feasible set.
    \item We extend \caty to both the offline settings with optimal and
    suboptimal expert in tabular MDPs, and we provide theoretical guarantees of
    sample complexity (Theorem \ref{thr: bounds tabular off}), in a setting
    analogous to that analyzed in \citep{lazzati2024offline}. 
\end{itemize}

\begin{table}[t!]
  \centering
  \begin{tabular}{||c || c c||} 
    \hline
     & Tabular MDPs & Linear MDPs\\ [0.5ex] 
    \hline\hline
    Online (\caty) & $\widetilde{\mathcal{O}}\Big(\frac{H^3SA}{\epsilon^2}
    \Bigr{S+ \log \frac{1}{\delta}}\Big)$ & $\widetilde{\mathcal{O}}\Big(\frac{H^5d}{\epsilon^2}
    \Big(d+\log\frac{1}{\delta}\Big)\Big)$\\ 
    Offline (\catyoff) & $\widetilde{\mathcal{O}}\Big(\frac{H^4\log\frac{1}
    {\delta}}{\epsilon^2 d_{\min}}\Bigr{S+\log\frac{1}
    {\delta}}\Big)$ & /\\
    \hline
    \end{tabular}
  \caption{ In this table, we summarize the theoretical guarantees of sample
  complexity concerning the estimation of the transition model for the proposed
  algorithms.
  Interestingly, our algorithms enjoy the \emph{same rates in both the optimal
  and suboptimal expert settings}. In tabular MDPs, $S$ denotes the size of the
  state space, $A$ that of the action space, $H$ is the horizon, $\epsilon$ the
  accuracy, and $\delta$ the failure probability. $d$ represents the feature
  dimension in Linear MDPs, while $d_{\min}$ keeps into account the coverage of
  the environment provided by the batch dataset.
  We leave for future works the development of efficient algorithms for Linear
  MDPs in the offline setting. Indeed, the notion of compatibility that we
  adopted for the offline tabular setting is based on a definition of coverage
  of the space that cannot be applied to Linear MDPs. For this reason, a new
  notion of reward compatibility that exploits the existing definitions of
  coverage for Linear MDPs \citep{Wang2020WhatAT,jin2021ispessimism} should be
  developed, falling outside the scope of this paper. }
  \label{table: problem settings}
  \end{table}

\section{Preliminaries}\label{sec: preliminaries}

\textbf{Notation.}~~Given an integer $N \in \Nat$, we define
$\dsb{N}\coloneqq\{1,\dotsc,N\}$. We denote by $\Delta^\cX$ the probability
simplex over $\cX$, and by $\Delta_\cY^\cX$ the set of functions from $\cY$ to
$\Delta^\cX$. Sometimes, we denote the dot product between vectors $x,y$ as
$\dotp{x,y}\coloneqq x^\intercal y$. We employ $\mathcal{O},\Omega,\Theta$ for
the common asymptotic notation and
$\widetilde{\mathcal{O}},\widetilde{\Omega},\widetilde{\Theta}$ to omit
logarithmic terms. Given an equivalence relation $\equiv\subseteq\cX\times\cX$
and an item $x\in\cX$, we denote by $[x]_{\equiv}$ the equivalence class of $x$.

\textbf{Markov Decision Processes.}~~A finite-horizon Markov Decision Process
(MDP) without reward \citep{puterman1994markov} is defined as a tuple
$\cM\coloneqq\tuple{\cS,\cA,H, d_0,p}$, where $\cS$ and $\cA$ are the measurable
state and action spaces, $H \in \Nat$ is the horizon, $d_0\in\Delta^\cS$ is the
initial-state distribution,\footnote{In most results, w.l.o.g. we will consider
a deterministic initial state distribution $d_0(s_0)=1$.} and
$p\in\cP\coloneqq\Delta_{\SAH}^\cS$ is the transition model. Given a
(deterministic) reward function $r\in\mathfrak{R}\coloneqq[-1,1]^{\SAH}$, we
denote by $\overline{\cM}\coloneqq\cM\cup\{r\}$ the MDP obtained by pairing
$\cM$ and $r$.
Each policy $\pi\in\Pi\coloneqq\Delta_{\SH}^\cA$ induces in $\overline{\cM}$ a
state-action probability distribution
$d^{p,\pi}\coloneqq\{d^{p,\pi}_h\}_{h\in\dsb{H}}$ (we omit $d_0$ for simplicity)
that assigns, to each subset $\cZ\subseteq\cS\times\cA$, the probability of
being in $\cZ$ at stage $h \in \dsb{H}$ when playing $\pi$ in $\overline{\cM}$.
We denote with $\cS^{p,\pi}_h$ (resp. $\cZ^{p,\pi}_h$) the set of states (resp.
state-action pairs) supported by $d^{p,\pi}_h$ at stage $h$, and with
$\cS^{p,\pi}$ (resp. $\cZ^{p,\pi}$) the disjoint union of sets
$\{\cS^{p,\pi}_h\}_{h \in \dsb{H}}$ (resp. $\{\cZ^{p,\pi}_h\}_{h \in \dsb{H}}$).
For distribution $d^{p,\pi}$, we define
$d_{\min}^{p,\pi}\coloneqq\min_{(s,a,h)\in\cZ^{p,\pi}}d^{p,\pi}_h(s,a)$.
The $Q$-function of policy $\pi$ in MDP $\overline{\cM}$ is defined at every
$(s,a,h) \in \cS \times \cA \times \dsb{H}$ as $Q^{\pi}_h(s,a;p,r)\coloneqq
\E_{p,\pi}[\sum_{t=h}^H r_{t}(s_t,a_t)|s_h=s,a_h=a]$, and the optimal
$Q$-function as $Q^*_h(s,a;p,r)\coloneqq\sup_{\pi \in \Pi} Q^{\pi}_h(s,a;p,r)$,
where the expectation $\E_{p,\pi}$ is computed over the stochastic process
generated by playing policy $\pi$ in the MDP $\overline{\cM}$. Similarly, we
define the $V$-function of policy $\pi$ at $(s,h)$ as $V^\pi_h(s;p,r)\coloneqq
\E_{p,\pi}[\sum_{t=h}^H r_{t}(s_t,a_t)|s_h=s]$, and the optimal $V$-function as
$V^*_h(s;p,r)\coloneqq\sup_{\pi \in \Pi} V^{\pi}_h(s;p,r)$. We define the
utility of $\pi$ as $J^\pi(r;p)\coloneqq \E_{s\sim
d_0}[V^\pi_1(s;p,r)]=\sum_{(s,a,h)\in\SAH}d_h^{p,\pi}(s,a)r_h(s,a)$, and the
optimal utility as $J^*(r;p)\coloneqq \E_{s\sim d_0}[V^*_1(s;p,r)]$.
A forward (sampling) model\footnote{See \citet{kakade2003onthesample} or
\citet{azar2013minimax}.} of the environment permits to collect samples starting
from $s\sim d_0$ and following some policy. A generative (sampling) model
consists in an oracle that, given an arbitrary state-stage pair $s,h$ (resp.
state-action-stage triple $s,a,h$) in input, returns a sampled action $a'\sim
\pi_h(\cdot|s)$ (resp. a sampled next state $s'\sim p_h(\cdot|s,a)$).
For simplicity, when $\cS,\cA$ are finite, given a set
$\overline{\cZ}\subseteq\SAH$, we define the equivalence relation
$\equiv_{\overline{\cZ}}\subseteq\cP\times\cP$ such that, for any pair
$p,p'\in\mathcal{P}$ of transition models, we have $p\equiv_{\overline{\cZ}}p'$
if and only if $\forall (s,a,h)\in \overline{\cZ}:\;
p_h(\cdot|s,a)=p_h(\cdot|s,a)$.

\textbf{BPI and RFE.}~~In both Best-Policy Identification (BPI) \citep{menard2021fast} and
Reward-Free Exploration (RFE) \citep{jin2020RFE}, the learner has to explore the \emph{unknown} MDP to
optimize a certain reward function. In BPI, the learner observes the reward
function $r$ during exploration, and its goal is to output a policy
$\widehat{\pi}$ such that, in the true MDP with transition model $p$, we have 
$\mathbb{P}\big(J^*(r;p)-J^{\widehat{\pi}}(r;p)\le\epsilon\big)\ge 1-\delta$ for every
$\epsilon,\delta\in(0,1)$.
RFE considers the setting in which the reward  is revealed \emph{a
posteriori} of the exploration phase. Thus,
the goal of the agent in RFE is to compute an estimate $\widehat{p}$ of the true dynamics $p$
so that $\mathbb{P}\big(\sup_{r\in\mathfrak{R}} (J^*(r;p)-
J^{\widehat{\pi}_r}(r;p))\le\epsilon\big)\ge 1-\delta$ for every $\epsilon,\delta\in(0,1)$, where $\widehat{\pi}_r$
is the optimal policy in the MDP with $\widehat{p}$ as transition model and $r$
as reward function.

\textbf{IRL.}~~In the most common IRL setting, we are given an MDP without
reward $\cM\coloneqq\tuple{\cS,\cA,H, d_0,p}$ ($p$ may be known or unknown) and
a dataset of $n$ state-action trajectories
$\cD=\{\tuple{s_1^i,a_1^i,\dotsc,s_H^i,a_H^i,s_{H+1}^i}\}_{i\in\dsb{n}}$
obtained by executing $n$ times the expert's policy $\pi^E$ in $\cM$. The
underlying assumption is that there exists a true expert's reward $r^E$ that
``someway'' relates to $\pi^E$, and the IRL objective is to use the knowledge of
$\cM$ and the dataset $\cD$ to find a reward $\widehat{r}\approx r^E$ that
represents a good approximation for the unknown $r^E$ according to some
performance index. In literature, different works have analysed different kinds
of relationships between $r^E$ and $\pi^E$. To make some examples,
\citet{ng2000algorithms} assume that $\pi^E$ is optimal under $r^E$,
\citet{ziebart2010MCE} assume that $\pi^E$ is the maximum causal entropy policy
for $r^E$, \citet{Ramachandran2007birl} assume that $\pi^E$ is Boltzmann
rational for $r^E$, and \citet{poiani2024inverse} assume that the suboptimality
of $\pi^E$ under $r^E$ is no more than some given threshold $\xi$. We remark
that different assumptions give birth to different \emph{solution concepts} for
IRL, i.e., given the same expert's policy $\pi^E$, these works aim to estimate
different reward functions. The IRL solution concept of the \emph{feasible set}
has been formulated more recently
\citep{metelli2021provably,metelli2023towards}.

\section{Limitations of the Feasible Set}\label{sec: limitations fs} 
 
In this section, we show that the notion of \emph{feasible set}, an important
solution concept for IRL, unfortunately, cannot be estimated efficiently in MDPs
with a large state space, even under the structure provided by certain
function approximation settings, particularly Linear MDPs.
This limitation, along with some implementability issues, urges the introduction
of a more powerful solution concept for IRL, that overcomes these limitations
while preserving the desirable properties of the feasible set.
Let us begin by formalizing the notion of feasible set in two important IRL
settings.

\subsection{The feasible set}
We consider the IRL settings where we are given demonstrations collected by the
expert's policy $\pi^E$, and we assume that $\pi^E$ is optimal
\citep{ng2000algorithms} or suboptimal \citep{poiani2024inverse} under the true,
unknown, expert's reward $r^E$. In both cases, even the exact knowledge of
$\pi^E$ is not sufficient for uniquely identifying $r^E$, because many other
reward functions satisfy the same constraint of making $\pi^E$ (sub)optimal.
The \emph{feasible set} is defined as the set of rewards that comply with such
constraint \citep{metelli2021provably,metelli2023towards}. Intuitively, we do
not know the ``true reason'' ($r^E$) why the expert plays $\pi^E$ because we do
not have enough information, but we know that if it plays $\pi^E$ then its
``goal'' is one of those in the feasible set.          
Thus, simply put, the feasible set represents the entire set of objectives that
the expert may be optimizing. Formally, with optimal
expert:\footnote{To be precise, previous works
\citep{metelli2021provably,metelli2023towards,poiani2024inverse}  
adopt more restrictive definitions of feasible set, that assume the existence of
the expert's policy at all $(s,a,h)$ of the state-action space. Our definitions
relax such requirement \citep{lazzati2024offline}.}
\begin{defi}[Feasible Set for Optimal Expert]\label{def: fs optimal}
    
    Let $\cM=\tuple{\cS,\cA,H,d_0,p}$ be an MDP without reward and let $\pie$ be
    the expert's policy, that is optimal for the true unknown reward $r^E$. The
    \emph{feasible set} $\fs$ of rewards compatible with $\pie$ in
    $\cM$ is defined as:
    \begin{equation}\label{eq: fs optimal}
        \fs\coloneqq 
        \Bigc{r\in\mathfrak{R}\,\big|\,J^{\pie}(r;p)=J^*(r;p)}.
\end{equation}
\end{defi}
Simply put, we are told that the true reward $r^E$ satisfies the constraint
$J^{\pie}(r^E;p)=J^*(r^E;p)$, thus, all the rewards in $\fs$ represent
equally-plausible candidates for $r^E$.

Analogously, we can define the feasible set for the setting in which the expert
is suboptimal with a suboptimality contained in the (known) range $[L,U]$, for
some $U\ge L\ge 0$. Note that this setting generalizes the setting in
\citet{poiani2024inverse}.
\begin{defi}[Feasible Set for Subptimal Expert]\label{def: fs suboptimal}
    
    Let $U\ge L\ge 0$. Let $\cM=\tuple{\cS,\cA,H,d_0,p}$ be an MDP without
    reward and let $\pie$ be the expert's policy, with suboptimality in $[L,U]$
    under the true unknown reward $r^E$. The \emph{feasible set} $\fs^{L,U}$
    of rewards compatible with $\pie$ in $\cM$ is defined as:
    \begin{equation}\label{eq: fs suboptimal}
        \fs^{L,U}\coloneqq\Bigc{r\in\fR\,\big|\,J^*(r;p)-J^{\pi^E}(r;p)\in[L,U]}.
\end{equation}
\end{defi}
Again, the notion of feasible set represents the set of rewards containing the
true reward $r^E$. Note that $\fs^{0,0}=\fs$, and that $\fs^{0,U}\supseteq\fs$
for any $U\ge 0$. Moreover, observe that the setting with $L=0,U>0$ coincides
with the setting of \citet{poiani2024inverse}.

\subsection{Learning the Feasible Set}

In practical applications, the dynamics $(d_0,p)$ and the expert's policy $\pi^E$
are not known, but have to be estimated from samples.
%
%
In tabular MDPs, in the simple setting with a generative model for both $p$ and
$\pi^E$, we have algorithms for learning the feasible set that require a number
of calls to the generative model (i.e., sample complexity) that grows at most
\emph{quadratically} in the size of the state space. For instance, algorithm
US-IRL of \citet{metelli2023towards} can compute accurate estimates of (a
variant of) the feasible set in Definition \ref{def: fs optimal} with high
probability (w.h.p.), while, for the setting with a suboptimal expert, algorithm
US-IRL-SE of \citet{poiani2024inverse} can compute accurate estimates of (a
variant of) the feasible set in Definition \ref{def: fs suboptimal} w.h.p..

These algorithms are efficient as long as the cardinality of the state space is
reasonably small, but when the state space is large, these algorithms can be
very sample \emph{inefficient} because of the quadratic dependence.
Unfortunately, without any structural assumption on the MDP, no efficient
algorithm can be developed when the state space is large, even using the simple
generative model, due to the lower bounds to the sample complexity presented in
\citet{metelli2023towards} (optimal expert) and \citet{poiani2024inverse}
(suboptimal expert), that grow directly with the size of the state
space.\footnote{Again, to be precise, the lower bounds address a definition of
feasible set that is slightly different than Definition \ref{def: fs optimal}
and \ref{def: fs suboptimal}.}

For this reason, similarly to what is done in (forward) RL in environments with
large or even continuous state spaces, we make the simple but strong structural
assumption that the environment is a Linear MDP, and we try to design sample
efficient algorithms. Specifically, we assume the existence of a feature mapping
$\phi:\SA\to\RR^d$, with $\|\phi(s,a)\|_2\le1$ for every $ (s,a) \in \mathcal{S
\times A}$, such that the considered MDP without reward
$\cM=\tuple{\cS,\cA,H,d_0,p}$ has a dynamics that is linear in $\phi$, i.e.,
$p_h(\cdot|s,a)=\dotp{\phi(s,a),\mu_h(\cdot)}$ for all $h\in\dsb{H}$, where
$\mu_h=[\mu_h^1,\dotsc,\mu_h^d]^\intercal$ is a vector of $d$ signed measures
such that $\||\mu_h|(\cS)\|_2\le\sqrt{d}$. Linear MDPs also assume linear
rewards. Thus, we consider the true expert's reward $r^E$ to be linear in the
feature mapping $r_h(s,a)=\dotp{\phi(s,a),\theta_h}$ for all $h\in\dsb{H}$, for
some parameters $\theta_h$ s.t. $\|\theta_h\|_2\le \sqrt{d}$. For simplicity,
let us define the set of rewards that are linear in $\phi$ as $\fR_\phi$:
\begin{align*}
    \scalebox{0.98}{$  \displaystyle \fR_\phi\coloneqq\Bigc{
        r\in\mathfrak{R}\,\big|\,
        \forall h\in\dsb{H}\,
        \exists \theta_h\in\RR^d:\,\|\theta_h\|_2\le \sqrt{d}\,\wedge\,\forall (s,a)\in\SA:\,r_h(s,a)=
        \dotp{\phi(s,a),\theta_h}
        }.$}
\end{align*}
We denote a Linear
MDP without reward as $\cM_\phi=\tuple{\cS,\cA,H,d_0,p,\phi}$.

Since the Linear MDP assumption gives additional information on $r^E$, then our
definitions of feasible set change. In particular, they contain only rewards
that are linear in $\phi$:
\begin{defi}[Feasible Set for Optimal Expert in Linear MDPs]\label{def: fs
optimal linear mdps}
    Let $\cM_\phi=\tuple{\cS,\cA,H,d_0,p,\phi}$ be a Linear MDP without reward
    and let $\pie$ be the expert's policy, that is optimal for the true
    unknown reward $r^E$. The \emph{feasible set} $\fsl$ of rewards
    compatible with $\pie$ in $\cM_\phi$ is defined as:
    \begin{align}\label{eq: fs optimal linear mdps}
        \fsl\coloneqq 
        \Bigc{
            r\in\mathfrak{R}_\phi\,\big|\,&J^{\pie}(r;p)=J^*(r;p)
            }.
\end{align}
\end{defi}
\begin{defi}[Feasible Set for Subptimal Expert in Linear MDPs]\label{def: fs suboptimal linear mdps}
    Let $U\ge L\ge 0$. Let $\cM_\phi=\tuple{\cS,\cA,H,d_0,p,\phi}$ be a Linear MDP without
    reward and let $\pie$ be the expert's policy,
    with suboptimality in $[L,U]$ under the true unknown reward $r^E$. The
    \emph{feasible set} $\fsl^{L,U}$ of rewards compatible with $\pie$
    in $\cM_\phi$ is defined as:
    \begin{align}\label{eq: fs suboptimal linear mdps}
        \fsl^{L,U}\coloneqq\Bigc{r\in\fR_\phi\,\big|\,&J^*(r;p)-J^{\pi^E}(r;p)\in[L,U]
            }.
\end{align}
\end{defi}
Before discussing whether efficient IRL algorithms can be developed for learning
the feasible set in Linear MDPs, we formalize what we mean by ``efficient
algorithm'' in the simple setting with a generative model for $p$ and $\pi^E$.
The following definition customizes the general notion of PAC algorithm for IRL
(e.g., see \citep{metelli2023towards}) to the Linear MDPs setting:
\begin{defi}[PAC Algorithm]\label{defi: pac algorithm}%
    An algorithm $\fA$ is a PAC algorithm for the IRL setting with \emph{optimal
        expert} w.r.t. a metric $\rho$ between sets of rewards if, for any
        $\epsilon,\delta\in(0,1)$, for any Linear MDP without reward $\cM_\phi$
        and expert's policy $\pi^E$, the algorithm $\fA$ outputs set
        $\widehat{\cR}$ such that:
        \begin{align*}
            \mathop{\P}\limits_{(\cM_\phi,\pi^E),\fA}\Bigr{\rho\bigr{\widehat{\cR},\fsl}\le\epsilon}\ge 1-\delta,
        \end{align*}
        where $\P_{(\cM_\phi,\pi^E),\fA}$ denotes the probability measure
        induced by executing $\fA$ in the Linear IRL problem $(\cM_\phi,\pie)$,
        and $\widehat{\cR}$ is an estimate of the feasible set computed by $\fA$
        using $\tau$ calls to the generative model of $p$ and $\tau^E$ calls to
        the generative model of $\pi^E$.\\
        We define a PAC algorithm for the IRL setting with
        $[L,U]-$\emph{suboptimal expert} (with arbitrary $U\ge L\ge 0$)
        analogously.
\end{defi}
We say that an IRL algorithm is \emph{efficient} if it is PAC with values of
$\tau$ and $\tau^E$ that do not depend on the cardinality of the state space.
%

Intuitively, Linear MDPs provide structure to the transition model $p$ (and the
reward function), but \emph{not to the expert's policy $\pi^E$}. If we assume
that $\pi^E$ is known at all $(s,h)$, then, it is possible to design a PAC
algorithm for the optimal expert setting (Eq. \eqref{eq: fs optimal linear
mdps}) whose sample complexity is independent of the size of the state space
(see Algorithm 1 of \citet{lazzati2024scaleinverserllarge}).
%
%
However, if $\pi^E$ is unknown, we prove that no efficient algorithm can be
developed, i.e., in the worst case, any PAC algorithm has to collect a number of
samples that depends on the size of the state space.
\begin{restatable}[Statistical Inefficiency for Optimal
Expert]{thr}{thrfsnotlearnableopt}\label{thr: fs optimal not learnable}
    Let $\fA$ be a PAC algorithm for the IRL setting with optimal expert. Then,
    there exists a Linear MDP without reward $\cM_\phi$ with a state space with
    finite but arbitrarily large cardinality $S$, and a deterministic expert's
    policy $\pi^E$, in which $\fA$ requires at least $\tau^E\ge S$ calls to the
    generative model of $\pi^E$.
\end{restatable}
\begin{proof}
    Let $\fA$ be a PAC algorithm for the setting with optimal expert, and let
    $\rho$ be any metric between sets of rewards with respect to which $\fA$ is
    PAC. Let $\cM_\phi=\tuple{\cS,\cA,H,d_0,p,\phi}$ be a Linear MDP without
    reward, where $\cS=\{s_1,s_2,\dotsc,s_S\}$ is the finite state space,
    $\cA=\{a_1,a_2\}$ is the action space, $H=1$ is the horizon, $d_0(s)>0$
    $\forall s\in\cS$ is the initial state distribution, and
    $\phi(s,a)=\indic{a=a_1}$ $\forall (s,a)\in\SA$ is the feature map. Since
    $H=1$, we do not have to specify $p$.
    
    We construct a family
    $\bM=\bigr{(\cM_\phi,\pi^{E,0}),(\cM_\phi,\pi^{E,1}),\dotsc,(\cM_\phi,\pi^{E,S})}$
    of $S+1$ IRL problem instances that share the same environment $\cM_\phi$,
    but differ in the expert's policy. Specifically, we define the deterministic
    policy $\pi^{E,0}$ that always plays action $a_1$:
    $\pi^{E,0}_1(a_1|s)=1\,\forall s\in\cS$. For all $i\in\dsb{S}$, we define
    the deterministic policy $\pi^{E,i}$ that plays action $a_1$ at all states
    except state $s_i$: $\pi^{E,i}_1(a_1|s)=1\,\forall s\in\cS\setminus\{s_i\}$,
    $\pi^{E,i}_1(a_2|s_i)=1$.

    By direct calculation, the feasible sets of the various problem instances
    are:
    \begin{align*}
        &\cR_{\cM_\phi,\pi^{E,0}}=\Bigc{r\in\fR\,\big|\,\forall s\in\cS:\,r_1(s,a_1)\in[0,1]\,\wedge\,
        r_1(s,a_2)=0},\\
        &\cR_{\cM_\phi,\pi^{E,1}}=\dotsc=\cR_{\cM_\phi,\pi^{E,S}}=
        \Bigc{r\in\fR\,\big|\,\forall (s,a)\in\SA:\,r_1(s,a)=0}.
    \end{align*}
    %
    Since $\rho$ is a metric and
    $\cR_{\cM_\phi,\pi^{E,0}}\neq\cR_{\cM_\phi,\pi^{E,1}}$, then there exists a
    constant $c>0$ such that:
    \begin{align*}
        \rho\Bigr{\cR_{\cM_\phi,\pi^{E,0}},\cR_{\cM_\phi,\pi^{E,1}}}\ge c.
    \end{align*}
    We proceed by contradiction. Assume that $\fA$ is a PAC algorithm that
    requires less than $\tau^E < S$ calls to the generative model of the
    expert's policy in all possible IRL problem instances. When $\fA$ tackles
    problem instance $(\cM_\phi,\pi^{E,0})$, irrespective of the randomization
    method it uses to choose the states from which collecting samples, since it
    collects less than $S$ samples, then there is at least one state $s_j$ with
    $j\in\dsb{S}$ in which $\fA$ does not know whether the expert's policy plays
    $a_1$ or $a_2$. Therefore, $\fA$ cannot distinguish between $\pi^{E,0}$ and
    $\pi^{E,j}$, and thus it cannot know whether the true feasible set is
    $\cR_{\cM_\phi,\pi^{E,0}}$ or
    $\cR_{\cM_\phi,\pi^{E,j}}=\cR_{\cM_\phi,\pi^{E,1}}$.
    Nevertheless, by hypothesis, being $\fA$ a PAC algorithm, it must output a set
    $\widehat{\cR}$ such that, for any $\epsilon,\delta\in(0,1)$:
    \begin{align*}
        \scalebox{0.97}{$  \displaystyle\mathop{\P}\limits_{(\cM_\phi,\pi^{E,0}),\fA}\biggr{\Bigc{\rho\bigr{
                \cR_{\cM_\phi,\pi^{E,0}},\widehat{\cR}}\le\epsilon}}\ge 1-\delta
                \wedge
                \mathop{\P}\limits_{(\cM_\phi,\pi^{E,0}),\fA}\biggr{\Bigc{\rho\bigr{
                \cR_{\cM_\phi,\pi^{E,1}},\widehat{\cR}}\le\epsilon}}\ge 1-\delta.$}
    \end{align*}

    In particular, for the choice of accuracy $\epsilon<c/2$ and failure
    probability $\delta<1/2$, we have:
    \begin{align*}
        \mathop{\P}\limits_{(\cM_\phi,\pi^{E,0}),\fA}\biggr{\Bigc{\rho\bigr{
                \cR_{\cM_\phi,\pi^{E,0}},\widehat{\cR}}<\frac{c}{2}}}>\frac{1}{2}
                \wedge
                \mathop{\P}\limits_{(\cM_\phi,\pi^{E,0}),\fA}\biggr{\Bigc{\rho\bigr{
                \cR_{\cM_\phi,\pi^{E,1}},\widehat{\cR}}<\frac{c}{2}}}>\frac{1}{2}.
    \end{align*}
    However, this results in a contradiction:
    \begin{align*}
        1&=\frac{1}{2}+\frac{1}{2}\\
        &<\mathop{\P}\limits_{(\cM_\phi,\pi^{E,0}),\fA}\biggr{\Bigc{\rho\bigr{
            \cR_{\cM_\phi,\pi^{E,0}},\widehat{\cR}}<\frac{c}{2}}}
        +\mathop{\P}\limits_{(\cM_\phi,\pi^{E,0}),\fA}\biggr{\Bigc{\rho\bigr{
            \cR_{\cM_\phi,\pi^{E,1}},\widehat{\cR}}<\frac{c}{2}}}\\
        &\markref{(1)}{=}
        \mathop{\P}\limits_{(\cM_\phi,\pi^{E,0}),\fA}\biggr{\Bigc{\rho\bigr{
            \cR_{\cM_\phi,\pi^{E,0}},\widehat{\cR}}<\frac{c}{2}}\;\cup\;
            \Bigc{\rho\bigr{
            \cR_{\cM_\phi,\pi^{E,1}},\widehat{\cR}}<\frac{c}{2}}},
    \end{align*}
    where at (1) we use that the two events are disjoint because $\rho$
    satisfies the triangle inequality. Since the probability of no event can be
    larger than 1, we get a contradiction that demonstrates the claim of the
    theorem.
\end{proof}

\begin{restatable}[Statistical Inefficiency for Subptimal
    Expert]{thr}{thrfsnotlearnablesubopt}\label{thr: fs suboptimal not learnable}
        Let $0\le L\le U<1$ be arbitrary. Let $\fA$ be a PAC algorithm for the
        IRL setting with $[L,U]$-suboptimal expert. Then, there exists a Linear
        MDP without reward $\cM_\phi$ with a state space with finite but
        arbitrarily large cardinality $S$, and a deterministic expert's policy
        $\pi^E$, in which $\fA$ requires at least $\tau^E\ge S$ calls to the
        generative model of $\pi^E$.
    \end{restatable}
\begin{proof}
    Let $0\le L\le U<1$ be arbitrary. Let $\fA$ be a PAC algorithm for the
    setting with $[L,U]$-suboptimal expert, and let $\rho$ be any metric between
    sets of rewards with respect to which $\fA$ is PAC. We consider the same
    family $\bM$ of IRL problem instances used in the proof of Theorem \ref{thr:
    fs optimal not learnable}. Let us compute the feasible sets for these problems.
    
    By definition of Linear MDPs, the reward parameter $\theta$ belongs to
    $[-1,+1]$. For all $\theta\in[-1,+1]$, we denote by
    $\Delta_0(\theta)\coloneqq J^*(r_\theta;p)-J^{\pi^{E,0}}(r_\theta;p)$ the
    suboptimality of $\pi^{E,0}$ in $\cM_\phi$ using reward $r_\theta$ defined
    as $r_\theta(s,a)=\dotp{\phi(s,a),\theta}$ $\forall s,a$. Similarly, for all
    $i\in\dsb{S}$, we denote by $\Delta_i(\theta)\coloneqq
    J^*(r_\theta;p)-J^{\pi^{E,i}}(r_\theta;p)$ the suboptimality of $\pi^{E,i}$
    in $\cM_\phi$ using reward $r_\theta$ defined previously. We distinguish
    three cases:
    \begin{itemize}
        \item If $\theta\in[-1,0)$ $\to$ then $\Delta_0(\theta)=-\theta$ and
        $\Delta_i(\theta)=-\frac{S-1}{S}\theta$ $\forall i\in\dsb{S}$.
        \item If $\theta=0$ $\to$ then $\Delta_0(\theta)=\Delta_i(\theta)=0$ $\forall i\in\dsb{S}$.
        \item If $\theta\in(0,+1]$ $\to$ then $\Delta_0(\theta)=0$ and
        $\Delta_i(\theta)=\frac{\theta}{S}$ $\forall i\in\dsb{S}$.
    \end{itemize}
    Based on these expressions, we can compute the feasible sets. When $L=0$,
    the values of $\theta\in[-1,+1]$ that make, respectively, $\Delta_0(\theta)$
    and $\Delta_i(\theta)$ in $[0,U]$ for all $i\in\dsb{S}$ are:
    \begin{align*}
        \cR_{\cM_\phi,\pi^{E,0}}^{0,U}&=
        \Bigc{r\in\fR\,\big|\,\exists \theta\in[-U,+1]:\,\forall (s,a)\in\SA:\,
        r(s,a)=\dotp{\phi(s,a),\theta}}\\
        &=\Bigc{r\in\fR\,\big|\,\forall s\in\cS:\,r_1(s,a_1)\in[-U,1]\,\wedge\,
        r_1(s,a_2)=0},\\
        \cR_{\cM_\phi,\pi^{E,i}}^{0,U}&=
        \Bigc{r\in\fR\,\big|\,\exists \theta\in[\max\{-\frac{S}{S-1}U,-1\},\min\{SU,+1\}]:\\
        &\qquad\qquad\qquad\forall (s,a)\in\SA:\,
        r(s,a)=\dotp{\phi(s,a),\theta}}\\
        &=\scalebox{0.99}{$\Bigc{r\in\fR\,\big|\,\forall s\in\cS:\,r_1(s,a_1)\in
        [\max\{-\frac{S}{S-1}U,-1\},\min\{SU,+1\}]\,\wedge\,
        r_1(s,a_2)=0}.$}
    \end{align*}
    Note that these sets are different for all $U\in[0,1),S\ge2$, thus we can
    proceed analogously as in the proof of Theorem \ref{thr: fs optimal not
    learnable} to demonstrate the result. When $L\neq 0$, we have, for all
    $i\in\dsb{S}$:
    \begin{align*}
        \cR_{\cM_\phi,\pi^{E,0}}^{L,U}&=
        \Bigc{r\in\fR\,\big|\,\exists \theta\in[-U,-L]:\,\forall (s,a)\in\SA:\,
        r(s,a)=\dotp{\phi(s,a),\theta}}\\
        &=\Bigc{r\in\fR\,\big|\,\forall s\in\cS:\,r_1(s,a_1)\in[-U,-L]\,\wedge\,
        r_1(s,a_2)=0},\\
        \cR_{\cM_\phi,\pi^{E,i}}^{L,U}&=
        \Big\{r\in\fR\,\big|\,\exists \theta\in[\max\{-\frac{S}{S-1}U,-1\},\max\{-\frac{S}{S-1}L,-1\}]\\
        &\qquad\qquad\qquad\scalebox{0.99}{$  \displaystyle\cup[\min\{SL,1\},\min\{SU,1\}]:
        \forall (s,a)\in\SA:\,
        r(s,a)=\dotp{\phi(s,a),\theta}\Big\} $}\\
        &=\Bigc{r\in\fR\,\big|\,\forall s\in\cS:\,r_1(s,a_1)\in
        [\max\{-\frac{S}{S-1}U,-1\},\max\{-\frac{S}{S-1}L,-1\}]
        \\
        &\qquad\qquad\qquad\cup[\min\{SL,1\},\min\{SU,1\}]\,\wedge\,
        r_1(s,a_2)=0}.
    \end{align*}
    Again, note that these sets are different for all $L\in(0,1),U\in[L,1)$,
    thus we can proceed analogously as in the proof of Theorem \ref{thr: fs
    optimal not learnable} to demonstrate the result. This concludes the proof.
\end{proof}
In words, the theorems above show that, even under the easiest learning
conditions (i.e., generative model and deterministic expert), and even imposing
the strong structure provided by Linear MDPs, the sample complexity of learning
the feasible set scales directly with the cardinality of the state space $S$,
making it is infeasible when $S$ is large or even infinite.

Of course, note that the feasible set is not efficiently learnable if we impose
structural assumptions milder than those imposed by Linear MDPs.
For instance, consider the families of problems with low Bellman rank
\citep{jiang2017contextual}, low Eluder dimension \citep{wang2020general}, or
low Bellman Eluder dimension \citep{jin2021eluder}. As shown by
\citet{jin2021eluder}, all of them contain the family of Linear MDPs, therefore,
the structure imposed to the reward function, and, thus, to the feasible set, is
\emph{at most} as strong as that provided by Linear MDPs (see Eq. \eqref{eq: fs
optimal linear mdps} and \eqref{eq: fs suboptimal linear mdps}) and at least as
strong as no structure (see Eq. \eqref{eq: fs optimal} and \eqref{eq: fs
suboptimal}). An analogous statement can be made for the transition model. For
these reasons, if we adapt Definition \ref{defi: pac algorithm} to these
settings, we obtain that, in the IRL setting with optimal or suboptimal expert,
any PAC algorithm for learning the feasible set requires at least $\tau^E\ge S$
calls to the generative model of $\pi^E$, because this is the minimum number of
samples required with Linear MDPs (see Theorem \ref{thr: fs optimal not
learnable} and Theorem \ref{thr: fs suboptimal not learnable}).
%

We remark that the same negative results hold also for other function
approximation settings, like Linear Mixture MDPs \citep{zhou2021nearly}:
\begin{infthr}[Statistical Inefficiency - Linear Mixture MDPs]
    \label{thr: fs not learnable linear mixture}
    In the IRL setting with optimal or suboptimal expert, any PAC algorithm for
    learning the feasible set in Linear Mixture MDPs
    requires at least $\tau^E\ge S$ calls to the generative model of $\pi^E$.
    \end{infthr}
    \begin{proofsketch}
        First, we observe that Linear Mixture MDPs do not impose structure on
        the reward function, thus the feasible sets coincide with those in
        Definition \ref{def: fs optimal} and \ref{def: fs suboptimal}. Next, we
        construct problem instances analogous to those constructed in the proofs
        of Theorem \ref{thr: fs optimal not learnable} and Theorem \ref{thr: fs
        suboptimal not learnable}, without the need of choosing specific values
        for the feature mapping. Finally, we recognize that if the learning
        algorithm does not ``know'' the expert's action in a state, then it does
        not know which action in that state should have the larger reward, thus
        ending up in an estimate with finite error, that can be reduced only
        discovering the expert's action at all states, i.e., with $\tau^E=S$
        samples (more than this if $\pi^E$ is stochastic). The proof can be
        formally concluded by contradiction as we did in Theorem \ref{thr: fs
        optimal not learnable}.
    \end{proofsketch}
To sum up, intuitively, learning the feasible set efficiently
requires the exact knowledge of (the support of) the expert's policy in the
entire state space. While collecting (at least) one sample from each state is
feasible in tabular MDPs, and there are many efficient algorithms doing so
\citep{metelli2023towards,zhao2023inverse,poiani2024inverse},
this becomes clearly infeasible when the state space is large.
Due to the results presented above, i.e., that the common function approximation
settings that permit the development of efficient algorithms for (forward) RL in
problems with large state spaces fail to do so for IRL, then we can conclude
that \textit{sample-efficient IRL requires structural assumptions stronger than
those required by sample-efficient RL}, as long as the solution concept
considered for IRL is the notion of feasible set.
%
%
In addition, we remark that, irrespective of the specific problem setting
considered (e.g., optimal/suboptimal expert), the feasible set contains an
\emph{infinite} amount of reward functions, thus it is not possible to
practically implementing an algorithm that outputs \emph{all} the rewards in the
feasible set. This fact limits the framework of the feasible set as a mere
theoretical tool.

\section{The Reward Compatibility Framework}\label{sec: the rewards
compatibility framework}

In this section, we present the main contribution of this work: \emph{Reward
Compatibility}, a novel framework for IRL that allows us to conveniently
rephrase the learning from demonstrations problem as a \emph{classification}
task.
More specifically, to overcome the two issues highlighted in Section \ref{sec:
limitations fs}:
\begin{itemize}[noitemsep, topsep=0pt]
    \item We introduce the notion of \emph{reward compatibility} to permit
    sample-efficient IRL without additional structural assumptions. This
    framework replaces the feasible set as solution concept for IRL.
    \item We reformulate IRL as a \emph{classification problem} to allow the
    practical development of efficient algorithms. An algorithm will take in input a
    single reward and will output a single boolean.
\end{itemize}

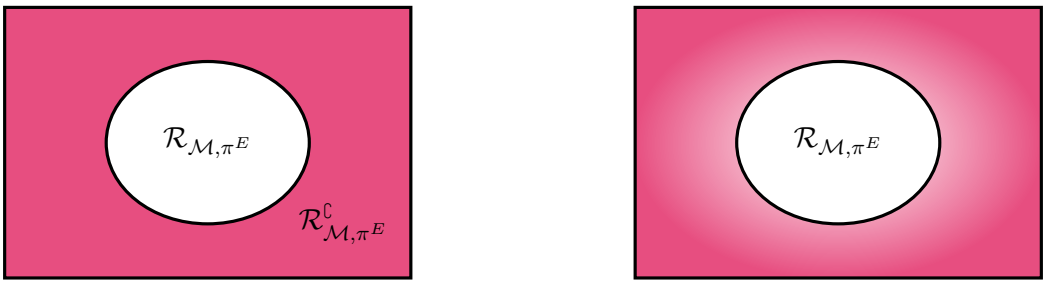
\begin{figure}[!t]
    \begin{subfigure}[b]{0.45\textwidth}
        \centering
        \usetikzlibrary{intersections}
        \usetikzlibrary{shadings}
        \begin{tikzpicture}[scale=0.90]
            \filldraw[color=rl3Salmon!80,draw=black, very thick]
            (-3,-2) rectangle (3,2);
            \filldraw[black,fill=white, very thick]
            (0,0) ellipse (1.5 and 1.2);
            \node at (0,0) {$\cR_{\cM,\pi^E}$};
            \node at (2,-1.2) {$\cR_{\cM,\pi^E}^\complement$};
        \end{tikzpicture}
    \end{subfigure}
    \hfill       
    \begin{subfigure}[b]{0.45\textwidth}
        \centering
    \usetikzlibrary{intersections}
    \usetikzlibrary{shadings}
    \begin{tikzpicture}[scale=0.90]
        
        \shade[draw=black,inner color=white, middle color=white,
        outer color=rl3Salmon!80, very thick]
        (-3,-2) rectangle (3,2);

        \filldraw[black,fill=white, very thick]
        (0,0) ellipse (1.5 and 1.2);
        \node at (0,0) {$\cR_{\cM,\pi^E}$};
    \end{tikzpicture}
    \end{subfigure}
    \caption
    {(Left) In the framework of the feasible set, all the rewards in
    $\cR_{\cM,\pi^E}^\complement$ (in pink), i.e., outside the feasible set
    $\cR_{\cM,\pi^E}$ (in white), are considered equally wrong. (Right) In the
    framework of the reward compatibility, the rewards outside the feasible set
    $\cR_{\cM,\pi^E}$ suffer from different errors (scale of pink). }
    \label{fig: fs vs rew comp}
\end{figure}

\subsection{Reward Compatibility}\label{sec: reward compatibility}

Let $\cM=\tuple{\cS,\cA,H,d_0,p}$ be an MDP without reward\footnote{For
simplicity, we consider a tabular MDP. Nevertheless, the presentation is
independent of structural assumptions of the MDP (e.g., Linear MDP).} and let
$\pie$ be the expert's policy. Let $U\ge L\ge0$ be arbitrary.
As described in the previous section, the feasible set for optimal expert $\fs$
(Definition \ref{def: fs optimal}) and the feasible set for suboptimal expert
$\fs^{L,U}$ (Definition \ref{def: fs suboptimal}) contain \emph{all and only}
the reward functions that \emph{exactly} comply with the given constraint, i.e.,
respectively, that make $\pi^E$ optimal, and that make the suboptimality of
$\pi^E$ inside $[L,U]$.
In other words, we can interpret the feasible set as carrying out a \emph{binary
classification} of rewards based on whether they satisfy a given ``hard''
constraint, as shown in Figure \ref{fig: fs vs rew comp} (Left). We say that the
rewards inside the feasible set are \emph{compatible} with the demonstrations of
$\pi^E$ \citep{metelli2023towards}. However, our insight is that some rewards
outside the feasible set are \emph{``more'' compatible} with $\pie$ than others,
as shown in Figure \ref{fig: fs vs rew comp} (Right) and in the following
example (a similar example may be constructed also for the suboptimal expert):

\begin{example}[label=exa:cont]\label{example: muffin cake soup} Consider an MDP
without reward $\cM$ with one state and horizon one in which the expert has
three actions: eating a muffin (M), a cake (C), or a soup (S). Assume that the
expert, assumed optimal, demonstrates action $\pi^E=M$, i.e., it eats the
muffin. Clearly, the rewards $r_1,r_2$ defined as follows:
  \begin{align*}
    r_1(a)=\begin{cases}
        +0.99 & \text{if }a=M\\
        +1 & \text{if }a=C\\
        0 & \text{if }a=S\\
    \end{cases},\qquad
    r_2(a)=\begin{cases}
        0 & \text{if }a=M\\
        +1 & \text{if }a=C\\
        0 & \text{if }a=S\\
    \end{cases},
  \end{align*}
  do not belong to the feasible set $\fs$, because they do not make ``eating the
  muffin'' the optimal strategy. However, intuitively, $r_1$ is \emph{``more''
  compatible} with $\pie$ than $r_2$, because it makes M a very good action,
  while reward $r_2$ makes it very bad. Clearly, we make a larger error if we
  model the preferences of the expert with $r_2$ instead of $r_1$. However, the
  feasible set $\fs$ is blind to the difference between $r_1$ and $r_2$, and it
  ``refuses'' both of them.
  \end{example}

  In the setting with optimal expert, we propose the following ``soft''
 definition of compatibility of a reward with demonstrations to capture this
 intuition. To be precise, since the larger the quantity the smaller the
 compatibility, we talk of (non)compatibility.

\begin{defi}[Reward (non)Compatibility - Optimal Expert]\label{def: reward
compatibility optimal} 

Let $\cM=\tuple{\cS,\cA,H,d_0,p}$ be an MDP without reward and let $\pie$ be the
expert's policy. For any reward $r\in\fR$, we define the
\emph{(non)compatibility} $\compr$ of $r$ with $\pi^E$ in
$\cM$ as:
  \begin{align}\label{eq: reward compatibility optimal}
      \compr\coloneqq
      J^*(r;p) - J^{\pie}(r;p).
  \end{align}
  \end{defi}
In words, $\compr$ quantifies the \textit{suboptimality} of $\pie$ in the MDP
obtained by $\cM$ and $r$.
Since, by definition of the setting, the true expert's reward
$r^E$ satisfies $J^*(r^E;p)-J^{\pie}(r^E;p)=0$, then it is clear that $\compr$
measures the degree to which reward $r$ fulfils this constraint.
Observe that the feasible set $\fs$ can be seen as the set of rewards with zero
(non)compatibility, i.e., $\fs=\{r\in\fR\,|\,\compr=0\}$.
Since the smaller the $\compr$ the more $r$ is compatible with $\pi^E$, and
since the larger the $\compr$ the less $r$ is compatible with $\pi^E$, we say
that $\comp$ quantifies the \emph{(non)}compatibility.
Visually, Figure \ref{fig: fs vs rew comp} (Right) represents function
$\comp:\fR\to[0,+\infty]$ using different magnitudes of the pink color. Rewards $r$
with zero $\compr$ (i.e., inside the feasible set) are pictured in white, while
rewards $r$ with larger values of $\compr$ are pictured with larger magnitudes
of pink.
\begin{example}[continues=exa:cont]
    Reward (non)compatibility permits to discriminate between $r_1$ and $r_2$.
  Indeed, we have that $\comp(r_1)=0.01$, $\comp(r_2)=1$. Since reward $r_1$
  suffers from a smaller (non)compatibility than $r_2$, i.e.,
  $\comp(r_1)<\comp(r_2)$, then we have that $r_1$ is more compatible with $\pie$
  than $r_2$, as expected.
  \end{example}
  By definition of IRL, the observed expert's policy $\pie$ does not reveal any
  information about the other policies. Thus, it is meaningful that $\comp$
  considers the suboptimality of $\pie$ only, as illustrated below.
  \begin{example}
  Consider now reward $r_1'$, defined below, and compare it with $r_1$.
  \begin{align*}
      r_1'(a)=\begin{cases}
          +0.99 & \text{if }a=M\\
          0 & \text{if }a=C\\
          +1 & \text{if }a=S\\
      \end{cases}.
    \end{align*}
    The optimal policy under $r_1'$ is to play $S$, while under $r_1$ we would
    play $C$. However, demonstrations from $\pi^E$ alone do not provide
    information on C or S, but only about $\pie=M$, therefore it is meaningful
    that $r_1$ and $r_1'$ are equally compatible $\comp(r_1')=\comp(r_1)=0.01$
    with the given demonstrations.
  \end{example}
In an analogous manner, we can define a notion of reward compatibility for the
setting with suboptimal expert:
\begin{defi}[Reward (non)Compatibility - Suboptimal Expert]\label{def: reward
compatibility suboptimal}
    
Let $U\ge L\ge 0$ be arbitrary. Let $\cM=\tuple{\cS,\cA,H,d_0,p}$ be an MDP
without reward and let $\pie$ be the expert's policy. For any reward $r\in\fR$,
we define the \emph{(non)compatibility} $\comprlu$ of $r$ with $\pi^E$ in $\cM$
based on $L,U$ as:
  \begin{align}\label{eq: reward compatibility suboptimal}
      \comprlu\coloneqq
      \min\limits_{x\in[L,U]}\bigg|\Bigr{J^*(r;p) - J^{\pie}(r;p)}-x\bigg|.
  \end{align}
    \end{defi}
    Simply put, $\overline{\cC}^{L,U}_{\cM,\pi^E}(r)$ measures how far from the
    interval $[L,U]$ lies the suboptimality of the
    expert's policy $\pie$ w.r.t. $r$ in $\cM$.
    In other words, $\overline{\cC}^{L,U}_{\cM,\pi^E}(r)$ quantifies the extent to
    which reward $r$ fulfils the constraint
    $J^*(\cdot;p)-J^{\pi^E}(\cdot;p)\in[L,U]$, that defines the setting. 
    Note that, again, the feasible set can be seen as the set of rewards with
    zero (non)compatibility
    $\fslu=\{r\in\fR\,|\,\comprlu=0\}$.

\subsection{The IRL Classification Formulation}\label{sec: irl classification
formulation}

In practical applications, we are often interested in understanding whether some
designed or computed reward function represents the preferences of the observed
expert agent. Beyond not being practically implementable, as mentioned in
Section \ref{sec: limitations fs}, an algorithm that outputs the feasible set
would not even permit to search efficiently for specific rewards inside the set.
For these reasons, we reformulate IRL as a classification problem as follows.
  \begin{defi}[IRL Classification Problem]\label{def: IRL problem} An \emph{IRL
    Classification Problem} instance is a tuple $\tuple{\cM,\pie,\cR,\Delta}$,
    where $\cM$ is an MDP without reward, $\pie$ is the expert's policy,
    $\cR\subseteq\mathfrak{R}$ is the set of rewards to classify, and
    $\Delta\in\RR_{\ge0}$ is a threshold. The goal is to classify the rewards
    $r\in\cR$ based on a given notion of (non)compatibility
    $\overline{\cC}:\fR\to\RR_{\ge0}$:
    \begin{align*}
        \forall r\in\cR:\;\textnormal{ \textbf{if} }\;\overline{\cC}(r)
        \le\Delta\;\textnormal{ \textbf{then} } \;  \text{ \textnormal{\textbf{return}} \textnormal{True}, \; \textnormal{\textbf{else}} \textnormal{\textbf{return}} \textnormal{False}}.
    \end{align*}
\end{defi}
\begin{defi}[IRL Algorithm]\label{def: IRL algorithm}
    An \emph{IRL algorithm} takes in input a
    reward $r\in\cR$ and outputs a boolean saying whether
    $\overline{\cC}(r)\le\Delta$.
\end{defi}
In words, in an IRL classification problem we aim to classify  a given set of
reward functions $\cR$ in two classes, depending on their (non)compatibility
with a given expert's policy $\pi^E$ in an environment $\cM$. Specifically, one
class represents rewards that are compatible with $\pi^E$, while the other class
describes rewards that are not. Whether a (non)compatibility value is considered
small or large depends on the classification threshold $\Delta$.
Moreover, note that the definitions above are general, and the notion of
(non)compatibility $\overline{\cC}$ adopted depends on the specific IRL setting
considered. In particular, they can be applied to settings other than IRL with
optimal \citep{ng2000algorithms} or suboptimal \citep{poiani2024inverse} expert,
like maximum entropy IRL \citep{ziebart2008maximum} or Bayesian IRL
\citep{Ramachandran2007birl}, as long as a suitable definition of
(non)compatibility is provided.
Observe that we allow for $\cR\neq\mathfrak{R}$ to manage scenarios in which,
for instance, we have some prior knowledge on $r^E$, i.e.,
$r^E\in\cR\subset\mathfrak{R}$, and therefore we want to ``update'' our
knowledge with the observed demonstrations of $\pi^E$ without wasting time or
computational power with the non-interesting rewards outside $\cR$.

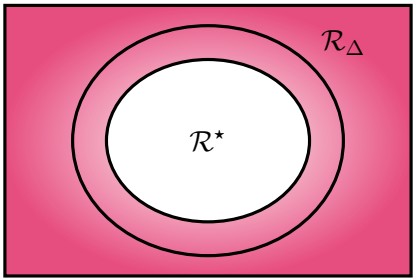
\begin{figure}[t]
    \centering
\usetikzlibrary{intersections}
\usetikzlibrary{shadings}
\begin{tikzpicture}[scale=0.9]
    \shade[draw=black,inner color=white,
    outer color=rl3Salmon!80, very thick]
    (-3,-2) rectangle (3,2);
    \filldraw[black,fill=white, very thick]
    (0,0) ellipse (1.5 and 1.2);
    \draw[black,very thick]
    (0,0) ellipse (2 and 1.7);
    \node at (0,0) {$\cR^\star$};
    \node at (2,1.45) {$\cR_{\Delta}$};
\end{tikzpicture}
\caption
{The set of rewards positively classified by an IRL algorithm
$\cR_\Delta$ with $\Delta>0$ represents an enlargement of the feasible set
$\cR^\star$, i.e., $\cR^\star\subseteq\cR_\Delta$. Visually, $\cR_\Delta$
contains rewards outside $\cR^\star$ whose magnitude of pink is not too
intense.}
\label{fig: r delta}
\end{figure}

\textbf{Relation with the feasible set.}~~How does the solution concept/learning
target of the IRL classification formulation relate to the feasible set? Let us
denote by $\cR^\star\coloneqq\{r \in \fR \,|\, \overline{\cC}(r)=0 \}$ the
feasible set for the considered IRL setting. Moreover, let $\cR=\fR$ (i.e., we
have to classify all the rewards) and define the set of rewards positively
classified by an IRL algorithm as $\cR_\Delta$, i.e.,
$\mathcal{R}_{\Delta}\coloneqq \{r \in \fR \,|\, \overline{\cC}(r) \le\Delta
\}$. Then, it is clear that, for any $\Delta\ge 0$:
\begin{align*}
    \cR^\star=\cR_0\subseteq\cR_\Delta.
\end{align*}
In words, if we choose $\Delta=0$, then we are simply classifying as positive
all and only the rewards inside the feasible set, while if we consider $\Delta
>0$ strictly, then we are enlarging the feasible set. The situation is
exemplified in Figure \ref{fig: r delta}.

Differently from the framework of the feasible set, as we shall see in the next
section, it is possible to practically implement the new notion of IRL algorithm
(Definition \ref{def: IRL algorithm}), with guarantees of sample efficiency even
when the state space is large, making the same structural assumptions considered
for (forward) RL.
Intuitively, the reward compatibility framework permits the development of
sample-efficient algorithms at the price of an enlargement of the feasible
set controlled by the classification threshold $\Delta$.

\section{Learning Compatible Rewards in the Online Setting}\label{sec: online}

In this section, we first describe the online learning setting for IRL
classification, and, then, we present an algorithm, \caty (\catylong), that
solves the task in a sample-efficient manner.

\subsection{Problem Setting}\label{subsec: problem setting online}

In practical applications, the transition model $p$ of the environment and the
expert's policy $\pi^E$ are not known, but have to be estimated from samples.
We consider the following learning setting for the IRL classification problem,
which is made of two phases: an \textcolor{vibrantBlue}{exploration phase} and a
\textcolor{vibrantRed}{classification phase}.

First, during the \textcolor{vibrantBlue}{exploration phase}, the learning algorithm receives as
input an expert's dataset
$\cD^E=\{\tuple{s_1^{E,i},a_1^{E,i},\dotsc,s_H^{E,i},a_H^{E,i},s_{H+1}^{E,i}}\}_{i\in\dsb{\tau^E}}$
of $\tau^E$ state-action trajectories collected by the expert's policy $\pi^E$
in the (unknown) environment $\cM$, and the set of rewards to classify $\cR$,
and it is allowed to explore the environment $\cM$ at will to collect an
exploration dataset
$\cD=\{\tuple{s_1^i,a_1^i,\dotsc,s_H^i,a_H^i,s_{H+1}^i}\}_{i\in\dsb{\tau}}$ of
$\tau$ state-action trajectories.

Next, during the \textcolor{vibrantRed}{classification phase}, the learning algorithm is not
allowed to interact with the environment $\cM$ anymore. It receives in input a
reward $r\in\cR$ and a threshold $\Delta\ge0$, and it must classify $r$ based on
its (non)compatibility $\overline{\cC}(r)$ and $\Delta$. Since the transition
model $p$ of $\cM$ and the expert's policy $\pi^E$ are unknown, the learning
algorithm has to use the datasets $\cD$ and $\cD^E$ to construct a meaningful
estimate of (non)compatibility $\widehat{\cC}(r)\approx\overline{\cC}(r)$
for the classification.

\subsection{Learning Framework}\label{subsec: online learning framework}

Observe that the considered learning setting represents the most common IRL
setting in practical applications, where we are given a batch dataset of expert
demonstrations $\cD^E$ and we can actively explore the environment to estimate
its dynamics $p$. Intuitively, a learning algorithm is efficient in this context
if it carries out an accurate classification of the input rewards with high
probability, using the least amount of samples $\tau^E$ and $\tau$ possible. We
formalize this concept as follows:
  
  \begin{defi}[PAC Algorithm - Online setting]\label{def: pac algorithm online}
  Let $\epsilon>0,\delta\in(0,1)$, and let $\cD^E$ be a dataset of $\tau^E$
  expert's trajectories. An algorithm $\mathfrak{A}$ exploring for $\tau$
  episodes is \emph{$(\epsilon,\delta)$-PAC} for the IRL classification problem if:
  \begin{align*} 
      \mathop{\mathbb{P}}\limits_{\cM,\pie,\mathfrak{A}}\Big(
          \sup\limits_{r\in\cR} \Big|\overline{\cC}(r)-
          \widehat{\cC}(r)\Big|\le\epsilon\Big)\ge 1-\delta,
  \end{align*}
  where $\mathbb{P}_{\cM,\pie,\mathfrak{A}}$ is the joint probability measure
  induced by $\pie$ and $\mathfrak{A}$ in $\cM$, and $\widehat{\cC}$ is the
  estimate of some (non)compatibility notion $\overline{\cC}$ computed by
  $\mathfrak{A}$ using $\tau^E$ and $\tau$ trajectories. The \emph{sample
  complexity} is defined by the pair $(\tau^E,\tau)$.
  \end{defi}
  Again, observe that this definition is general and independent of the specific
  (non)compatibility notion $\overline{\cC}$ adopted.
  We remark that, according to Definition \ref{def: pac algorithm online}, an
  algorithm is PAC depending on the accuracy with which it estimates the
  (non)compatibility of the various rewards in input, and not on the accuracy
  with which it classifies them. 
  Thus, to understand why it is interesting to design PAC algorithms for the IRL
  classification problem based on Definition \ref{def: pac algorithm online}, we
  have to understand what is the accuracy of such an algorithm for that task.

To this aim, let us consider an IRL classification problem instance
$\tuple{\cM,\pi^E,\cR,\Delta}$ with (non)compatibility $\overline{\cC}$, and an
$(\epsilon,\delta)$-PAC algorithm  $\fA$ (for some $\epsilon>0,\delta\in(0,1)$),
that computes estimates of (non)compatibility $\widehat{\cC}$ and classifies the
rewards based on some threshold $\eta\in\RR$ potentially different from
$\Delta$.\footnote{Threshold $\Delta$ is settled by the problem, but we can
implement our algorithm using any other value of threshold.} We denote the set
of rewards to classify with true (non)compatibility smaller than $\Delta$ as
$\cR_{\Delta}\coloneqq \{r \in \cR \,|\, \overline{\cC}(r) \le\Delta \}$, and
the set of rewards positively classified by algorithm $\fA$ as
$\widehat{\cR}_{\eta}\coloneqq \{r \in \cR \,|\, \widehat{\cC}(r) \le\eta \}$.
The main question is: \emph{what is the relationship between $\cR_{\Delta}$ and
$\widehat{\cR}_{\eta}$}?

By Definition \ref{def: pac algorithm online}, it is easy to see that, with probability at
least $1-\delta$, it holds that:
\begin{align*}
    \widehat{\cR}_{\Delta-\epsilon}\subseteq\cR_{\Delta}\subseteq
    \widehat{\cR}_{\Delta+\epsilon},
\end{align*}
namely, with a careful choice of the classification threshold $\eta$ adopted by
the algorithm $\fA$, we can guarantee that, under the good event (i.e., with probability at least $1-\delta$), either
\emph{all} the rewards inside $\cR_\Delta$, i.e., rewards with true
(non)compatibility smaller than $\Delta$, are classified correctly
($\eta\ge\Delta+\epsilon$), or \emph{only} the rewards inside $\cR_\Delta$,
i.e., rewards with true (non)compatibility smaller than $\Delta$, are classified
correctly ($\eta\le\Delta-\epsilon$).
Thus, we can trade-off the amount of ``false negatives''/``false positives'' by
choosing the threshold $\eta$. In particular, note that, if we choose
$\eta=\Delta-\epsilon$ (to minimize the ``false positives''), then we have the
guarantee that $\widehat{\cR}_{\eta}$ is not ``too small'':
\begin{align*}
    \cR_{\Delta-2\epsilon}\subseteq\widehat{\cR}_\eta\subseteq\cR_{\Delta}.
\end{align*}
Analogously, if we choose $\eta=\Delta+\epsilon$ (to minimize the ``false
negatives''), then we have the guarantee that $\widehat{\cR}_{\eta}$ is not
``too large'':
\begin{align*}
    \cR_{\Delta}\subseteq\widehat{\cR}_\eta\subseteq\cR_{\Delta+2\epsilon}.
\end{align*} 

To sum up, the notion of PAC algorithm in Definition \ref{def: pac algorithm
online} is meaningful because it permits to manage the amount of ``false
negatives'' and ``false positives'' in a simple yet effective manner. Moreover,
it guarantees that, if $\eta\in[\Delta-\epsilon,\Delta+\epsilon]$, then all the
rewards $r\in\cR$ with true (non)compatibility $\overline{\cC}(r)\le
\Delta-2\epsilon \vee \overline{\cC}(r)\ge \Delta+2\epsilon$ are correctly
classified with high probability, as shown in Figure \ref{fig: reward
classification}.

\begin{figure*}[!t]
    \centering
    \begin{subfigure}[b]{0.3\textwidth}
        \centering
        \begin{tikzpicture}
            \filldraw [black] (0,0) circle (2pt);
            \draw[-{Latex[scale=1.]}] (0,0) -- (3.5,0);
            \draw[thick,red] (0.5,0) -- (2.1,0);
            \filldraw [black] (1.3,0) circle (2pt);
            \draw (2.5,0.5) -- (2.5,-0.5);
            \draw[thick] (0.5,0.3) -- (0.5,-0.3);
            \draw[thick] (0.5,0.3) -- (0.6,0.3);
            \draw[thick] (0.5,-0.3) -- (0.6,-0.3);
            \draw[thick] (2.1,0.3) -- (2.1,-0.3);
            \draw[thick] (2,0.3) -- (2.1,0.3);
            \draw[thick] (2,-0.3) -- (2.1,-0.3);
            \draw (0.1,0) node[anchor=north east] {$0$};
            \draw (3.2,0) node[anchor=north west] {$ \RR$};
            \draw (2.5,-0.5) node[anchor=north] {$\Delta$};
            \draw (1.3,0) node[anchor=north] {\tiny$\overline{\cC}(r)$};
            \draw (0.6,0.3) node[anchor=south] {\tiny$-\epsilon$};
            \draw (2.1,0.3) node[anchor=south] {\tiny$+\epsilon$};
        \end{tikzpicture}
        \caption
        {{ Reward $r$ is classified correctly.}}    
        \label{fig:r correct}
    \end{subfigure}
    \hfill
    \begin{subfigure}[b]{0.3\textwidth}  
        \centering 
        \begin{tikzpicture}
            \filldraw [black] (0,0) circle (2pt);
            \draw[-{Latex[scale=1.]}] (0,0) -- (3.5,0);
            \draw[thick,red] (1.2,0) -- (2.8,0);
            \filldraw [black] (2,0) circle (2pt);
            \draw (2.5,0.5) -- (2.5,-0.5);
            \draw[thick] (1.2,0.3) -- (1.2,-0.3);
            \draw[thick] (1.2,0.3) -- (1.3,0.3);
            \draw[thick] (1.2,-0.3) -- (1.3,-0.3);
            \draw[thick] (2.8,0.3) -- (2.8,-0.3);
            \draw[thick] (2.7,0.3) -- (2.8,0.3);
            \draw[thick] (2.7,-0.3) -- (2.8,-0.3);
            \draw (0.1,0) node[anchor=north east] {$0$};
            \draw (3.2,0) node[anchor=north west] {$ \RR$};
            \draw (2.5,-0.5) node[anchor=north] {$\Delta$};
            \draw (2,0) node[anchor=north] {\tiny$\overline{\cC}(r)$};
            \draw (1.2,0.3) node[anchor=south] {\tiny$-\epsilon$};
            \draw (2.8,0.3) node[anchor=south] {\tiny$+\epsilon$};
        \end{tikzpicture}
        \caption[]%
        {{ Reward $r$ can be mis-classified.}}
        \label{fig: r maybe wrong}
    \end{subfigure}
    \hfill
    \begin{subfigure}[b]{0.3\textwidth}  
        \centering 
        \begin{tikzpicture}
            \filldraw [black] (0,0) circle (2pt);
            \draw[-{Latex[scale=1.]}] (0,0) -- (3.5,0);
            \draw[thick,red] (1.2,0) -- (2.8,0);
            \draw (2,0.5) -- (2,-0.5);
            \draw[thick] (1.2,0.3) -- (1.2,-0.3);
            \draw[thick] (1.2,0.3) -- (1.3,0.3);
            \draw[thick] (1.2,-0.3) -- (1.3,-0.3);
            \draw[thick] (2.8,0.3) -- (2.8,-0.3);
            \draw[thick] (2.7,0.3) -- (2.8,0.3);
            \draw[thick] (2.7,-0.3) -- (2.8,-0.3);
            \draw[-{Latex[scale=0.7]}] (1.2,0.1) -- (1.9,0.1);
            \draw[-{Latex[scale=0.7]}] (2.8,0.1) -- (2.1,0.1);
            \draw (0.1,0) node[anchor=north east] {$0$};
            \draw (3.2,0) node[anchor=north west] {$ \RR$};
            \draw (2,-0.5) node[anchor=north] {$\Delta$};
            \draw (1.2,0.3) node[anchor=south] {\tiny$-\epsilon$};
            \draw (2.8,0.3) node[anchor=south] {\tiny$+\epsilon$};
        \end{tikzpicture}
        \caption[]%
        {{ Range of uncertain (non)compatibility values.}}    
        \label{fig: range uncertain noncomp values}
    \end{subfigure}
    \caption
    {\small The axis represents (non)compatibility values
    $\overline{\cC}(\cdot)$ and we consider threshold $\eta=\Delta$. (a) Rewards
    $r$ with $\overline{\cC}(r)\le\Delta-\epsilon$ or
    $\overline{\cC}(r)\ge\Delta+\epsilon$ are correctly classified by an
    $(\epsilon,\delta)$-PAC with high probability, while (b) in the opposite
    case, $r$ can be mis-classified. (c) The red interval
    $[\Delta-\epsilon,\Delta+\epsilon]$ exemplifies the set of rewards
    $\{r\in\cR\,|\, |\overline{\cC}(r)-\Delta|\le\epsilon\}$ that are
    (potentially) mis-classified. The length of the interval reduces with
    $\epsilon$.}
    \label{fig: reward classification}
    \end{figure*}
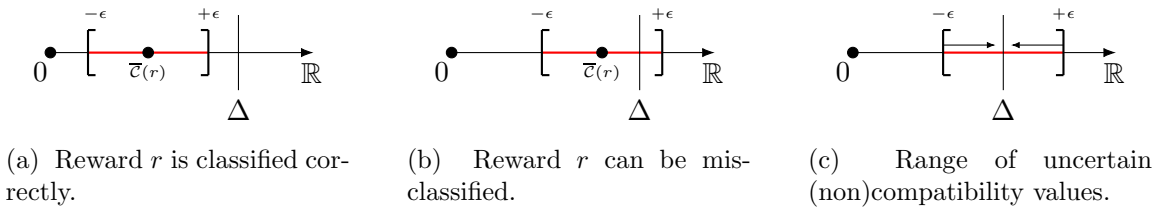

\subsection{Algorithm}\label{subsec: algorithm caty online}

In this section, we present \caty (\catylong), a learning algorithm for solving
the IRL classification problem in the online setting presented earlier. We begin
by describing \caty as a general algorithmic template that can be applied to a
variety of function approximation settings (i.e., classes of MDPs, like tabular
MDPs or linear MDPs) that satisfy certain properties,\footnote{As we will see,
if two properties are satisfied, then \caty is a PAC algorithm for the
considered setting.} and to two different kinds of reward (non)compatibility
(optimal expert in Definition \ref{def: reward compatibility optimal}, and
suboptimal expert in Definition \ref{def: reward compatibility suboptimal}).
Later on, we ``instantiate'' \caty in three classes of MDPs (tabular MDPs,
tabular MDPs with linear rewards, Linear MDPs), and we demonstrate that \caty is
a PAC algorithm (see Definition \ref{defi: pac algorithm}) for all these
settings by providing explicit sample complexity bounds.

An IRL algorithm for the online setting is made of two phases:

\textbf{\textcolor{vibrantBlue}{Exploration phase.}}~~During the
\emph{exploration} phase (see Algorithm \ref{alg: caty online exploration}),
\caty collects a dataset $\cD$ of trajectories by executing a reward-free
exploration (RFE) algorithm $\fA$ \citep{jin2020RFE} for the considered function
approximation setting. In this way, we are guaranteed that, whatever reward
$r\in\cR$ will be provided in input to the classification phase, we will be able
to compute an estimate $\widehat{J}^*(r)$ ``sufficiently'' close to $J^*(r;p)$
using a reasonably small amount of samples (i.e., polynomial in the problem
dimensions). Formally, by definition, any RFE algorithm $\fA$ guarantees that,
for any $\epsilon>0,\delta\in(0,1)$, there exists (a reasonably small) $N>0$
such that, if $\fA$ is executed for $\tau\ge N$ times, then it holds that:
\begin{align}\label{eq: requirement rfe}
    \mathop{\P}\limits_{\cM,\pie,\fA}\Bigr{\sup\limits_{r\in\fR}\big|J^*(r;p)-
    \widehat{J}^*(r)\big|\le\frac{\epsilon}{2}}\ge 1-\frac{\delta}{2}.
\end{align}
The reason why Algorithm \ref{alg: caty online exploration} receives in input
also the set $\cR$ will be clear later.

\begin{algorithm}[t!]
    \caption{\caty - Exploration phase
    }\label{alg: caty online exploration}
    \DontPrintSemicolon
    \KwData{
    Rewards to classify $\cR$, number of episodes $\tau$
    }
    \nonl \texttt{// Explore the environment with a RFE algorithm:}\;
    $\cD\gets$ RFE\_Exploration($\tau$)\label{line: rfe exploration}\;
    Return $\cD$
    \end{algorithm}

\textbf{\textcolor{vibrantRed}{Classification phase.}}~~During the
\emph{classification} phase (see Algorithm \ref{alg: caty online
classification}), \caty performs the estimation $\widehat{\cC}(r)$ of the
(non)compatibility term for the input reward $r\in\cR$ by splitting it into two
\emph{independent} estimates: $\widehat{J}^E(r)\approx J^{\pie}(r;p)$, which is
computed with $\cD^E$ only, and $\widehat{J}^*(r)\approx J^*(r;p)$, which is
computed with $\cD$ only. 
We have already mentioned that the usage of a RFE algorithm guarantees that
$\widehat{J}^*(r)$ is close to $J^*(r;p)$.
Concerning the estimate $\widehat{J}^E(r)$ of the expert's performance
$J^{\pie}(r;p)$, as shown in Line \ref{line: estimate JE online} of Algorithm
\ref{alg: caty online classification}, we use the empirical estimate (i.e.,
sample mean). We want $\widehat{J}^E(r)$ to be ``close'' to $J^{\pie}(r;p)$ for
any $r\in\cR$. As such, we require the function approximation setting considered
to guarantee this closeness using a finite and reasonably small (i.e.,
polynomial in the problem dimensions) number of samples. Formally, we want that,
for any $\epsilon>0,\delta\in(0,1)$, whatever the input reward $r\in\cR$, there
exists (a reasonably small) $N^E>0$ such that, if the dataset $\cD^E$ contains
at least $\tau^E\ge N^E$ trajectories, then it holds that:
\begin{align}\label{eq: requirement estimate JE}
    \mathop{\P}\limits_{\cM,\pie}\Bigr{\sup\limits_{r\in\fR}\big|J^{p,\pie}(r;p)-
    \widehat{J}^E(r)\big|\le\frac{\epsilon}{2}}\ge 1-\frac{\delta}{2}.
\end{align}
Then, \caty combines the estimates $\widehat{J}^E(r)\approx J^{\pie}(r;p)$ and
$\widehat{J}^*(r)\approx J^*(r;p)$ to obtain the estimate of compatibility
$\widehat{\cC}(r)$ in a way dependent on the specific notion of
(non)compatibility considered. Specifically, in the setting with optimal expert,
in which the (non)compatibility $\overline{\cC}$ considered is $\comp$, defined
in Eq. \eqref{eq: reward compatibility optimal}, then \caty computes estimate
$\widehat{\cC}(r)$ for input reward $r\in\cR$ as:
\begin{align}\label{eq: caty estimate comp optimal online}
    \widehat{\cC}(r)=\widehat{J}^*(r)-\widehat{J}^E(r),
\end{align}
irrespective of the class of MDP.
By using a union bound and a triangle inequality, if Eq. \eqref{eq: requirement
rfe} and Eq. \eqref{eq: requirement estimate JE} hold, then it is easy to see
that \caty is a PAC
algorithm, i.e.:
\begin{align*} 
    \mathop{\mathbb{P}}\limits_{\cM,\pie,\mathfrak{A}}\Big(
        \sup\limits_{r\in\cR} \Big|\comp(r)-
        \widehat{\cC}(r)\Big|\le\epsilon\Big)\ge 1-\delta.
\end{align*}
Similarly, in the setting with $[L,U]$-suboptimal expert (for arbitrary $U\ge
L\ge 0$) where the (non)compatibility $\overline{\cC}$ considered is $\complu$
defined in Eq. \eqref{eq: reward compatibility suboptimal}, \caty computes
estimate $\widehat{\cC}(r)$ for input reward $r\in\cR$ as:
\begin{align}\label{eq: caty estimate comp suboptimal online}
    \widehat{\cC}(r)=\min_{x\in[L,U]}\Big|x-(\widehat{J}^*(r) -
    \widehat{J}^E(r))\Big|,
\end{align}
irrespective of the class of MDP. Note that this quantity can be computed
in constant time since, for any $y\in\RR$:
\begin{align}\label{eq: comp subopt using indic}
    \min_{x\in[L,U]}|x-y|=\begin{cases}
        L-y & \text{if } y< L\\
        0 & \text{if } y\in[L,U]\\
        y-U & \text{if } y>U\\
    \end{cases}.
\end{align}
Again, if Eq. \eqref{eq: requirement rfe} and Eq. \eqref{eq: requirement
estimate JE} hold, by using a union bound and the fact that the difference of
two minima can be upper bounded by the maximum of the difference, i.e., that for
any $y,y'\in\RR$:
\begin{align*}
    \Big|\min\limits_{x\in[L,U]}|x-y|-\min\limits_{x\in[L,U]}|x-y'|\Big|
    \le |y-y'|,
\end{align*}
then it is easy to see
that \caty is a PAC
algorithm, i.e.:
\begin{align*} 
    \mathop{\mathbb{P}}\limits_{\cM,\pie,\mathfrak{A}}\Big(
        \sup\limits_{r\in\cR} \Big|\comp^{L,U}(r)-
        \widehat{\cC}(r)\Big|\le\epsilon\Big)\ge 1-\delta.
\end{align*}
Finally, \caty applies the potentially negative threshold $\eta$ to the estimate
of (non)compatibility $\widehat{\cC}(r)$ to perform the classification. Observe
that $\Delta\in\RR_{\ge0}$ is the threshold in the definition of IRL
classification problem of Definition \ref{def: IRL problem}, while $\eta\in\RR$
is the actual threshold applied by \caty. As explained in the previous section,
$\eta$ can be different from $\Delta$ to trade-off the amount of false
negative/positive.
\begin{remark}
To prove that \caty is a PAC algorithm, the function approximation setting
considered must satisfy two properties. First, it must admit a RFE algorithm for
which Eq. \eqref{eq: requirement rfe} holds for a ``small'' $N$. Next, it must
guarantee that Eq. \eqref{eq: requirement estimate JE} holds with a finite value
of $N^E$ at most polynomial in the dimensions of the problem. If both these
conditions are satisfied, then, as explained in this section, \caty can be shown
to be a PAC algorithm in this function approximation setting, with a sample
complexity of $N,N^E$.
\end{remark}

\begin{algorithm}[t!]
        \caption{\caty - Classification phase
        }\label{alg: caty online classification}
        \DontPrintSemicolon
        \KwData{
        Expert data
        $\cD^E$,
        exploration data $\cD$,
        reward to classify $r\in\cR$,
        threshold $\eta$
        }
        \nonl \texttt{// Estimate the expert's performance $\widehat{J}^E(r)$:}\;
        $\widehat{J}^E(r)\gets\frac{1}{\tau^E}\sum\limits_{i\in\dsb{\tau^E}}\sum\limits_{h\in\dsb{H}}r_h(s_h^{E,i},a_h^{E,i})$\label{line: estimate JE online}\;
        \nonl \texttt{// Estimate the optimal performance $\widehat{J}^*(r)$:}\;
        $\widehat{J}^*(r)\gets$ RFE\_Planning($\cD,r$)\label{line: estimate Jstar online}\;
    \nonl \texttt{// Classify the reward:}\;
    \begin{tcolorbox}[enhanced, attach boxed title to top
        right={yshift=-4.5mm,yshifttext=-1mm},
            colframe=brightGreen,colbacktitle=brightGreen,colback=white,
            title=optimal expert,fonttitle=\bfseries,
            boxed title style={size=small, sharp corners}, sharp corners ,boxsep=-1.5mm]
            $\widehat{\mathcal{C}}(r)\gets \widehat{J}^*(r) - \widehat{J}^E(r)$\label{line: comp
            online optimal}\;
          \end{tcolorbox}
          \begin{tcolorbox}[enhanced, attach boxed title to top
            right={yshift=-4.5mm,yshifttext=-1mm},
            colframe=brightGreen,colbacktitle=brightGreen,colback=white,
                title=suboptimal expert,fonttitle=\bfseries,
            boxed title style={size=small, sharp corners}, sharp corners
            ,boxsep=-1.5mm]\small
            $\widehat{\cC}(r)\gets\min_{x\in[L,U]}\big|x-(\widehat{J}^*(r) -
\widehat{J}^E(r))\big|$\label{line: comp online suboptimal}\;
              \end{tcolorbox}
        $class \gets True$ if $\widehat{\cC}(r)\le \eta$ else
        $False$\label{line: class online}\;
        Return $class$
        \end{algorithm}


\subsubsection{Tabular MDPs and Tabular MDPs with Linear Rewards}

In tabular MDPs and tabular MDPs with linear rewards, for exploration (Line
\ref{line: rfe exploration} of Algorithm \ref{alg: caty online exploration}), we
let \caty execute two different RFE algorithms depending on the amount of
rewards to classify $\cR$. Specifically, if $|\cR|$ is a ``small'' constant
w.r.t. to the size of the MDP, i.e., if it holds that $S+ \log \frac{1}{\delta}
\ge |\cR|\log \frac{|\cR|}{\delta}$ (see Theorem \ref{thr: upper bound caty
online}), then we run algorithm BPI-UCBVI \citep{menard2021fast} for each reward
$r \in \cR$. Otherwise we execute algorithm RF-Express \citep{menard2021fast}
once. Consequently, the computation of $\widehat{J}^*(r)$ (Line \ref{line:
estimate Jstar online} of Algorithm \ref{alg: caty online classification})
depends on the exploration algorithm adopted. Concerning BPI-UCBVI, \caty
considers:\footnote{In case \caty executes BPI-UCBVI as many times as there are
rewards in $r\in\cR$, for the estimate of $\widehat{J}^*(r)$, we, of course,
consider the $\widetilde{Q}$ obtained by exploring with the specific reward
$r$.}
\begin{align}\label{eq: caty choice Jstarhat bpiucbvi}
    \widehat{J}^*(r)=\max\limits_{a\in\cA}\widetilde{Q}_1^{\tau+1}(s_0,a),
\end{align}
where $\widetilde{Q}$ is the upper confidence bound on the optimal $Q$-function
constructed by BPI-UCBVI (see Algorithm 2 of \citet{menard2021fast}), $\cA$ is
the action space of the considered MDP without reward
$\cM=\tuple{\cS,\cA,H,d_0,p}$, $s_0$ such that $d_0(s_0)=1$ is the initial
state, and $\tau$ is the number of iterations of algorithm BPI-UCBVI.
With regards to RF-Express, \caty executes the \emph{Backward Induction}
algorithm (see Section 4.5 of \citet{puterman1994markov}) to compute the optimal
$V$-function $V^*_h(\cdot;p,r)$ at all $h$ in the MDP
$\tuple{\cS,\cA,H,d_0,\widehat{p},r}$, where $\widehat{p}\approx p$ is the
empirical estimate of the transition model of $\cM$ constructed with the data
$\cD$ gathered at exploration phase by RF-Express (see Algorithm 1 of
\citet{menard2021fast}). Then, \caty takes:
\begin{align}\label{eq: caty choice Jstarhat rfexpress}
    \widehat{J}^*(r)=V^*_1(s_0;\widehat{p},r),
\end{align}
where $s_0$ such that $d_0(s_0)=1$ is the initial state.

\subsubsection{Linear MDPs}

In Linear MDPs, \caty uses algorithm RFLin \citep{wagenmaker2022noharder} to
explore the environment and to construct an estimate of the optimal utility
$\widehat{J}^*(r)$ for any input reward $r\in\fR$:
\begin{align}\label{eq: caty choice Jstarhat rflin}
    \widehat{J}^*(r)=V_1(s_0),
\end{align}
where function $V$ is the upper bound to the optimal value function defined at
line 9 of Algorithm 2 of \citet{wagenmaker2022noharder} (RFLin-Plan), and $s_0$
is the initial state.

\subsection{Sample Complexity}
The following result shows that \caty is a PAC algorithm for IRL classification
in the online setting (Definition \ref{def: pac algorithm online}) for all the
classes of MDP and definitions of (non)compatibility mentioned earlier.

\begin{restatable}[Sample Complexity of
\caty]{thr}{upperboundtabular}\label{thr: upper bound caty online} Assume that
there is a single initial state. Let $U\ge L\ge 0$ be arbitrary and let
$\epsilon>0,\delta\in(0,1)$. Then \caty executed with $\eta=\Delta$ is
$(\epsilon,\delta)$-PAC for IRL classification in the online setting for both
the optimal expert and the $[L,U]$-suboptimal expert settings, with a sample
complexity upper bounded by:
\[
    \begin{array}{lll}
    \text{Tabular MDPs:}&
    \displaystyle \tau^E\le
    \widetilde{\mathcal{O}}\Big(\frac{H^{3}SA}{\epsilon^2}\log\frac{1}{\delta}\Big), &
    \scalebox{0.94}{$  \displaystyle \tau\le \widetilde{\mathcal{O}}\Big(\frac{H^3SA}{\epsilon^2}
    \Bigr{S+ \log \frac{1}{\delta}}\Big),$}\\[.35cm]
    \text{Tabular MDPs with linear rewards:}&
    \displaystyle \tau^E\le \widetilde{\mathcal{O}}\Big(\frac{H^{3}d}{\epsilon^2}
    \log\frac{1}{\delta}\Big),& 
    \scalebox{0.94}{$  \displaystyle \tau\le \widetilde{\mathcal{O}}\Big(\frac{H^3SA}{\epsilon^2}
        \Bigr{S+ \log \frac{1}{\delta}}\Big),$}\\[.35cm]
    \text{Linear MDPs:}&
    \displaystyle \tau^E\le \widetilde{\mathcal{O}}\Big(\frac{H^{3}d}{\epsilon^2}
    \log\frac{1}{\delta}\Big),&
          \displaystyle \tau\le \widetilde{\mathcal{O}}\Big(\frac{H^5d}{\epsilon^2}
          \Big(d+\log\frac{1}{\delta}\Big)\Big).
    \end{array}
\]
If $|\cR|\log(|\cR|/\delta)\le S+ \log (1/\delta)$, then \caty achieves the
following improved rate in both tabular MDPs and tabular MDPs with linear
rewards, in both the optimal and suboptimal expert settings:
\begin{align*}
    \tau^E\le \widetilde{\mathcal{O}}\Big(\frac{H^2}{\epsilon^2}
    \log\frac{|\cR|}{\delta}\Big),\qquad
    \tau\le \widetilde{\mathcal{O}}\Big(\frac{H^3SA}{\epsilon^2}
        |\cR|\log\frac{|\cR|}{\delta}\Big).
\end{align*}
\end{restatable}

Some observations are in order.
First, we remark that the assumption of a single initial state is common and not
restrictive \citep{menard2021fast}.
The rate dependent on $|\cR|$ that permits to improve over $S+ \log (1/\delta)$,
i.e., over a $S^2$ dependency, is possible in tabular MDPs and tabular MDPs with
linear rewards, where executing $|\cR|$ times the algorithm BPI-UCBVI requires
less samples than running RF-Express once, i.e., as long as
$|\cR|\log(|\cR|/\delta)\le S+ \log (1/\delta)$. Instead, we conjecture that the
$d^2$ dependence is unavoidable in Linear MDPs because of the lower bound for
BPI in \citet{wagenmaker2022noharder}.
Observe that the bounds do not depend on the classification threshold $\Delta$
as long as we set $\eta=\Delta$. Moreover, it is remarkable that both the
settings with optimal and $[L,U]$-suboptimal expert enjoy the same sample
complexity, which does not depend on $L,U$.
In tabular MDPs with deterministic expert, one might use the results in
\citet{Xu2023ProvablyEA} to reduce the rate of $\tau^E$ at the price of
increasing $\tau$ with additional logarithmic factors. However, in
\citet{lazzati2024scaleinverserllarge}, we present a lower bound to $\tau$ for
tabular MDPs in the setting with optimal expert that demonstrates that the upper
bound provided by \caty is unimprovable up to logarithmic factors when
$|\cR|\log(|\cR|/\delta)> S+ \log (1/\delta)$, thus showing that \caty is
\emph{minimax optimal} in tabular MDPs with optimal expert.

In light of the result in Theorem \ref{thr: upper bound caty online}, we
conclude that the \emph{reward compatibility} framework allows the
\emph{practical} development\footnote{Clearly, \caty can be implemented in
practice, since it considers a single reward at a time instead of computing the
full feasible set.} of \emph{sample efficient} algorithms (e.g., \caty) in
Linear MDPs with large/continuous state spaces.
\begin{proofthr}{thr: upper bound caty online}
As explained in Section \ref{subsec: algorithm caty online}, to show that \caty
is PAC for the three function approximation settings considered in the statement
of the theorem, we have to show that Eq. \eqref{eq: requirement rfe} and Eq.
\eqref{eq: requirement estimate JE} hold with the mentioned number of samples
$N,N^E$.

Let us begin by analyzing $N^E$ for tabular MDPs. When $|\cR|$ is finite, we
have:
\begin{align*}
    \mathop{\mathbb{P}}\limits_{\cM,\pie}\Big(
            \exists r\in\cR:\;\Big|J^{\pie}(r;p)-\widehat{J}^E(r)\Big|>
            \frac{\epsilon}{2}\Big)&\markref{(1)}{\le}
            \sum\limits_{r\in\cR}
            \mathop{\mathbb{P}}\limits_{\cM,\pie}\Big(
            \Big|J^{\pie}(r;p)-\widehat{J}^E(r)\Big|>
            \frac{\epsilon}{2}\Big)\\
            &\markref{(2)}{=}
            \sum\limits_{r\in\cR}
            \mathop{\mathbb{P}}\limits_{\cM,\pie}\Big(
            \E[\widehat{J}^E(r)]-\widehat{J}^E(r)
            \Big|>
            \frac{\epsilon}{2}\Big)\\
            &\markref{(3)}{\le}
                \sum\limits_{r\in\cR}2e^{\frac{-\tau^E\epsilon^2}{2H^2}}\\
            &\markref{(4)}{=}\frac{\delta}{2},
\end{align*}
where at (1) we apply a union bound, at (2) we note that
$J^{\pie}(r;p)=\E_{\cM,\pie}[\widehat{J}^E(r)]\eqqcolon\E[\widehat{J}^E(r)]$ is
the expected value of $\widehat{J}^E(r)$, at (3) we apply the Hoeffding's
inequality, and at (4) we set
$\delta/2=|\cR|2e^{\frac{-\tau^E\epsilon^2}{2H^2}}$, so that:
\begin{align*}
    \tau^E\ge \frac{2H^2}{\epsilon^2}\ln\frac{4|\cR|}{\delta}.
\end{align*}
For any choice of $\cR$ we have:
\begin{align*}
    \mathop{\mathbb{P}}\limits_{\cM,\pie}\Big(
            \exists r\in\cR:\;\Big|J^{\pie}(r;p)-\widehat{J}^E(r)\Big|>
            \frac{\epsilon}{2}\Big)
            &\le
            \mathop{\mathbb{P}}\limits_{\cM,\pie}\Big(
            \exists r\in\fR:\;\Big|J^{\pie}(r;p)-\widehat{J}^E(r)\Big|>
            \frac{\epsilon}{2}\Big)\\
            &\markref{(5)}{=}
            \scalebox{0.98}{$  \displaystyle\mathop{\mathbb{P}}\limits_{\cM,\pie}\Big(
            \exists r\in\fR_{\text{vertices}}:\;\Big|J^{\pie}(r;p)-\widehat{J}^E(r)\Big|>
            \frac{\epsilon}{2}\Big)$}\\
            &\markref{(6)}{\le}
            \sum\limits_{r\in\fR_{\text{vertices}}}
            \mathop{\mathbb{P}}\limits_{\cM,\pie}\Big(
            \Big|J^{\pie}(r;p)-\widehat{J}^E(r)\Big|>
            \frac{\epsilon}{2}\Big)\\
            &\markref{(7)}{=}
            \sum\limits_{r\in\fR_{\text{vertices}}}
            \mathop{\mathbb{P}}\limits_{\cM,\pie}\Big(
            \E[\widehat{J}^E(r)]-\widehat{J}^E(r)
            \Big|>
            \frac{\epsilon}{2}\Big)\\
            &\markref{(8)}{\le}
                \sum\limits_{r\in\fR_{\text{vertices}}}2e^{\frac{-\tau^E\epsilon^2}{2H^2}}\\
            &\markref{(9)}{=}\frac{\delta}{2},
\end{align*}
where at (5) we make the same observation as in the proof of Lemma 6 of
\citet{Shani2021OnlineAL}: $\fR$ is the $SAH$-dimensional hypercube, thus all
the rewards $r\in\fR$ can be written as a convex combination of rewards in
$\fR_{\text{vertices}}\coloneqq\{r\in\fR\,|\, \forall
(s,a,h)\in\SAH:\,r_h(s,a)\in\{-1,+1\}\}$. Thus, to upper bound the difference:
\begin{align*}
    \sup\limits_{r\in\fR}\Big|J^{\pie}(r;p)-\widehat{J}^E(r)\Big|=\sup\limits_{r\in\fR}\Big|
    \sum\limits_{(s,a,h)\in\SAH}\bigr{d^{p,\pie}_h(s,a)-
    \widehat{d}^E_h(s,a)}r_h(s,a)\Big|,
\end{align*}
where $\widehat{d}^E_h(s,a)$ is the sample mean of $d^{p,\pie}_h(s,a)$ using
$\cD^E$, it suffices to consider just the rewards in $\fR_{\text{vertices}}$,
since the quantity is maximized by $r_h(s,a)=+1$ when the difference
$d^{p,\pie}_h(s,a)- \widehat{d}^E_h(s,a)\ge0$, and by $r_h(s,a)=-1$ otherwise.
At (6) we apply a union bound, at (7) we recognize the expectation, at (8) we
apply Hoeffding's inequality, and at (9) we set
$\delta/2=2^{SAH}2e^{\frac{-\tau^E\epsilon^2}{2H^2}}$, since
$|\fR_{\text{vertices}}|=2^{SAH}$, i.e., the number of vertices in the
$n$-dimensional hypercube is $2^n$. Thus:
\begin{align*}
    \tau^E\ge \frac{2SAH^3}{\epsilon^2}\ln\frac{4}{\delta}.
\end{align*}
Concerning tabular MDPs with linear rewards and Linear MDPs, note that the
derivation is exactly the same, with the only difference that the reward
functions are now $d$-dimensional for all $h\in\dsb{H}$, thus the union bound
has to be computed over a $d$-dimensional hypercube, obtaining:
\begin{align*}
    \tau^E\ge \frac{2dH^3}{\epsilon^2}\ln\frac{4}{\delta}.
\end{align*}

Let us now analyze the number of samples $N$. In tabular MDPs and tabular MDPs
with linear rewards, when $\cR$ is finite and satisfies
$|\cR|\log(|\cR|/\delta)\le S+ \log (1/\delta)$, then \caty executes algorithm
BPI-UCBVI of \cite{menard2021fast} for every $r\in\cR$, and it sets:
\begin{align*}
    \widehat{J}^*(r)=\max\limits_{a\in\cA}\widetilde{Q}_1^{\tau+1}(s_0,a),
\end{align*}
as explained in Eq. \eqref{eq: caty choice Jstarhat bpiucbvi}. As shown by the
authors \citep{menard2021fast} in their proof of Lemma 2, it holds that:
\begin{align*}
    J^*(r;p)-J^{\widehat{\pi}}(r;p)\le \max\limits_{a\in\cA}\widetilde{Q}_1^{\tau+1}(s_0,a)
    -J^{\widehat{\pi}}(r;p),
\end{align*}
where $\widehat{\pi}$ is the actual output of BPI-UCBVI. Thus, for any $\epsilon'>0$:
\begin{align*}
    J^*(r;p)-J^{\widehat{\pi}}(r;p)\le \epsilon'\implies \max\limits_{a\in\cA}\widetilde{Q}_1^{\tau+1}(s_0,a)-
    J^*(r;p)\le\epsilon',
\end{align*}
since $\max\limits_{a\in\cA}\widetilde{Q}_1^{\tau+1}(s_0,a)\ge J^*(r;p)$ under
their good event. BPI-UCBVI guarantees that, at each execution with a reward
$r\in\cR$, with:
\begin{align*}
    \tau\le\widetilde{\mathcal{O}}\Big(\frac{H^3SA}{\epsilon^2}\ln\frac{|\cR|}{\delta}\Big),
\end{align*}
it holds that (see Theorem 2 of \citet{menard2021fast}):
\begin{align*}
    \mathop{\mathbb{P}}\limits_{\cM,\pie,\mathfrak{A}}\Big(
            \Big|J^*(r;p)-\widehat{J}^*(r)\Big|\le
            \frac{\epsilon}{2}\Big)\ge 1-\frac{\delta}{2|\cR|}.
\end{align*}
Through a union bound, we obtain:
\begin{align*}
    \mathop{\mathbb{P}}\limits_{\cM,\pie,\mathfrak{A}}\Big(
        \exists r\in\cR:\, \Big|J^*(r;p)-\widehat{J}^*(r)\Big|>
        \frac{\epsilon}{2}\Big)&\le
        \sum\limits_{r\in\cR}\mathop{\mathbb{P}}\limits_{\cM,\pie,\mathfrak{A}}\Big(
            \Big|J^*(r;p)-\widehat{J}^*(r)\Big|>
            \frac{\epsilon}{2}\Big)
        \\
        &\le \frac{\delta}{2}.
\end{align*}
In tabular MDPs and tabular MDPs with linear rewards, when \caty executes
RF-Express as subroutine, and it sets, for all $r\in\cR$ (Eq. \eqref{eq: caty
choice Jstarhat rfexpress}):
\begin{align*}
    \widehat{J}^*(r)=V^*_1(s_0;\widehat{p},r),
\end{align*}
since $V^*_1(s_0;\widehat{p},r)$ is the exact optimal performance in the MDP
with dynamics $\widehat{p}$ since it has been compute through Backward Induction
\citep{puterman1994markov}, then, thanks to Theorem 1 of \citet{menard2021fast},
we have that Eq. \eqref{eq: requirement rfe} holds
with:
\begin{align*}
    \tau\le \widetilde{\mathcal{O}}
    \Big(\frac{H^3SA}{\epsilon^2}\Bigr{S+ \log \frac{1}{\delta}}\Big).
\end{align*}
Therefore, we prefer to execute BPI-UCBVI for $|\cR|$ times instead of
RF-Express once if $\cR$ satisfies (modulo some constants):
\begin{align*}
    S+ \log \frac{1}{\delta}\ge|\cR|\log \frac{|\cR|}{\delta}.
\end{align*}
Finally, in Linear MDPs, by using algorithm RFLin
\citep{wagenmaker2022noharder}, and making the estimate (Eq. \eqref{eq: caty
choice Jstarhat rflin}):
\begin{align*}
    \widehat{J}^*(r)=V_1(s_0),
\end{align*}
then, Theorem 1 of \citet{wagenmaker2022noharder} guarantees that Eq. \eqref{eq:
requirement rfe} holds with (we omit linear terms in $1/\epsilon$):
\begin{align*}
    \tau\le\widetilde{\mathcal{O}}\Big(\frac{H^5d}{\epsilon^2}\big(d+\log\frac{1}{\delta}\big)\Big).
\end{align*}
This concludes the proof.
\end{proofthr}

\section{Learning Compatible Rewards in the Offline Setting}\label{sec: offline}

In this section, we analyse the reward compatibility framework in the offline
setting for tabular MDPs.

\subsection{Problem Setting}

In many applications, IRL is better framed as an \emph{offline} problem, in
which there is no possibility to actively exploring the environment to improve
our estimates. For this reason, in this section, we consider the offline
scenario in which we are given a batch expert's dataset
$\cD^E=\{\tuple{s_1^{E,i},a_1^{E,i},\dotsc,s_H^{E,i},a_H^{E,i},s_{H+1}^{E,i}}\}_{i\in\dsb{\tau^E}}$
of $\tau^E$ state-action trajectories collected by the expert's policy $\pi^E$
in an MDP without reward $\cM=\tuple{\cS,\cA,H,d_0,p}$ (with unknown
dynamics $p$), and an additional batch dataset
$\cD^b=\{\tuple{s_1^{b,i},a_1^{b,i},\dotsc,s_{H}^{b,i},a_{H}^{b,i},s_{H+1}^{b,i}}\}_{i\in\dsb{\tau^b}}$
of $\tau^b$ state-action trajectories collected by executing a behavioral policy
$\pi^b$ in the same MDP without reward of the expert $\cM$.

Comparing with the \emph{online} setting presented in Section \ref{subsec:
problem setting online}, we still have a batch dataset $\cD^E$ that gives us
information on $\pi^E$ (and its occupancy measure $d^{p,\pie}$). However,
instead of being allowed to explore the environment at will to construct an
estimate of the transition model $p$, we now have to estimate $p$ using the
trajectories in the new batch dataset $\cD^b$.

We make two remarks to clarify the importance of using two datasets.
\begin{remark}
    As we did in the \emph{online} setting, we will \emph{not} use the data in
    $\cD^E$ to improve our estimate of the transition model $p$. The reason is
    that mixing the data would \emph{unnecessarily} complicate the theoretical
    analysis without any significant advantage. Nevertheless, note that, in
    practice, using all the samples in $\cD^E\cup\cD^b$ to estimate $p$ might
    improve the performance of the algorithms.
\end{remark}
\begin{remark}
    The requirement of \emph{two datasets} is \emph{not necessary}, although it
    can be useful. In fact, in many applications, the expert's policy $\pi^E$ is
    deterministic or moderately stochastic. Thus, it provides a limited coverage of
    the state-action space, preventing us from constructing estimates of the
    transition model $p$ in portions of the space that are not reached by
    $\pi^E$. To avoid this, we can consider an additional dataset $\cD^b$ collected
    by a potentially more explorative (stochastic) policy $\pi^b$. Note that
    this is a generalization of the common setting with only $\cD^E$, that we
    recover if we take $\pi^b=\pi^E$. In the following, no restriction is made
    on $\pi^b$.
\end{remark}


\subsection{Reward compatibility}

We now extend the reward compatibility approach to this setting.

\paragraph{Non-learnability of the (non)compatibility.}

The major difficulty of the offline setting is that, even in the limit of
infinite trajectories in the batch dataset $\cD^b$, the behavioral policy may
cover only a portion $\cZ^{p,\pi^b}\subset\SAH$ of the space, preventing the
estimation of the transition model in triples $(s,a,h)\notin \cZ^{p,\pi^b}$.
Observe that, whatever reward $r\in\fR$ we consider, both notions of
(non)compatibility $\compr$ (Definition \ref{def: reward compatibility optimal})
and $\comprlu$ (Definition \ref{def: reward compatibility suboptimal}) depend on
the optimal performance $J^*(r;p)$ that can be achieved under $r$. Intuitively,
since $J^*(r;p)$ depends on the transition model $p$ at all the triples
$(s,a,h)\in\SAH$ (reachable by at least one policy from the initial state
$s_0$), and since we do not have information on the transition model in some of
these triples because of the partial coverage of $\pi^b$, then estimating the
(non)compatibility of $r$ is not feasible in the offline
setting:\footnote{Observe that also the notion of feasible set is not
``learnable'' in this setting, as we have shown in Theorem C.1 of
\citet{lazzati2024offline}.}
\begin{thr}[Non-learnability of $\comp$ and $\complu$ in the offline setting]
    \label{thr: comp non learnable offline}
    Let $U\ge L\ge 0$ be arbitrary. Let $\rho:\RR\times\RR\to\RR_{\ge0}$ be an
    arbitrary metric between scalars. Let $\fA$ be any algorithm that aims to
    estimate the (non)compatibilities in Definition \ref{def: reward
    compatibility optimal} and \ref{def: reward compatibility suboptimal}. Then,
    there exists an MDP without reward $\cM=\tuple{\cS,\cA,H,d_0,p}$, an
    expert's policy $\pi^E$, and a behavioral policy $\pi^b$, such that, even if
    the batch datasets $\cD^E\sim\pi^E$ and $\cD^b\sim\pi^b$ contain an infinite
    amount of trajectories, there are $\epsilon,\delta\in(0,1)$ and $r\in\fR$
    for which:
    \begin{align*}
        \mathop{\P}\limits_{\cM,\pi^E,\pi^b}\Bigr{
            \rho\Bigr{
                \overline{\cC}(r),\widehat{\cC}(r)
            } \ge \epsilon
        } \ge \delta,
    \end{align*}
    where $\P_{\cM,\pi^E,\pi^b}$ denotes the probability measure induced by
    $\pi^E$ and $\pi^b$ in $\cM$, $\overline{\cC}(r)$ is any of $\compr$ or
    $\comprlu$, and $\widehat{\cC}(r)$ is the estimate of $\overline{\cC}(r)$
    computed by $\fA$ using only the data in $\cD^E$ and $\cD^b$.
\end{thr}
\begin{proof}
    We begin by considering the setting with optimal expert.

    Consider the MDP without reward pictured below, where $s_0$ is the initial
    state, there are two actions $a_1,a_2$ in each state, and the horizon is $H=2$.
    From state $s_0$, action $a_1$ brings deterministically to $s_1$, while
    action $a_2$ brings to $s_2$ with probability $q\in[0,1]$, and to $s_1$ with
    probability $1-q$:
    \begin{figure}[h!]
        \centering
        \begin{tikzpicture}[node distance=3.5cm]
        \node[state,initial] at (0,0) (s0) {$s_0$};
        \node[state] at (3,1.5) (s1) {$s_1$};
        \node[state] at (3,-1.5) (s2) {$s_2$};
        \node[state, draw=none] at (5,1.5) (ss1) {};
        \node[state, draw=none] at (5,-1.5) (ss2) {};
        \node[draw=none,fill=black] at (2,-0.6) (ss3) {};
        \draw (s0) edge[->, solid, above, sloped] node{\scriptsize$a_1$} (s1);
        \draw (s0) edge[-, solid, above, sloped] node{\scriptsize$a_2$} (ss3);
        \draw (ss3) edge[->, solid, above, sloped] node{\scriptsize$1-q$} (s1);
        \draw (ss3) edge[->, solid, above] node{\scriptsize$q$} (s2);
        \draw (s1) edge[->, solid, above, sloped] node{\scriptsize$a_1,a_2$} (ss1);
        \draw (s2) edge[->, solid, above, sloped] node{\scriptsize$a_1,a_2$} (ss2);
      \end{tikzpicture}
      \end{figure}

      We denote as $\cM_0,\cM_1$ the MDPs without reward corresponding to,
      respectively, $q=0$ and to $q=1$. Let $\pi^E=\pi^b$ be the policies that
      play action $a_1$ at every state.

      We now show that there exists at least a reward $r\in\fR$ for which no
      algorithm $\fA$ can estimate the (non)compatibility
      $\overline{\cC}_{\cM_0,\pi^E}(r)$ of reward $r$ in $\cM_0$ and $\cM_1$
      with arbitrary accuracy and failure probability $\epsilon,\delta\in(0,1)$
      using only $\cD^E\sim\pi^E,\cD^b\sim\pi^b$, even if they contain an
      infinite amount of samples.

      Let $r\in\fR$ be the reward function that assigns reward $1$ to all
      actions played in states $s_0$ and $s_2$, and reward $0$ to actions played
      in state $s_1$. The (non)compatibilities of $r$ with $\pi^E$ in $\cM_0$
      and $\cM_1$ are:
      \begin{align*}
        &\overline{\cC}_{\cM_0,\pi^E}(r)=0,\qquad \overline{\cC}_{\cM_1,\pi^E}(r)=1.
      \end{align*}
      Since
      $\overline{\cC}_{\cM_0,\pi^E}(r)\neq\overline{\cC}_{\cM_1,\pi^E}(r)$, and
      since $\rho$ is a metric, then there exists a value $k>0$ such that:
      \begin{align*}
        \rho\Bigr{\overline{\cC}_{\cM_0,\pi^E}(r),\overline{\cC}_{\cM_1,\pi^E}(r)}=k.
      \end{align*}
      Since $\pi^E=\pi^b$ always play action $a_1$, then even with an infinite
      amount of samples, $\cD^E,\cD^b$ do not reveal any information on the
      transition model of action $a_2$ in $s_0$, thus algorithm $\fA$ cannot
      discriminate between environments $\cM_0,\cM_1$. By choosing
      $\epsilon<k/2,\delta>1/2$, we have that either the output
      $\widehat{\cC}(r)$ of $\fA$ satisfies, for any $\cM\in\{\cM_0,\cM_1\}$,
      $\P_{\cM,\pi^E,\pi^b}(\rho(\overline{\cC}_{\cM_0,\pi^E},\widehat{\cC}(r))<k/2)>
      1/2$ or
      $\P_{\cM,\pi^E,\pi^b}(\rho(\overline{\cC}_{\cM_1,\pi^E},\widehat{\cC}(r))<k/2)>
      1/2$, but not both since $\rho$ satisfies the triangle inequality. Thus,
      in at least one of the two instances $\cM_0,\cM_1$ any algorithm $\fA$
      satisfies the statement of the theorem. This concludes the proof for the
      setting with optimal expert.

      An analogous construction can be constructed for the setting with
      suboptimal expert.
\end{proof}

\paragraph{New notions of reward compatibility.}

Theorem \ref{thr: comp non learnable offline} shows that, because of the partial
coverage induced by $\pi^b$, the notions of (non)compatibility $\comp$ and
$\complu$ do not represent reasonable learning targets in the offline setting.
For this reason, the best we can do is to define an ``optimistic'' and a
``pessimistic'' extension of (non)compatibility, that represent, respectively,
the best and the worst possible value of (non)compatibility of the considered
reward given the available information:
\begin{defi}[Best and worst (non)compatibility]\label{def: best worst comp}
    Let $\cM=\tuple{\cS,\cA,H,d_0,p}$ be an MDP without reward, $\pi^E$ the
    expert's policy, and $\pi^b$ the behavioral policy. Let $\cZ\coloneqq\cZ^{p,\pi^b}$ be
    the portion of space covered by $\pi^b$. Then, given a notion of
    (non)compatibility $\overline{\cC}$, for any $r\in\fR$, we define the
    \emph{best (non)compatibility} $\overline{\cC}^b(r)$ of reward $r$ given
    partial coverage $\cZ$ as:
    \begin{align}\label{eq: best comp}
        \overline{\cC}^b(r)\coloneqq \min\limits_{p'\in[p]_{\equiv_\cZ}}
        \overline{\cC}(r).
    \end{align}
    Similarly, we define the \emph{worst (non)compatibility}
    $\overline{\cC}^w(r)$ of reward $r$ given partial coverage $\cZ$ as:
    \begin{align}\label{eq: worst comp}
        \overline{\cC}^w(r)\coloneqq \max\limits_{p'\in[p]_{\equiv_\cZ}}
        \overline{\cC}(r).
    \end{align}
\end{defi}
Intuitively, at best, the (non)compatibility of $r$ is
$\overline{\cC}^b(r)\le \overline{\cC}(r)$, and, at worst, it is
$\overline{\cC}^w(r)\ge \overline{\cC}(r)$.
In other words, we are proposing a best-/worst-case approach to cope with the
missing knowledge of the dynamics $p$ outside $\cZ^{p,\pi^b}$.
From now on, we use $\cZ\coloneqq\cZ^{p,\pi^b}$ to denote the portion of space
covered by $\pi^b$. Given any $r$, we denote respectively as
$\comproffb\coloneqq\min_{p'\in[p]_{\equiv_\cZ}}\compr$ and
$\comproffw\coloneqq\max_{p'\in[p]_{\equiv_\cZ}}\compr$ the best and worst
(non)compatibility of $r$ in the setting with optimal expert. Analogously, for
arbitrary $U\ge L\ge 0$, we denote respectively as
$\comprluoffb\coloneqq\min_{p'\in[p]_{\equiv_\cZ}}\comprlu$ and
$\comprluoffw\coloneqq\max_{p'\in[p]_{\equiv_\cZ}}\comprlu$ the best and worst
(non)compatibility of $r$ in the setting with $[L,U]$-suboptimal expert.

\paragraph{Each reward compatibility defines a feasible set.}

We observe that, analogously to what we have done in Section \ref{sec: the rewards
compatibility framework}, we can define two new notions of ``feasible set'' as
the sets of rewards with, respectively, zero best and worst (non)compatibility:
\begin{align*}
    &\cR^b\coloneqq\Bigc{r\in\fR\,\big|\,
    \overline{\cC}^b(r)=0
    },\\
    &\cR^w\coloneqq\Bigc{r\in\fR\,\big|\,
    \overline{\cC}^w(r)=0
    }.
\end{align*}
These sets satisfy the inclusion monotonicity property
\citep{lazzati2024offline}:
\begin{align*}
    \cR^b\subseteq \cR\subseteq \cR^w,
\end{align*}
where $\cR\coloneqq \bigc{r\in\fR\,\big|\, \overline{\cC}(r)=0 }$ is the
feasible set. If we consider the setting with optimal expert
$\overline{\cC}=\comp$, then we recover the notions of sub-/super-feasible set
that we analysed in \citet{lazzati2024offline}.

Observe that the notions of IRL classification problem (Definition \ref{def: IRL
problem}) and IRL algorithm (Definition \ref{def: IRL algorithm}) comply with
the new definitions. The only difference is that, now, we expect our IRL
algorithm to output two booleans, representing the classification carried out
based on both the best and worst (non)compatibilities (we use a single threshold
$\Delta$ for both).

We conclude this section with a remark. As explained in
\citet{lazzati2024offline} (Proposition 8.1) under the name of ``bitter
lesson'', in the setting with optimal expert, when we have only data collected
by a deterministic expert's policy $\pi^E$, i.e., $\cD^b=\cD^E$, then the
sub-feasible set $\cR_{\cM,\pie}^w\coloneqq \bigc{r\in\fR\,\big|\, \comproffw=0
}$ exhibits some degeneracy, since it contains only reward functions whose
greedy optimal action at all $(s,h)\in\cS^{p,\pie}$ is the expert's action
$\pie_h(s)$. However, reward compatibility can overcome this limitation, in the
following manner. Let $\cR_{\cM,\pie,\Delta}^w\coloneqq\bigc{r\in\fR\,\big|\,
\comproffw\le\Delta }$ be the set of rewards that should be positively
classified in the IRL classification problem with threshold $\Delta$. Clearly,
$\cR_{\cM,\pie}^w\subseteq\cR_{\cM,\pie,\Delta}^w$, and it is not difficult to
see that, if $\Delta>0$ strictly, then $\cR_{\cM,\pie,\Delta}^w$ contains
rewards that do not suffer from the ``bitter lesson'', i.e., do not make $\pi^E$
the greedy policy w.r.t. the immediate reward $r$.
The price to pay for the increased expressivity of the solution concept is an
increased (non)compatibility (i.e., error) in the learned rewards. In other
words, the larger the $\Delta$, the more the expressivity of the learned
rewards, but, at the same time, the larger their (non)compatibility.
  
\subsection{Learning Framework}

We are interested in solving the IRL classification problem in the offline
setting. For this purpose, we define an algorithm to be efficient if it provides
``good'' estimates for both the best and worst (non)compatibilities:

\begin{defi}[PAC Algorithm - Offline setting]\label{def: pac algorithm offline} Let
  $\epsilon>0,\delta\in(0,1)$, and let $\cD^E$ be a dataset of $\tau^E$ expert's
  trajectories and $\cD^b$ a dataset of $\tau^b$ behavioral
  trajectories. An algorithm $\mathfrak{A}$ is
  $(\epsilon,\delta)$-PAC for the IRL classification problem if:
  \begin{align*} 
      \mathop{\mathbb{P}}\limits_{\cM,\pie,\pi^b}\Big(
          \sup\limits_{r\in\cR} \Big|\overline{\cC}^b(r)-
          \widehat{\cC}^b(r)\Big|\le\epsilon\;
          \wedge\;\sup\limits_{r\in\cR} \Big|\overline{\cC}^w(r)-
          \widehat{\cC}^w(r)\Big|\le\epsilon\Big)\ge 1-\delta,
  \end{align*}
  where $\mathbb{P}_{\cM,\pie,\pi^b}$ is the joint probability measure induced
  by $\pie$ and $\pi^b$ in $\cM$, and $\widehat{\cC}^b,\widehat{\cC}^w$ are,
  respectively, the estimates of the best and worst (non)compatibilities
  $\overline{\cC}^b,\overline{\cC}^w$ computed by $\mathfrak{A}$. The
  \emph{sample complexity} is defined by the pair $(\tau^E,\tau^b)$.
  \end{defi}
  Observe that the PAC framework is general, and it can be instantiated for both
  the optimal and suboptimal expert settings. Let $\eta_b,\eta_w$ be the
  thresholds used by our learning algorithm $\fA$ to classify
  rewards,\footnote{Even though we have only one $\Delta$, we can use different
  thresholds $\eta_b,\eta_w$.} and define the sets of rewards positively
  classified as:
  \begin{align*}
    &\widehat{\cR}^b_{\eta_b}\coloneqq\Bigc{r\in\fR\,|\,\widehat{\cC}^b(r)\le\eta_b},\\
    &\widehat{\cR}^w_{\eta_w}\coloneqq\Bigc{r\in\fR\,|\,\widehat{\cC}^w(r)\le\eta_w}.
  \end{align*}
  Similarly, define the sets of rewards that should actually be positively classified as:
  \begin{align*}
    &\cR^b_{\Delta}\coloneqq\Bigc{r\in\fR\,|\,\overline{\cC}^b(r)\le\Delta},\\
    &\cR^w_{\Delta}\coloneqq\Bigc{r\in\fR\,|\,\overline{\cC}^w(r)\le\Delta}.
  \end{align*}
  It is not difficult to see that a PAC algorithm for the offline setting
  provides the same guarantees highlighted in Section \ref{subsec: online
  learning framework} between sets $\widehat{\cR}^b_{\eta_b},\cR^b_{\Delta}$ and
  $\widehat{\cR}^w_{\eta_w},\cR^w_{\Delta}$. In other words, accurately tuning
  the thresholds $\eta_b,\eta_w$ permits to trade-off the amount of ``false
  negatives''/``false positives''.

\subsection{Algorithm}

In this section, we present \catyoff (\catyofflong), a learning algorithm for
solving the IRL classification problem in the offline setting in \emph{tabular MDPs},
in both the settings with optimal and suboptimal expert. The pseudocode of the
algorithm is reported in Algorithm \ref{alg: caty offline}.

\begin{algorithm}[h!]
    \caption{\catyoff}\label{alg: caty offline}
    \DontPrintSemicolon
    \KwData{
    Expert dataset
    $\cD^E$, behavioral dataset
    $\cD^b$,
    classification threshold $\eta$,
    reward to classify $r\in\cR$
    }
    \nonl \texttt{// Estimate the expert's performance $\widehat{J}^E(r)$:}\;
    $\widehat{J}^E(r)\gets\frac{1}{\tau^E}\sum\limits_{i\in\dsb{\tau^E}}
    \sum\limits_{h\in\dsb{H}}r_h(s_h^{E,i},a_h^{E,i})$\label{line: JE
    last}\;
    \nonl \texttt{// Estimate the support $\cZ$ and the transition
    model $p$:}\;
    \uIf{$|\cR|=1$}{
        Randomly split $\cD^b=\{\omega^b_i\}_{i\in\dsb{\tau^b}}$ into $H$ datasets $\{\cD^b_h\}_{h\in\dsb{H}}$
        with $|\cD^b_h|=\floor{\tau^b/H}$\label{line: split datasets}\;
        $\widehat{\cZ}\gets \{(s,a,h)\in\SAH\,|\,
    \exists
    i\in\dsb{\tau^b}:\,\omega^b_i\in\cD^b_h\,\wedge\,(s_h^{b,i},a_h^{b,i})=(s,a)\}$\label{line: Z R1}\;
        $\widehat{p}_h(\cdot|s,a)\gets$ empirical estimate of $p$ from
        $\cD^b_h$ for all $(s,a,h)\in \widehat{\cZ}$ through Eq. \eqref{eq: estimate phat R1}\label{line: p R1}
        \;
    }
    \uElse{
        $\widehat{\cZ}\gets \{(s,a,h)\in\SAH\,|\, \exists
        i\in\dsb{\tau^b}:\,(s_h^{b,i},a_h^{b,i})=(s,a)\}$\label{line: Z Rall}
        \;
        $\widehat{p}_h(\cdot|s,a)\gets$ empirical estimate of $p$ from $\cD^b$
        for all $(s,a,h)\in \widehat{\cZ}$ through Eq. \eqref{eq: estimate phat Rall}\label{line: p Rall}\;
    }
    \nonl \texttt{// Estimate the optimal performances $\widehat{J}^*_m(r)$ and $\widehat{J}^*_M(r)$:}\;
    $\widehat{Q}^m_H(s,a),\widehat{Q}^M_H(s,a)\gets r_H(s,a)$  $\forall (s,a) \in \cS \times \cA$\label{line: EVI1}\;
\For{$h=H-1$ {\textnormal{\textbf{to}}} $1$}{
\For{$(s,a)\in\SA$}{
     $\widehat{Q}^M_h(s,a)\gets r_h(s,a)+\max\limits_{p' \in [\widehat{p}]_{\equiv_{\widehat{\cZ}}}}
    \sum\limits_{s'\in\cS}{p}'_h(s'|s,a)\max\limits_{a'\in\cA}\widehat{Q}^M_{h+1}(s',a')$\;
    $\widehat{Q}^m_h(s,a)\gets r_h(s,a)+\min\limits_{p' \in [\widehat{p}]_{\equiv_{\widehat{\cZ}}}}
    \sum\limits_{s'\in\cS}{p}'_h(s'|s,a)\max\limits_{a'\in\cA}\widehat{Q}^m_{h+1}(s',a')$\;
}
}
$\widehat{J}^*_m(r)\gets \max\limits_{a\in\cA}\widehat{Q}^m_1(s_0,a)$\;
$\widehat{J}^*_M(r)\gets \max\limits_{a\in\cA}\widehat{Q}^M_1(s_0,a)$\label{line: EVIlast}\;
\nonl \texttt{// Classify the reward:}\;
$\widehat{\Delta}_m(r) \gets \widehat{J}_m^*(r)-\widehat{J}^E(r)$\label{line:
estimate delta m}\;
$\widehat{\Delta}_M(r) \gets \widehat{J}_M^*(r)-\widehat{J}^E(r)$\label{line:
estimate delta M}\;
\begin{tcolorbox}[enhanced, attach boxed title to top
    right={yshift=-4.5mm,yshifttext=-1mm},
        colframe=vibrantBlue,colbacktitle=vibrantBlue,colback=white,
        title=optimal expert,fonttitle=\bfseries,
        boxed title style={size=small, sharp corners}, sharp corners ,boxsep=-1.5mm]\small
        $\widehat{\cC}^b(r)\gets\widehat{\Delta}_m(r)$\;
    $\widehat{\cC}^w(r)\gets\widehat{\Delta}_M(r)$\;
      \end{tcolorbox}\label{line: comp
      off optimal}
      \begin{tcolorbox}[enhanced, attach boxed title to top
        right={yshift=-4.5mm,yshifttext=-1mm},
        colframe=vibrantBlue,colbacktitle=vibrantBlue,colback=white,
            title=suboptimal expert,fonttitle=\bfseries,
        boxed title style={size=small, sharp corners}, sharp corners
        ,boxsep=-1.5mm]\small
        $\widehat{\cC}^b(r)\gets\max\bigc{\indic{\widehat{\Delta}_M(r)<L}(L-\widehat{\Delta}_M(r)),
        \indic{\widehat{\Delta}_m(r)>U}(\widehat{\Delta}_m(r)-U)}
        $\;
        $\widehat{\cC}^w(r)\gets\max\bigc{\indic{\widehat{\Delta}_m(r)<L}(L-\widehat{\Delta}_m(r)),
        \indic{\widehat{\Delta}_M(r)>U}(\widehat{\Delta}_M(r)-U)}
        $\;
          \end{tcolorbox}
          \label{line: comp off suboptimal}
          $class^b \gets True$ if
          $\widehat{\cC}^b(r)\le \eta$ else $False$\label{line: class best}\;
          $class^w \gets True$ if $\widehat{\cC}^w(r)\le \eta$ else
          $False$\label{line: class worst}\;
    Return $class^b,class^w$
    \end{algorithm}

\paragraph{Description of the algorithm.}
\catyoff takes in input a reward function $r\in\cR$ and a threshold $\Delta$,
and it aims to output two booleans meant to classify the reward based on
its best and worst (non)compatibilities.
In both the optimal and suboptimal expert settings, to compute an estimate of
the best and worst (non)compatibilities, \catyoff first estimates three different
quantities: the expert's performance $\widehat{J}^E(r)\approx J^{\pie}(r;p)$, the
best optimal performance $\widehat{J}^*_M(r)\approx
\max_{p'\in[p]_{\equiv_{\cZ}}}J^*(r;p')$, and the worst optimal performance
$\widehat{J}^*_m(r)\approx \min_{p'\in[p]_{\equiv_{\cZ}}}J^*(r;p')$ of the
considered reward $r\in\cR$.

\catyoff computes $\widehat{J}^E(r)$ as the empirical estimate (sample mean) of
the expert's performance $J^{\pie}(r;p)$ using dataset $\cD^E$ and reward $r$
(see Line \ref{line: JE last}), analogously to \caty.
To compute estimates $\widehat{J}^*_m(r)$ and $\widehat{J}^*_M(r)$, \catyoff
uses the empirical estimates of the support of the behavioral policy
distribution $\widehat{\cZ}\approx\cZ$ (Lines \ref{line: Z R1} and \ref{line: Z
Rall}),\footnote{Observe that we distinguish between the case in which $|\cR|=1$
and $|\cR|>1$ through the splitting of the dataset $\cD^b$ carried out at line
\ref{line: EVI1}. This passage is needed to obtain the sample complexity
guarantee in Theorem \ref{thr: bounds tabular off}. } and of the transition
model $\widehat{p}\approx p$ (Lines \ref{line: p R1} and \ref{line: p Rall}).
Specifically, the estimates of transition model at Line \ref{line: p R1} are
computed as:
\begin{align}\label{eq: estimate phat R1}
    \widehat{p}_h(s'|s,a)\coloneqq \frac{\sum_{i\in\dsb{\tau^b}}
    \indic{(s_h^{b,i},a_h^{b,i},s_{h+1}^{b,i})=(s,a,s')\wedge\omega^b_i\in\cD^b_h}}
    {\sum_{i\in\dsb{\tau^b}}
    \indic{(s_h^{b,i},a_h^{b,i})=(s,a)\wedge\omega^b_i\in\cD^b_h}}
    \;\forall(s,a,h)\in\widehat{\cZ},\forall s'\in\cS,
\end{align}
where we use $\{\omega^b_i\}_{i\in\dsb{\tau^b}}$ to denote the $\tau^b$
state-action trajectories contained into $\cD^b$, and splitted into the $H$
datasets $\{\cD^b_h\}_h$. Instead, the estimates of transition model at Line
\ref{line: p Rall} are computed as:
\begin{align}\label{eq: estimate phat Rall}
    \widehat{p}_h(s'|s,a)\coloneqq \frac{\sum_{i\in\dsb{\tau^b}}
    \indic{(s_h^{b,i},a_h^{b,i},s_{h+1}^{b,i})=(s,a,s')}}
    {\sum_{i\in\dsb{\tau^b}}
    \indic{(s_h^{b,i},a_h^{b,i})=(s,a)}}\qquad\forall(s,a,h)\in\widehat{\cZ},\forall s'\in\cS.
\end{align}
Then, \catyoff computes:
\begin{align*}
    &\widehat{J}^*_m(r)=\min\limits_{p'\in[\widehat{p}]_{\equiv_{\widehat{\cZ}}}}J^*(r;p'),\\
    &\widehat{J}^*_M(r)=\max\limits_{p'\in[\widehat{p}]_{\equiv_{\widehat{\cZ}}}}J^*(r;p'),
\end{align*}
by applying the Extended Value Iteration (EVI) algorithm \citep{auer2008nearoptimal}
(Lines \ref{line: EVI1}-\ref{line: EVIlast}).
Finally, \catyoff estimates the best and worst (non)compatibilities
$\widehat{\cC}^b(r)\approx \comproffb$ and $\widehat{\cC}^w(r)\approx
\comproffw$ in the setting with optimal expert as (see Line \ref{line: comp off
optimal}):
\begin{align*}
    &\widehat{\cC}^b(r)=\widehat{J}^*_m(r)-\widehat{J}^E(r),\\
    &\widehat{\cC}^w(r)=\widehat{J}^*_M(r)-\widehat{J}^E(r).
\end{align*}
In the setting with suboptimal expert, observe that the best and worst
(non)compatibilities can be re-written in a clearer form. To this aim, we need
some additional notation. For any $r\in\fR$, for any transition model $p'\in[p]_{\equiv_{\cZ}}$, we
define $\Delta(r;p')\coloneqq J^*(r;p')-J^{\pi^E}(r;p')$, and also
$\Delta_m(r)\coloneqq\min_{p'\in[p]_{\equiv_{\cZ}}} J^*(r;p')-J^{\pi^E}(r;p')$,
$\Delta_M(r)\coloneqq\max_{p'\in[p]_{\equiv_{\cZ}}} J^*(r;p')-J^{\pi^E}(r;p')$. Then:
\begin{prop}\label{prop: comp off subopt}
    For any $U\ge L\ge 0$, for any $r\in\fR$, it holds that:
    \begin{align*}
    &\comprluoffw=\max\Bigc{\indic{\Delta_m(r)<L}\bigr{L-\Delta_m(r)},\indic{\Delta_M(r)>U}\bigr{\Delta_M(r)-U}},\\
    &\comprluoffb=\max\Bigc{\indic{\Delta_M(r)<L}\bigr{L-\Delta_M(r)},\indic{\Delta_m(r)>U}\bigr{\Delta_m(r)-U}}.
    \end{align*}
\end{prop}
    \begin{proof}
        We begin with the worst (non)compatibility:
        \begin{align*}
            \comprluoffw&\coloneqq\max\limits_{p'\in[p]_{\equiv_\cZ}}\comprlu\\
            &=\max\limits_{p'\in[p]_{\equiv_\cZ}}\min\limits_{x\in[L,U]}
            \Big|x-\Bigr{J^*(r;p')-J^{\pie}(r;p')}\Big|\\
            &=\max\limits_{p'\in[p]_{\equiv_\cZ}}\min\limits_{x\in[L,U]}
            \Big|x-\Delta(r;p')\Big|\\
            &\markref{(1)}{=}\max\limits_{p'\in[p]_{\equiv_\cZ}}
            \begin{cases}
                L-\Delta(r;p'), & \text{if }\Delta(r;p')<L\\
                0, & \text{if }\Delta(r;p')\in[L,U]\\
                \Delta(r;p')-U, & \text{if }\Delta(r;p')>U
            \end{cases}\\
            &\markref{(2)}{=}
            \begin{cases}
                L-\Delta_m(r), & \text{if }\Delta_m(r)<L\\
                0, & \text{if }\forall p'\in [p]_{\equiv_\cZ}:\,\Delta(r;p')\in[L,U]\\
                \Delta_M(r)-U, & \text{if }\Delta_M(r)>U
            \end{cases}\\
            &=\max\Bigc{\indic{\Delta_m(r)<L}\bigr{L-\Delta_m(r)},\indic{\Delta_M(r)>U}\bigr{\Delta_M(r)-U}},
            \end{align*}
            where at (1) we use the observation in Eq. \eqref{eq: comp subopt
            using indic}, and at (2) we simply recognize the worst cases. An
            analogous derivation can be carried out for the best
            (non)compatibility.
    \end{proof}
Thus, at Line \ref{line: comp off suboptimal}, \catyoff simply applies the
formulas in Proposition \ref{prop: comp off subopt} replacing the values of
$\Delta_m(r),\Delta_M(r)$ with their estimates computed at lines \ref{line:
estimate delta m}-\ref{line: estimate delta M}.
Finally, at lines \ref{line: class best}-\ref{line: class worst}, \catyoff
performs the classification of the input reward.

\paragraph{Sample efficiency.}%
We have the following result on the sample complexity of \catyoff:
\begin{restatable}[Sample Complexity of
\catyoff]{thr}{upperboundtabularoff}\label{thr: bounds tabular off} Assume that
there is a single initial state. Let $U\ge L\ge 0$ be arbitrary and let
$\epsilon>0,\delta\in(0,1)$. Then \catyoff executed with $\eta=\Delta$ is
$(\epsilon,\delta)$-PAC for IRL classification in the offline (tabular) setting
for both the optimal expert and the $[L,U]$-suboptimal expert settings, with a
sample complexity upper bounded by:
        \begin{align*}
          \text{if }|\cR|=1:\quad&\tau^E\le
        \widetilde{\mathcal{O}}\Big(\frac{H^2}{\epsilon^2}\log\frac{1}{\delta}\Big),
        \qquad
        \tau^b\le \widetilde{\mathcal{O}}\Big(\frac{H^5\log^2\frac{|\cZ|}
        {\delta}}{\epsilon^2 d_{\min}^{p,\pi^b}}+
        \frac{\log\frac{1}{\delta}}{\log\frac{1}{1-d_{\min}^{p,\pi^b}}}\Big),\\
        \text{otherwise:}\quad&\tau^E\le
        \widetilde{\mathcal{O}}\Big(\frac{H^{3}SA}{\epsilon^2}\log\frac{1}{\delta}\Big),
        \qquad
        \tau^b\le \widetilde{\mathcal{O}}\Big(\frac{H^4\log\frac{|\cZ|}
        {\delta}}{\epsilon^2 d_{\min}^{p,\pi^b}}\Bigr{S+\log\frac{|\cZ|}
        {\delta}}+
        \frac{\log\frac{1}{\delta}}{\log\frac{1}{1-d_{\min}^{p,\pi^b}}}\Big),\\
        \end{align*}
        where $d_{\min}^{p,\pi^b}\coloneqq \min_{(s,a,h)\in\cZ}d^{p,\pi^b}_h(s,a)$.
    \end{restatable}
We make some comments.
First, note that the expert sample complexity $\tau^E$ coincides with that
provided in Theorem \ref{thr: upper bound caty online}, since \catyoff computes
the same estimate as \caty for $J^{\pi^E}(r;p)$.
Next, observe that the upper bound to $\tau^b$ is made of two terms, one that
displays a tight dependence on the desired accuracy $\epsilon$ and a dependence
of order $H^5$ or $H^4$ on the horizon, and another term dependent on the
minimum non-zero value of the visitation distribution $d_{\min}^{p,\pi^b}>0$,
that is needed to ensure $\widehat{\cZ}=\cZ$. More in detail, note that, since
$1/d_{\min}^{p,\pi^b}\le\cZ$, and since $\cZ=\SAH$ when $\pi^b$ covers the
entire space, then the dependence on $SA$ is hidden inside $d_{\min}^{p,\pi^b}$.
Also note that, when $|\cR|=1$, we need at most $\propto S$ samples, otherwise we
require $\propto S^2$ data.
We stress that \catyoff is computationally efficient since it just implements
EVI \citep{auer2008nearoptimal}, and thus it has a time complexity of order
$\cO(HS^2A)$. Nevertheless, notice that \catyoff does not scale to problems with
large/infinite state spaces even under the structure imposed by Linear MDPs,
because our definitions of best and worst (non)compatibility rely on the portion
$\cZ$ of the state-action space covered by $\pi^b$ (through
$d_{\min}^{p,\pi^b}$). We leave to future works the development of alternative
(non)compatibility notions more suitable for large-scale settings, that are able
to exploit the structure of the considered problem (e.g., Linear MDP). Simply
put, these notions should be based on a concept of coverage other than $\cZ$.
Finally, we note that, differently from the results on the feasible set in this
setting (see Theorem 5.1 of \citet{lazzati2024offline}), our Theorem \ref{thr:
bounds tabular off} does \emph{not} require the assumption that
$\cZ^{p,\pi^E}\subseteq\cZ^{p,\pi^b}$, providing an additional advantage of the
(non)compatibility framework over the feasible set.

To prove Theorem \ref{thr: bounds tabular off}, we begin by showing that
estimating $\widehat{\Delta}_m(\cdot)\approx\Delta_m(\cdot)$ and
$\widehat{\Delta}_M(\cdot)\approx\Delta_M(\cdot)$ accurately suffices also for
the setting with suboptimal expert. Indeed, in the setting with optimal expert, it is
immediate.
\begin{lemma}\label{lemma: suboptimal compatibility smaller delta offline}
    For any problem instance $\cM,\pi^E,\pi^b$, for any $U\ge L\ge 0$, for any
    $r\in\fR$, it holds that:
    \begin{align*}
        &\Big|\comprluoffw-\widehat{\cC}^w(r)\Big|\le
        \max\Bigc{\big|\Delta_m(r)-\widehat{\Delta}_m(r)\big|,
        \big|\Delta_M(r)-\widehat{\Delta}_M(r)\big|},\\
        &\Big|\comprluoffb-\widehat{\cC}^b(r)\Big|\le
        \max\Bigc{\big|\Delta_m(r)-\widehat{\Delta}_m(r)\big|,
        \big|\Delta_M(r)-\widehat{\Delta}_M(r)\big|}.
    \end{align*} 
\end{lemma}
\begin{proof}
    We prove the result only for the worst (non)compatibility. For the best
    (non)compatibility the proof is analogous. We can write:
    \begin{align*}
        \Big|\comprluoffw-\widehat{\cC}^w(r)\Big|&\markref{(1)}{=}
        \Big|\max\Bigc{\indic{\Delta_m(r)<L}\bigr{L-\Delta_m(r)},\indic{\Delta_M(r)>U}\bigr{\Delta_M(r)-U}}\\
        &\qquad\scalebox{0.96}{$  \displaystyle-\max\Bigc{\indic{\widehat{\Delta}_m(r)<L}\bigr{L-\widehat{\Delta}_m(r)},
        \indic{\widehat{\Delta}_M(r)>U}\bigr{\widehat{\Delta}_M(r)-U}}\Big|$}\\
        &\markref{(2)}{\le}
        \max\Big\{\Big|\indic{\Delta_m(r)<L}\bigr{L-\Delta_m(r)}-
        \indic{\widehat{\Delta}_m(r)<L}\bigr{L-\widehat{\Delta}_m(r)}\Big|,\\
        &\qquad
        \Big|\indic{\Delta_M(r)>U}\bigr{\Delta_M(r)-U}-
        \indic{\widehat{\Delta}_M(r)>U}\bigr{\widehat{\Delta}_M(r)-U}\Big|
        \Big\}\\
        &\markref{(3)}{\le}
        \max\Bigc{\Big|\bigr{L-\Delta_m(r)}-\bigr{L-\widehat{\Delta}_m}\Big|,
        \Big|\bigr{\Delta_M(r)-U}-\bigr{\widehat{\Delta}_M-U}\Big|}\\
        &=\max\Bigc{\Big|\Delta_m(r)-\widehat{\Delta}_m\Big|,
        \Big|\Delta_M(r)-\widehat{\Delta}_M\Big|},
    \end{align*}
    where at (1) we use Proposition \ref{prop: comp off subopt}, and the
    estimate used in Line \ref{line: comp off suboptimal} of \catyoff, at (2) we
    use the Lipschitzianity of the maximum operator $|\max\{x,y\}-
    \max\{x',y'\}| \le \max\{|x-x'|,|y-y'|\}$ for all $x,y,x',y'\in\RR$, at (3)
    we recognize that, in all four cases obtained by comparing $L$ with
    $\Delta_m(r)$, and $L$ with $\widehat{\Delta}_m$, the upper bound holds;
    similarly also for $U$ and $\Delta_M(r)$ and $\widehat{\Delta}_M$.
\end{proof}
Next, we have to show that our estimates are close to the true quantities with
high probability. Depending on the cardinality of $\cR$ (the set of rewards to
classify), we have two different lemmas. Before presenting the results, we need
some additional notation. Specifically, given any set $\cX\subseteq \SAH$, and
any pair of ``partial'' transition models
$p'\in\Delta^\cS_\cX,p''\Delta^\cS_{\cX^\complement}$, i.e., one defined only on
triples in $\cX$, and the other on the remaining triples
$\cX^\complement=\SAH\setminus\cX$, we define the $V$- and $Q$-functions and the
expected utility under the transition model obtained by combining $p'$ and $p''$
as $V^\pi_h(s;p',p'',r),Q^\pi_h(s,a;p',p'',r),J^\pi(r;p',p'')$ for any
$(s,a,h)\in\SAH,r\in\fR,\pi\in\Pi$. We extend the notation analogously to the
optimal $V$- and $Q$-functions and expected utility. In addition, we denote the
visit distribution of $\pi$ in this context as $d^{p',p'',\pi}$.
\begin{lemma}[Concentration for $|\cR|=1$]\label{lemma: conc single reward} Let
    $|\cR|=1$, and let $p'\in\Delta^\cS_{\cZ^\complement}$, $\pi\in\Pi$, and $r\in\fR$ be
    arbitrary. Let $\delta\in(0,1)$. Then, there exists a constant
    $c> 0$ for which the event $\cE\coloneqq\cE_1\cap\cE_2\cap\cE_3$ defined
    as the intersection of the events:
    \begin{align*}
        &\cE_1\coloneqq\bigg\{\Big|J^{\pi^E}(r;p)-\widehat{J}^E(r)\Big|
        \le c\sqrt{\frac{H^2\log\frac{6}{\delta}}{\tau^E}}\bigg\},\\
        &\cE_2\coloneqq\bigg\{
        \Big|\E_{s'\sim p_h(\cdot|s,a)}[V^{\pi}_{h+1}(s';\widehat{p},p',r)]
        -\E_{s'\sim \widehat{p}_h(\cdot|s,a)}[V^{\pi}_{h+1}(s';\widehat{p},p',r)]\Big|
        \le c\sqrt{\frac{H^2\ln\frac{12|\widehat{\cZ}|}{\delta}}{N_h^b(s,a)}},\\
        &\qquad\qquad\wedge\; \frac{1}{N_h^b(s,a)}\le c\cdot \frac{H\ln\frac{6|\widehat{\cZ}|}
        {\delta}}{\tau^b d^{p,\pi^b}_h(s,a)}
        \;\forall (s,a,h)\in\widehat{\cZ}
        \bigg\},\\
        &\mathcal{E}_3\coloneqq\bigg\{N_h^b(s,a)\ge 1, \;\forall (s,a,h)\in \cZ
        \quad\text{when }\tau^b\ge H\frac{\ln\frac{3|\cZ|}{\delta}}
        {\ln\frac{1}{1-d_{\min}^{p,\pi^b}}}\bigg\},
    \end{align*}
    where $N_h^b(s,a)$ is the random variable that counts the number of samples
    inside dataset $\cD_h^b$ for triple $(s,a,h)$, and
    $d_{\min}^{p,\pi^b}\coloneqq \min_{(s,a,h)\in\cZ}d^{p,\pi^b}_h(s,a)$, holds
    with probability at least $1-\delta$.
\end{lemma}
\begin{proof}
    The result follows by an application of the union bound after having
    observed that each event $\cE_1,\cE_2,\cE_3$ holds with probability
    $1-\delta/3$. Specifically, we have already shown in the proof of Theorem
    \ref{thr: upper bound caty online} that $\cE_1$ holds with probability
    $1-\delta/3$. Concerning event $\cE_2$, the result follows from Lemma B.1 of
    \citet{xie2021bridging}. Finally, w.r.t. event $\cE_3$, the result follows
    as shown in Lemma F.1 of \citet{lazzati2024offline}, with an additional $H$
    term due to the splitting of the datasets carried out by \catyoff.
\end{proof}
\begin{lemma}[Concentration for $|\cR|\ge1$]\label{lemma: conc all rewards}
    Let $\delta\in(0,1)$. Then, there exists a constant
    $c> 0$ for which the event $\cE'\coloneqq\cE_1'\cap\cE_2'\cap\cE_3'$ defined
    as the intersection of the events:
    \begin{align*}
        &\cE_1'\coloneqq\bigg\{\sup\limits_{r\in\fR}\Big|J^{\pi^E}(r;p)-\widehat{J}^E(r)\Big|
        \le c\sqrt{\frac{SAH^3\log\frac{6}{\delta}}{2\tau^E}}\bigg\},\\
        &\cE_2'\coloneqq\bigg\{
        N_h^b(s,a) KL(\widehat{p}_h(\cdot|s,a)\|p_h(\cdot|s,a))
        \le \beta(N_h^b(s,a),\delta/6),\\
        &\qquad\qquad\wedge\; \frac{1}{N_h^b(s,a)}\le c\cdot \frac{\ln\frac{6|\widehat{\cZ}|}
        {\delta}}{\tau^b d^{p,\pi^b}_h(s,a)}
        \;\forall (s,a,h)\in\widehat{\cZ}
        \bigg\},\\
        &\mathcal{E}_3'\coloneqq\bigg\{N_h^b(s,a)\ge 1, \;\forall (s,a,h)\in \cZ
        \quad\text{when }\tau^b\ge \frac{\ln\frac{3|\cZ|}{\delta}}
        {\ln\frac{1}{1-d_{\min}^{p,\pi^b}}}\bigg\},
    \end{align*}
    where $N_h^b(s,a)$ is the random variable that counts the number of samples
    inside dataset $\cD^b$ for triple $(s,a,h)$,
    $\beta(n,\delta)\coloneqq\ln(|\widehat{\cZ}|/{\delta})+ (S-1)\ln(
    e(1+n/{(S-1)}))$ and $d_{\min}^{p,\pi^b}\coloneqq
    \min_{(s,a,h)\in\cZ}d^{p,\pi^b}_h(s,a)$, holds with probability at least
    $1-\delta$.
\end{lemma}
\begin{proof}
    The result follows by an application of the union bound after having
    observed that each event $\cE_1',\cE_2',\cE_3'$ holds with probability
    $1-\delta/3$. Specifically, we have already shown in the proof of Theorem
    \ref{thr: upper bound caty online} that $\cE_1'$ holds with probability
    $1-\delta/3$. Concerning event $\cE_2'$, the result is contained in Lemma
    F.2 of \citet{lazzati2024offline}. In particular, the first part is proved
    through Lemma 10 in \citet{kaufmann2021adaptiveRFE}, while for the second
    part we simply apply Lemma A.1 of \citet{xie2021bridging}. Finally, w.r.t.
    event $\cE_3'$, the result follows as shown in Lemma F.1 of
    \citet{lazzati2024offline}.
\end{proof}
Lemma \ref{lemma: conc single reward} and Lemma \ref{lemma: conc all rewards}
are needed for the following tasks. Event $\cE_1$ guarantees that \catyoff
provides an accurate estimate of the expert's expected utility for a single
reward, while $\cE_1'$ provides the guarantee for all the bounded rewards. Event
$\cE_2$, intuitively, guarantees that the estimate of the transition model
$\widehat{p}$ is close to the true one $p$ at all the triples in $\widehat{\cZ}$
for a single reward, while $\cE_2'$ provides the guarantee in one-norm. Finally,
events $\cE_3$ and $\cE_3'$ guarantee that $\widehat{\cZ}=\cZ$. Note that, due
to the splitting of the datasets when $|\cR|=1$, the bound on $\tau^b$ for
$\cE_3$ is larger by a $H$ term than the bound for $\cE_3'$.

To prove Theorem \ref{thr: bounds tabular off}, we need to prove one last lemma:
\begin{lemma}\label{lemma: difference of J} 
    Let $s_0$ be the initial state and let $p,\widehat{p}\in\Delta^\cS_\cZ$
   and $p'\in\Delta^\cS_{\cZ^\complement}$ be arbitrary. Then, for any policy
   $\pi$ and reward function $r\in\fR$, it holds that:
    \begin{align*}
        \scalebox{0.98}{$  \displaystyle\Big| J^\pi(r;p,p')-J^\pi(r;\widehat{p},p') \Big|\le
        \sum\limits_{(s,a,h)\in\cZ}d_h^{p,p',\pi}(s,a)\Big|
        \sum\limits_{s'\in\cS}(p_h(s'|s,a)-\widehat{p}_h(s'|s,a))V^\pi_{h+1}(s';\widehat{p},p',r)
        \Big|.$}
    \end{align*}
\end{lemma}
\begin{proof}
    Let us denote by $\overline{p},\widecheck{p}\in\Delta_{\SAH}^\cS$ the
    transition models obtained by combining, respectively, $p$ with $p'$, and
    $\widehat{p}$ with $p'$. Then, we can write:
    \begin{align*}
        \Big|J^\pi(r;\overline{p})-J^\pi(r;\widecheck{p})\Big|
        &=\Big|V^\pi_1(s_0;\overline{p},r)-V^\pi_1(s_0;\widecheck{p},r)\Big|\\
        &=
        \Big|\sum\limits_{a\in\cA}\pi_1(a|s_0)\sum\limits_{s'\in\cS}\overline{p}_1(s'|s_0,a)V_2^\pi(s';\overline{p},r)\\
        &\qquad-\sum\limits_{a\in\cA}\pi_1(a|s_0)
        \sum\limits_{s'\in\cS}\widecheck{p}_1(s'|s_0,\pi_1(s_0))V_2^\pi(s';\widecheck{p},r)\Big|\\
        &\le\sum\limits_{a\in\cA}\pi_1(a|s_0)
        \Big|\sum\limits_{s'\in\cS}(\overline{p}_1(s'|s_0,a)-
        \widecheck{p}_1(s'|s_0,a))V_2^\pi(s';\widecheck{p},r)\Big|\\
        &\qquad+\sum\limits_{a\in\cA}\pi_1(a|s_0)\sum\limits_{s'\in\cS}\overline{p}_1(s'|s_0,a)
        \Big|V_2^\pi(s';\overline{p},r)-
        V_2^\pi(s';\widecheck{p},r)\Big|\\
        &\markref{(1)}{\le}\dotsc\\
        &\le \E_{\overline{p},\pi}\bigg[\sum\limits_{h\in\dsb{H}}
        \Big|
        \sum\limits_{s'\in\cS}(\overline{p}_h(s'|s,a)-
        \widecheck{p}_h(s'|s,a))V^\pi_{h+1}(s';\widecheck{p},r)
        \Big|\bigg]\\
        &=\sum\limits_{(s,a,h)\in\SAH}d_h^{\overline{p},\pi}(s,a)\Big|
        \sum\limits_{s'\in\cS}(\overline{p}_h(s'|s,a)-
        \widecheck{p}_h(s'|s,a))V^\pi_{h+1}(s';\widecheck{p},r)
        \Big|,
    \end{align*}
    where at (1) we have unfolded the recursion.
    
    The result follows by changing notation and by noticing that, by definition,
    $\overline{p}_h(s'|s,a)=\widecheck{p}_h(s'|s,a)=p'_h(s'|s,a)$ $\forall
    s'\in\cS$ in all $(s,a,h)\notin \cZ$.
\end{proof}
We are now ready to prove Theorem \ref{thr: bounds tabular off}.
\begin{proofthr}{thr: bounds tabular off}
    Thanks to the definition of PAC algorithm in Definition \ref{def: pac
    algorithm offline}, and thanks to Proposition \ref{prop: comp off subopt},
    then it is clear that, if we show that \catyoff satisfies:
    \begin{align*} 
        \mathop{\mathbb{P}}\limits_{\cM,\pie,\pi^b}\Big(
            \sup\limits_{r\in\cR} \Big|\Delta_m(r)-
            \widehat{\Delta}_m(r)\Big|\le\epsilon\;
            \wedge\;\sup\limits_{r\in\cR} \Big|\Delta_M(r)-
            \widehat{\Delta}_M(r)\Big|\le\epsilon\Big)\ge 1-\delta,
    \end{align*}
    then we have successfully proved that \catyoff is PAC.

    Let $r\in\fR$ be arbitrary, and assume that event $\cE$ holds if $|\cR|=1$,
    otherwise assume that event $\cE'$ holds. Then, we can write:
    \begin{align*}
         \Big|\Delta_M(r)-\widehat{\Delta}_M(r)\Big|
          &\le  \underbrace{\Big|\max\limits_{p'\in[p]_{\equiv_{\cZ}}}
          J^*(r;p')-\max\limits_{\widehat{p}'\in[\widehat{p}]_{\equiv_{\widehat{\cZ}}}}J^*(r;\widehat{p}')\Big|}_{\eqqcolon \cI}
          + \Big|\widehat{J}^E(r)-J^{\pie}(r;p)\Big|,\\
          \Big|\Delta_m(r)-\widehat{\Delta}_m(r)\Big|
          &\le  \underbrace{\Big|\min\limits_{p'\in[p]_{\equiv_{\cZ}}}
          J^*(r;p')-\min\limits_{\widehat{p}'\in[\widehat{p}]_{\equiv_{\widehat{\cZ}}}}J^*(r;\widehat{p}')\Big|}_{\eqqcolon \cJ}
          + \Big|\widehat{J}^E(r)-J^{\pie}(r;p)\Big|,
    \end{align*}
    where we have defined symbols $\cI,\cJ$. Next:
    \begin{align*}
        \cI&\markref{(1)}{\le}  \Big|\max\limits_{p'\in[p]_{\equiv_{\cZ}}}
        J^*(r;p')-\max\limits_{\widehat{p}'\in[\widehat{p}]_{\equiv_{\cZ}}}J^*(r;\widehat{p}')\Big|
      \\
      &\markref{(2)}{=}  \Big|\max\limits_{p'\in\Delta^\cS_{\cZ^\complement}}
      J^*(r;p,p')-\max\limits_{p'\in\Delta^\cS_{\cZ^\complement}}
      J^*(r;\widehat{p},p')\Big|\\
      &\le
      \max\limits_{p'\in\Delta^\cS_{\cZ^\complement}}\Big|
      J^*(r;p,p')-J^*(r;\widehat{p},p')\Big|\\
      &\le \max\limits_{p'\in\Delta^\cS_{\cZ^\complement}}
      \max\limits_\pi\Big|J^\pi(r;p,p')-J^\pi(r;\widehat{p},p')\Big|\\
      &\markref{(3)}{\le}\max\limits_{p'\in\Delta^\cS_{\cZ^\complement}}
      \max\limits_\pi
      \sum\limits_{(s,a,h)\in\cZ}d_h^{p,p',\pi}(s,a)\Big|
      \sum\limits_{s'\in\cS}(p_h(s'|s,a)-\widehat{p}_h(s'|s,a))V^\pi_{h+1}(s';\widehat{p},p',r)
      \Big|,
    \end{align*}
    where at (1) we use that, under $\cE_3$ or $\cE_3'$, we have
    $\widehat{\cZ}=\cZ$, at (2) we use the definitions of $[p]_{\equiv_{\cZ}}$
    and $[\widehat{p}]_{\equiv_{\cZ}}$, and at (3) we apply Lemma \ref{lemma:
    difference of J}. Note that we can show that also quantity $\cJ$ enjoys the
    same upper bound by using the property that $|\min_x f(x)-\min_x g(x)|\le
    \max_x|f(x)-g(x)|$.

    Now, when $|\cR|=1$, under the good event $\cE$, we can upper bound the
    last term by:
    \begin{align*}
        \max\limits_{p'\in\Delta^\cS_{\cZ^\complement}}
          \max\limits_\pi&\sum\limits_{(s,a,h)\in\cZ}d_h^{p,p',\pi}(s,a)\cdot
        c\sqrt{\frac{H^3 \ln^2\frac{12|\widehat{\cZ}|}{\delta}}{\tau^b d^{p,\pi^b}_h(s,a)}}
        \le cH\sqrt{\frac{H^3 \ln^2\frac{12|\cZ|}
        {\delta}}{\tau^b d_{\min}^{p,\pi^b}}}\le\frac{\epsilon}{2},
    \end{align*}
    from which we obtain the bound in the theorem for $\tau^b$ when $|\cR|=1$.
    For the bound for $\tau^E$, we simply set the confidence bound of event
    $\cE_1$ to be $\le\epsilon/2$ and solve w.r.t. $\tau^E$. The result follows
    through an application of Lemma \ref{lemma: conc single reward}.

    For $|\cR|\ge1$, under good event $\cE'$, by applying also the Pinsker's
    inequality, we can upper bound:
    \begin{align*}
          \max\limits_{p'\in\Delta^\cS_{\cZ^\complement}}
          \max\limits_\pi\sum\limits_{(s,a,h)\in\cZ}d_h^{p,p',\pi}(s,a)\cdot
        cH\sqrt{\frac{ \ln\frac{6|\widehat{\cZ}|}{\delta}\beta(\tau^b,\delta/6)}{\tau^b d^{p,\pi^b}_h(s,a)}}
        \le cH^2\sqrt{\frac{ \ln\frac{6|\cZ|}{\delta}\beta(\tau^b,\delta/6)}{\tau^b d_{\min}^{p,\pi^b}}}
        \le\frac{\epsilon}{2},
    \end{align*}
    and solving w.r.t. $\tau^b$ using Lemma J.3 in \citet{lazzati2024offline} we
    obtain the bound in the theorem for $\tau^b$ when $|\cR|\ge1$.
    For the bound on $\tau^E$, we simply impose that the confidence bound in
    event $\cE_1'$ is smaller than $\epsilon/2$. Finally, we apply Lemma
    \ref{lemma: conc all rewards} to show that the guarantee holds with
    probability at least $1-\delta$. This concludes the proof.    
\end{proofthr}

\section{Discussion}\label{sec: discussion}

In this section, we provide a discussion on the flexibility of the reward
compatibility framework by presenting additional problem settings and extensions
in which the adoption of reward compatibility for efficient learning is
straightforward. Moreover, we discuss about some design choices made in this
paper, and we collect additional insights about the practical ``usage'' of the
rewards learned through IRL.

\paragraph{Reward Learning.}

In the context of Reward Learning (ReL), the learner receives a variety of
expert feedbacks to learn the true reward function $r^E$
\citep{jeon2020rewardrational}. From the ``constraint'' column of Table 2 in
\citet{jeon2020rewardrational}, we recognize that each feedback, similarly to
the IRL feedback, provides a different kind of (inequality) constraint with
which reducing the amount of rewards in $\mathfrak{R}$ that represent feasible
candidates for $r^E$. It should be remarked that these constraints are ``hard'',
in that a reward either satisfies the constraint or not (we might define a
notion of feasible set as the set of rewards satisfying such constraints). To
permit efficient learning when the transition model is unknown, our reward
compatibility framework proposes to transform such ``hard'' constraints into
``soft'' constraints, by measuring the compatibility of the rewards with the
constraints.

\paragraph{Other IRL settings.}

A popular alternative to the IRL setting with optimal expert
\citep{ng2000algorithms}, is that in which the expert is optimal in a certain
\emph{entropy-regularized MDP} \citep{ziebart2010MCE,Fu2017LearningRR}, i.e.:
\begin{align*}
    \pi^E=\argmax\limits_{\pi\in\Pi} \mathop{\E}\limits_{p,\pi} \bigg[
        \sum\limits_{h\in\dsb{H}}
        r_h(s_h,a_h) + \mathcal{H}(\pi_h(\cdot|s_h))
        \bigg]\eqqcolon \argmax\limits_{\pi\in\Pi}\overline{J}^\pi(r;p),
\end{align*}
where $\mathcal{H}$ denotes the entropy. In words, the objective of the expert
consists in the maximization of the \emph{entropy-regularised} expected utility
$\overline{J}$. The advantage of this formulation is the existence of a unique
optimal policy. To permit efficient learning when the dynamics is unknown, it is
possible to extend our notion of reward compatibility to the maximum-entropy IRL
framework as:
\begin{align*}
    \overline{\cC}^{\text{ENT}}_{p,\pi^E}(r)\coloneqq
    \max\limits_\pi \overline{J}^\pi(r;p)-\overline{J}^{\pi^E}(r;p).
\end{align*}
In this manner, we quantify the degree of compatibility of reward $r$ with the
given constraint. We observe that the same approach can be adopted also for
other settings like, for instance, IRL with risk-sensitive agents
\citep{lazzati2024learningutilitiesdemonstrationsmarkov}, once that the form of
the objective of the expert has been defined.

\paragraph{Robustness to misspecification.}

When the given constraint is not known exactly, our framework permits to be
robust by accurately choosing the classification threshold $\Delta$. For
instance, when the expert's suboptimality is contained into $[L,U]$ for some
$U\ge L\ge 0$, but we are uncertain about the specific values $L,U$, then we can
use a larger classification threshold. Intuitively, the larger the
classification threshold, the more uncertain we are on the feedback received.

\paragraph{Demonstrations in multiple environments.}

Consider the learning setting in which we are given many expert's demonstrations
about the same reward $r^E$ in a variety of environments
\citep{amin2016resolving,cao2021identifiability}. Clearly, multiple constraints
reduce the partial identifiability of the problem and permit to retrieve a
smaller feasible set. When the transition model is unknown, reward compatibility
permits to cope with uncertainty in a straightforward way. For instance, let
$\{\widehat{\cC}_i(\cdot)\}_{i\in\dsb{N}}$ be the estimated (non)compatibilities
associated to demonstrations $\{\cD^E_i\}_{i\in\dsb{N}}$ in $N$ environments. A
meaningful objective consists in finding a reward $r$ such that
$\max_{i\in\dsb{N}}\widehat{\cC}_i(r)\le\epsilon$, i.e., which is at most
$\epsilon$-(non)compatible (for some $\epsilon>0$) with all the input
demonstrations.

\paragraph{Multiplicative compatibility.}

Any reward $r\in\mathfrak{R}$ induces, in the considered environment $\cM$ with
dynamics $p$, an ordering in the space of policies $\Pi$, based on the
performance $J^{\pi}(r;p)$ of each policy $\pi\in\Pi$. It is easy to notice that
for any scaling and translation parameters $\alpha \in \RR_{> 0},\beta\in \RR$,
the reward constructed as $r'(\cdot,\cdot)=\alpha r(\cdot,\cdot)+\beta$ induces
the same ordering as $r$ in the space of policies.\footnote{Indeed, simply
observe that, for any $\pi\in\Pi$: $J^{\pi}(r';p)=J^{\pi}(\alpha
r+\beta;p)=\alpha J^{\pi}(r;p)+\beta$.} Nevertheless, the \emph{additive} notion
of (non)compatibility $\overline{\cC}_{\cM,\pie}$ in Definition \ref{def: reward
compatibility optimal} (setting with optimal expert), is such that, for any
$r\in\mathfrak{R}$:
\begin{align*}
    &\overline{\cC}_{\cM,\pie}(r+\beta)=\overline{\cC}_{\cM,\pie}(r)\qquad\forall \beta\in \RR,\\
    &\overline{\cC}_{\cM,\pie}(\alpha r)=\alpha\overline{\cC}_{\cM,\pie}(r)\neq\overline{\cC}_{\cM,\pie}(r)
    \qquad\forall\alpha\in \RR_{>0}.
\end{align*}
Simply put, the scale $\alpha$ of the reward matters, and rescaling the reward
modifies the (non)compatibility. If we are interested in invariance to scale, we
can define a \emph{multiplicative} notion of
compatibility\footnote{Multiplicative suboptimality has already been analysed in
literature. E.g., see Theorem 7.2.7 in \cite{puterman1994markov}, which is
inspired by \cite{ornstein1969existence}.} $\cF$ (defined only for non-negative
rewards $r$ and non-zero optimal performance $J^{*}(r;p)$) as:
\begin{align*}
    \cF_{\cM,\pie}(r)\coloneqq\frac{J^{\pie}(r;p)}{J^{*}(r;p)}.
\end{align*}
Clearly, the larger $\cF_{\cM,\pie}(r)$, the closer is the performance of $\pie$
to the optimal performance. By definition, we have:
\begin{align*}
    & \cF_{\cM,\pie}(\alpha r)=\cF_{\cM,\pie}(r)\qquad\forall\alpha\in \RR_{>0}\\
    & \cF_{\cM,\pie}(r+\beta)\neq \cF_{\cM,\pie}(r)\qquad\forall \beta\in \RR,
\end{align*}
i.e., this definition does not care about the scaling $\alpha$ of the reward,
but it is sensitive to the actual ``location'' $\beta$ of the reward. Thus,
intuitively, none of $\overline{\cC}_{\cM,\pie},\overline{\cF}_{\cM,\pie}$ can
be seen as perfect. We prefer to use $\overline{\cC}_{p,\pie}$ instead of
$\cF_{\cM,\pie}$ because $(i)$ most of the RL literature prefers the additive
notion of suboptimality instead of the multiplicative one, giving importance to
the scale of the reward, and $(ii)$ the additive notion of suboptimality is
simpler to analyze from a theoretical viewpoint w.r.t. the multiplicative one.

\paragraph{When can a learned reward be ``used'' for (forward) RL?}

We are interested in applications of IRL like Apprenticeship Learning (AL). We
say that a reward function $r$ \emph{can be ``used'' for (forward) RL} if the
policy $\pi$ obtained through the optimization of $r$ performs acceptably under
the true expert's reward $r^E$. What properties should the reward $r$, learned
through IRL, satisfy in order to be ``usable''? We now list and analyze some
plausible requirements which are common in literature.

First, $(i)$ we might ask that, being $\pie$ optimal w.r.t. $r^E$, then any
reward $r$ such that $\pie\in\argmax_\pi J^\pi(r)$ can be used, i.e., any reward
in the feasible set $\cR_{\cM,\pie}$ \citep{metelli2023towards}. However, there
are rewards $r\in\cR_{\cM,\pie}$ that induce more than one optimal policy (e.g.,
both $\overline{\pi},\pie$ as optimal), and optimal policies other than $\pie$
(e.g., $\overline{\pi}$) are not guaranteed to perform well under $r^E$
(actually, $\overline{\pi}$ can be an arbitrary policy in $\Pi$). Clearly, this
is not satisfactory.
Another possibility $(ii)$ consists in rewards $r\in\cR_{\cM,\pie}$ such that
$\pie$ is the \emph{unique} optimal policy, similarly to what happens in
entropy-regularized MDPs \citep{ziebart2010MCE,Fu2017LearningRR}. However,
in practice, due to computational limitations or uncertainty (e.g., estimated
dynamics $\widehat{p}$), we can just afford to compute an $\epsilon$-optimal
policy for $r$ in $p$. Since any policy can be $\epsilon$-optimal under reward
$r$, then we have no guarantee on its performance w.r.t. $r^E$.
A third requirement $(iii)$ asks for rewards $r$ that make $\pie$ at least
$\epsilon$-optimal, i.e., $\epsilon$-compatible rewards based on Definition
\ref{def: reward compatibility optimal}. However, since these rewards represent
a superset of the feasible set $\cR_{\cM,\pie}$, then even this requirement is
not satisfactory.

The three requirements described above do not provide guarantees that
optimizing the considered reward $r$ provides a policy with satisfactory
performance w.r.t. the true $r^E$.
\begin{remark}[Sufficient condition for reward usability]
    If we want to be sure that an $\epsilon$-optimal policy for the learned
    reward $r$ is at least $f(\epsilon)$-optimal for $r^E$ (for some function
    $f$), then \emph{it suffices that all the (at least) $\epsilon$-optimal policies
    $\pi$ under the learned $r$ have visitation distribution close to that
    of $\pie$ in 1-norm}:
    \begin{align*}
        |J^{\pie}(r^E;p)-J^{\pi}(r^E;p)|=
        \Big|\sum\limits_{h\in\dsb{H}}\dotp{d^{p,\pie}_h-d^{p,\pi}_h,r^E_h}\Big|
        \le
        \sum\limits_{h\in\dsb{H}}\|d^{p,\pi}_h-d^{p,\pie}_h\|_1.
    \end{align*}
\end{remark}
If we define distance $d^{\text{all}}$ between rewards $r,r'$ (see Section 3.1
of \citet{zhao2023inverse}) as:
\begin{align*}
    d^{\text{all}}(r,r')\coloneqq \sup\limits_{\pi\in\Pi}|J^\pi(r;p)-J^\pi(r';p)|,
\end{align*}
then we see that $d^{\text{all}}(r,r^E)\le\epsilon$ for some small
$\epsilon\ge0$ represents a stronger condition for reward usability. Indeed, if
$d^{\text{all}}(r,r^E)$ is small, then the performance of \emph{any} policy as
measured by $r$, not just optimal policy or $\epsilon$-optimal policy, is
similar to its performance as measured by $r^E$.
Therefore, clearly, \emph{rewards $r$ with small distance to $r^E$ w.r.t.
$d^{\text{all}}$ \emph{can} be ``used'' for forward RL}. However, since expert
demonstrations do not provide any information about the performance of policies
other than $\pie$ under $r^E$, we have the following result:
\begin{restatable}{prop}{propnoguaranteesdall}\label{prop: no guarantees on
    dall} Let $\cM=\tuple{\cS,\cA,H, d_0,p}$ be a \emph{known} MDP without
    reward, and let $\pie$ be a \emph{known} expert's policy. Let $r^E$ the true
    \emph{unknown} reward optimized by the expert to construct $\pie$. Then,
    there does not exist a learning algorithm that, for any
    $\epsilon,\delta\in(0,1)$, receives in input an arbitrary pair
    $\tuple{\cM,\pie}$ and outputs a single reward $r$ such that
    $d^{\text{all}}(r,r^E)\le\epsilon$ w.p. $1-\delta$.
\end{restatable}
\begin{proofsketch}
    Simply, we can construct $\cM$ to be an MDP without reward in which there is
    at least a policy $\pi\neq\pi^E$ such that
    $\cS^{p,\pi}\cap\cS^{p,\pi^E}=\emptyset$, i.e., the states visited by the
    two policies are different. Then, in the feasible set $\fs$, there are both
    the rewards $r',r''$ that give $J^{\pi}(r';p)=0$ and $J^{\pi}(r'';p)=H$.
    Since we do not know the performance of $\pi$ under $r^E$, then both rewards
    are plausible. For this reason, whatever the output $r$ of the algorithm,
    the distance $d^{\text{all}}(r,r^E)$ (note that $d^{\text{all}}$ is a
    metric) cannot be smaller than $H/2$ in the worst case. By taking $\delta >
    1/2$ and using triangle inequality, we can prove the result.
\end{proofsketch}
Nevertheless, \citet{metelli2023towards,lazzati2024offline,zhao2023inverse}
\emph{seem} to provide sample efficient algorithms w.r.t.
$d^{\text{all}}$.\footnote{Actually,
\cite{metelli2023towards,lazzati2024offline} use different notions of distance,
like $d_\infty(r,r')\coloneqq\|r-r'\|_\infty$. However, we can write
$\|r-r'\|_\infty\ge\|r-r'\|_1/(SAH)$, and by dual norms we have that
$d^{\text{all}}(r,r')=\sup_{\pi\in\Pi}|\dotp{d^{p,\pi},r-r'}|\le
\sup_{\overline{d}:\|\overline{d}\|_\infty\le1}|\dotp{\overline{d},r-r'}|=\|r-r'\|_1$.
Therefore, the guarantees of \cite{metelli2023towards,lazzati2024offline} can be
converted to $d^{\text{all}}$ guarantees too.} By looking at Proposition
\ref{prop: no guarantees on dall}, we realize that this is clearly a
\emph{contradiction}. What is the right interpretation?
The trick is that the algorithms proposed in
\citet{metelli2023towards,lazzati2024offline,zhao2023inverse} are \emph{not}
able to output a single reward $r$ which is close to $r^E$ w.r.t.
$d^{\text{all}}$, but, \emph{for any possible reward $r^E=r^E(V,A)$
parametrized\footnote{While \cite{zhao2023inverse} makes this parametrization
explicit, \cite{metelli2023towards,lazzati2024offline} keep the parametrization
implicit, but everything is analogous.} by some value and advantage functions
$V,A$, they are able to output a reward $r$ such that
$d^{\text{all}}(r,r^E(V,A))$ is small.} In other words, it is like if these
works \emph{assume to know} the $V,A$ parametrization of the true reward $r^E$.
In this way, these works are able to output a reward $r$ that can be used for
``forward'' RL, otherwise their algorithms cannot provide such guarantee.

\section{Related Work}\label{sec: related work}

In this section, we report and describe the literature that relates the most to
this paper. Theoretical works concerning sample efficient IRL can be grouped
in works that concern the feasible set, and works that do not.

\paragraph{Feasible Set.}
Let us begin with works related to the feasible set. While the notion of
feasible set has been introduced implicitly in \cite{ng2000algorithms}, the
first paper that analyses the sample complexity of estimating the feasible set
in online IRL is \cite{metelli2021provably}. Authors in
\cite{metelli2021provably} adopt the simple generative model in tabular MDPs,
and devise two sample efficient algorithms. \cite{lindner2022active} focuses on
the same problem as \cite{metelli2021provably}, but adopts a forward model in
tabular MDPs. By adopting RFE exploration algorithms, they devise sample
efficient algorithms. However, as remarked in \cite{zhao2023inverse}, the
learning framework considered in \cite{lindner2022active} suffers from a major
issue. \cite{metelli2023towards} builds upon \cite{metelli2021provably} to
construct the first minimax lower bound for the problem of estimating the
feasible set using a generative model. The lower bound is in the order of
$\Omega\big(\frac{H^3SA}{\epsilon^2}(S+\log\frac{1}{\delta})\big)$, where $S$
and $A$ are the cardinality of the state and action spaces, $H$ is the horizon,
$\epsilon$ is the accuracy and $\delta$ the failure probability. In addition,
\cite{metelli2023towards} develops US-IRL, an efficient algorithm whose sample
complexity matches the lower bound. \cite{poiani2024inverse} analyze a setting
analogous to that of \cite{metelli2023towards}, in which there is availability
of a single optimal expert and multiple suboptimal experts with known
suboptimality. \cite{lazzati2024offline} analyse the problem of estimating the
feasible set when no active exploration of the environment is allowed, but the
learner is given a batch dataset collected by some behavior policy $\pi^b$. In
particular, \cite{lazzati2024offline} focus on two novel learning targets that
are suited for the offline setting, i.e., a subset and a superset of the
feasible set, and demonstrate that such sets are the tightest learnable subset
and superset of the feasible set. They conclude by proposing a pessimistic
algoroithm, PIRLO, to estimate them. \cite{zhao2023inverse} analyses the same
offline setting as \cite{lazzati2024offline}, but instead of focusing on the
notion of feasible set directly, it considers the notion of reward mapping,
which considers reward functions as parametrized by their value and advantage
functions, and whose image coincides with the feasible set.

\paragraph{Other sample-efficient IRL settings.}
With regards to IRL works that do not consider the feasible set, we
mention \cite{lopes2009active}, which analyses an active learning framework for
IRL. However, \cite{lopes2009active} assumes that the transition model is known,
and its goal is to estimate the expert policy only. Works
\cite{komanduru2019correctness} and \cite{komanduru2021lowerbound} provide,
respectively, an upper bound and a lower bound to the sample complexity of IRL
for $\beta$-strict separable problems in the tabular setting. However, both the
setting considered and the bound obtained are fairly different from ours. 
Analogously, \cite{dexter2021IRL} provides a sample efficient IRL algorithm for
$\beta$-strict separable problems with continuous state space. However, their
setting is different from ours since they assume that the system can be modelled
using a basis of orthonormal functions.

\paragraph{Identifiability and Reward Learning.}
As aforementioned, the IRL problem is ill-posed, thus, to retrieve a single
reward, additional constraints shall be imposed. \cite{amin2016resolving}
analyse the setting in which demonstrations of an optimal policy for the same
reward function are provided across environments with different transition
models. In this way, authors can reduce the experimental unidentifiability, and
recover the state-only reward function. \cite{cao2021identifiability} and
\cite{kim2020domain} concern reward identifiability but in entropy-regularized
MDPs \citep{ziebart2010MCE,Fu2017LearningRR}. Such setting is in some sense
easier than the common IRL setting, because entropy-regularization permits a
unique optimal policy for any reward function. \cite{cao2021identifiability} use
expert demonstrations from multiple transition models and multiple discount
factors to retrieve the reward function, while \cite{kim2020domain} analyse
properties of the dynamics of the MDP to increase the constraints. With regards
to the more general field of Reward Learning (ReL), we mention
\cite{jeon2020rewardrational}, which introduce a framework that formalizes the
constraints imposed by various kinds of human feedback (like demonstrations or
preferences \citep{wirth2017surveyPbRL}). Intuitively, multiple feedbacks about
the same reward represent additional constraints beyond mere demonstrations.
\cite{skalse2023invariance} characterize the partial identifiability of the
reward function based on various reward learning data sources.

\paragraph{Online Apprenticeship Learning.}
The first works that provide a theoretical analysis of the AL setting when the
transition model is unknown are \cite{abbeel2005exploration} and
\citet{syed2007game}. Recently, \cite{Shani2021OnlineAL} formulate the online
AL problem, which resembles the online IRL problem. The main difference
is that in online AL the ultimate goal is to imitate the expert, while in IRL is
to recover a reward function. \cite{Xu2023ProvablyEA} improve the results in
\cite{Shani2021OnlineAL} by combining an RFE algorithm with an efficient
algorithm for the estimation of the visitation distribution of the deterministic
expert's policy in tabular MDPs, presented in \cite{Rajaraman2020TowardTF}. We
mention also \cite{rajaraman2021value,swamy2022minimax} for the sample
complexity of estimating the expert's policy in problems with linear function
approximation. In the context of Imitation Learning from Observation alone
(ILfO) \citep{liu2018imitation}, \cite{sun2019provably} propose a
probably efficient algorithm for large-scale MDPs with unknown transition model.
\cite{liu2022learning} provide an efficient AL algorithm based on GAIL
\citep{ho2016gail} in Linear Kernel Episodic MDPs \cite{zhou2021nearly} with
unknown transition model.

\paragraph{Others.} We mention \cite{klein2012IRLclassification}, which consider
a classification approach for IRL. However, this is fairly different from our
IRL problem formulation in Sections \ref{sec: the rewards compatibility
framework} and \ref{sec: offline}.

\section{Conclusion}\label{sec: conclusions}

Motivated by major limitations of the feasible reward set as a unifying
framework for sample-efficient IRL, in this paper, we have presented the
powerful framework of \emph{reward compatibility}, which permits efficient
learning in many IRL problems. The major advantage of reward compatibility is
its flexibility, since it can easily be adapted to a multitude of problem
settings. In this paper, we have considered both an optimal and a suboptimal
expert, and we have analysed online IRL in tabular and Linear MDPs, and offline
IRL in tabular MDPs only, and we have provided sample efficient algorithms,
\caty and \catyoff, for solving the newly-proposed IRL classification problem.
We have also discussed on the flexibility of the framework introduced in a
variety of \emph{complementary} settings, and provided some insights about the
usage of reward functions learned through IRL for (forward) RL.

\textbf{Limitations.}~~In this work, we have not analysed the offline IRL
problem in Linear MDPs. In particular, as already mentioned, we note that the
approach with the best and worst notions of (non)compatibility that we adopted
cannot be extended straightforwardly to linear MDPs, because it is based on a
notion of coverage best-suited for tabular MDPs. For this reason, an original
definition of reward compatibility that permits to use existing notions of
coverage of the state space for linear MDPs has to be devised, in order to
analyse the offline IRL problem in linear MDPs through our framework.
In addition, we acknowledge that the empirical validation of the proposed
algorithms is beyond the scope of this work. Our focus has been on developing
the theoretical framework and foundational aspects, leaving the empirical
evaluation for future research.

\textbf{Future Directions.}~~Promising directions for future works concern the
extension of the analysis of the \emph{reward compatibility} framework beyond
Linear MDPs to general function approximation and to the offline setting.
Moreover, it would be fascinating to extend the notion of reward compatibility
to other kinds of expert feedback (in the context of ReL), and to other IRL
settings (e.g., Boltzmann rational experts), as discussed in Section \ref{sec:
discussion}. In such way, we believe that it will be possible to bring IRL
closer to real-world applications.

\acks{AI4REALNET has received funding from European Union's Horizon Europe
Research and Innovation programme under the Grant Agreement No 101119527. Views
and opinions expressed are however those of the author(s) only and do not
necessarily reflect those of the European Union. Neither the European Union nor
the granting authority can be held responsible for them. Funded by the European
Union - Next Generation EU within the project NRPP M4C2, Investment 1.,3 DD. 341
- 15 march 2022 - FAIR - Future Artificial Intelligence Research - Spoke 4 -
PE00000013 - D53C22002380006.}

\newpage

\appendix

\vskip 0.2in
\bibliography{refs}

\begin{thebibliography}{72}
\providecommand{\natexlab}[1]{#1}
\providecommand{\url}[1]{\texttt{#1}}
\expandafter\ifx\csname urlstyle\endcsname\relax
  \providecommand{\doi}[1]{doi: #1}\else
  \providecommand{\doi}{doi: \begingroup \urlstyle{rm}\Url}\fi

\bibitem[Abbeel and Ng(2004)]{abbeel2004apprenticeship}
Pieter Abbeel and Andrew~Y. Ng.
\newblock Apprenticeship learning via inverse reinforcement learning.
\newblock In \emph{International Conference on Machine Learning 21 (ICML)}, 2004.

\bibitem[Abbeel and Ng(2005)]{abbeel2005exploration}
Pieter Abbeel and Andrew~Y. Ng.
\newblock Exploration and apprenticeship learning in reinforcement learning.
\newblock In \emph{International Conference on Machine Learning 22 (ICML)}, 2005.

\bibitem[Abbeel et~al.(2006)Abbeel, Coates, Quigley, and Ng]{abbeel2006helicopter}
Pieter Abbeel, Adam Coates, Morgan Quigley, and Andrew Ng.
\newblock An application of reinforcement learning to aerobatic helicopter flight.
\newblock In \emph{Advances in Neural Information Processing Systems 19 (NeurIPS)}, 2006.

\bibitem[Adams et~al.(2022)Adams, Cody, and Beling]{adams2022survey}
Stephen Adams, Tyler Cody, and Peter~A. Beling.
\newblock A survey of inverse reinforcement learning.
\newblock \emph{Artificial Intelligence Review}, 55:\penalty0 4307--4346, 2022.

\bibitem[Amin and Singh(2016)]{amin2016resolving}
Kareem Amin and Satinder Singh.
\newblock Towards resolving unidentifiability in inverse reinforcement learning, 2016.

\bibitem[Arora and Doshi(2018)]{arora2018survey}
Saurabh Arora and Prashant Doshi.
\newblock A survey of inverse reinforcement learning: Challenges, methods and progress.
\newblock \emph{Artificial Intelligence}, 297:\penalty0 103500, 2018.

\bibitem[Auer et~al.(2008)Auer, Jaksch, and Ortner]{auer2008nearoptimal}
Peter Auer, Thomas Jaksch, and Ronald Ortner.
\newblock Near-optimal regret bounds for reinforcement learning.
\newblock In \emph{Advances in Neural Information Processing Systems 21 (NeurIPS)}, pages 89--96, 2008.

\bibitem[Azar et~al.(2013)Azar, Munos, and Kappen]{azar2013minimax}
Mohammad~Gheshlaghi Azar, R{\'e}mi Munos, and Hilbert Kappen.
\newblock {Minimax PAC bounds on the sample complexity of reinforcement learning with a generative model}.
\newblock \emph{{Machine Learning}}, 91\penalty0 (3):\penalty0 325--349, 2013.

\bibitem[Barnes et~al.(2024)Barnes, Abueg, Lange, Deeds, Trader, Molitor, Wulfmeier, and O'Banion]{barnes2024massively}
Matt Barnes, Matthew Abueg, Oliver~F. Lange, Matt Deeds, Jason Trader, Denali Molitor, Markus Wulfmeier, and Shawn O'Banion.
\newblock Massively scalable inverse reinforcement learning in google maps, 2024.

\bibitem[Cao et~al.(2021)Cao, Cohen, and Szpruch]{cao2021identifiability}
Haoyang Cao, Samuel Cohen, and Lukasz Szpruch.
\newblock Identifiability in inverse reinforcement learning.
\newblock In \emph{Advances in Neural Information Processing Systems 34 (NeurIPS)}, pages 12362--12373, 2021.

\bibitem[Cheng et~al.(2020)Cheng, Kolobov, and Agarwal]{cheng2020policyimprovementimitationmultiple}
Ching-An Cheng, Andrey Kolobov, and Alekh Agarwal.
\newblock Policy improvement via imitation of multiple oracles, 2020.

\bibitem[Dexter et~al.(2021)Dexter, Bello, and Honorio]{dexter2021IRL}
Gregory Dexter, Kevin Bello, and Jean Honorio.
\newblock Inverse reinforcement learning in a continuous state space with formal guarantees.
\newblock In \emph{Advances in Neural Information Processing Systems 34 (NeurIPS)}, pages 6972--6982, 2021.

\bibitem[Finn et~al.(2016)Finn, Levine, and Abbeel]{finn2016guided}
Chelsea Finn, Sergey Levine, and Pieter Abbeel.
\newblock Guided cost learning: Deep inverse optimal control via policy optimization.
\newblock In \emph{International Conference on Machine Learning 33 (ICML)}, volume~48, pages 49--58, 2016.

\bibitem[Freihaut and Ramponi(2024)]{freihaut2024multiagentinversereinforcementlearning}
Till Freihaut and Giorgia Ramponi.
\newblock On multi-agent inverse reinforcement learning, 2024.

\bibitem[Fu et~al.(2017)Fu, Luo, and Levine]{Fu2017LearningRR}
Justin Fu, Katie Luo, and Sergey Levine.
\newblock Learning robust rewards with adversarial inverse reinforcement learning.
\newblock In \emph{International Conference on Learning Representations 5 (ICLR)}, 2017.

\bibitem[Hadfield-Menell et~al.(2016)Hadfield-Menell, Russell, Abbeel, and Dragan]{hadfieldmenell2016cooperativeIRL}
Dylan Hadfield-Menell, Stuart~J Russell, Pieter Abbeel, and Anca Dragan.
\newblock Cooperative inverse reinforcement learning.
\newblock In \emph{Advances in Neural Information Processing Systems 29 (NeurIPS)}, 2016.

\bibitem[Hadfield-Menell et~al.(2017)Hadfield-Menell, Milli, Abbeel, Russell, and Dragan]{hadfieldmenell2017inverserewarddesign}
Dylan Hadfield-Menell, Smitha Milli, Pieter Abbeel, Stuart~J Russell, and Anca Dragan.
\newblock Inverse reward design.
\newblock In \emph{Advances in Neural Information Processing Systems 30 (NeurIPS)}, 2017.

\bibitem[Ho and Ermon(2016)]{ho2016gail}
Jonathan Ho and Stefano Ermon.
\newblock Generative adversarial imitation learning.
\newblock In \emph{Advances in Neural Information Processing Systems 29 (NeurIPS)}, 2016.

\bibitem[Jarboui and Perchet(2021)]{jarboui2021offlineinversereinforcementlearning}
Firas Jarboui and Vianney Perchet.
\newblock Offline inverse reinforcement learning, 2021.

\bibitem[Jeon et~al.(2020)Jeon, Milli, and Dragan]{jeon2020rewardrational}
Hong~Jun Jeon, Smitha Milli, and Anca Dragan.
\newblock Reward-rational (implicit) choice: A unifying formalism for reward learning.
\newblock In \emph{Advances in Neural Information Processing Systems 33 (NeurIPS)}, pages 4415--4426, 2020.

\bibitem[Jiang et~al.(2017)Jiang, Krishnamurthy, Agarwal, Langford, and Schapire]{jiang2017contextual}
Nan Jiang, Akshay Krishnamurthy, Alekh Agarwal, John Langford, and Robert~E. Schapire.
\newblock Contextual decision processes with low {B}ellman rank are {PAC}-learnable.
\newblock In \emph{International Conference on Machine Learning 34 (ICML)}, volume~70, pages 1704--1713, 2017.

\bibitem[Jin et~al.(2020{\natexlab{a}})Jin, Krishnamurthy, Simchowitz, and Yu]{jin2020RFE}
Chi Jin, Akshay Krishnamurthy, Max Simchowitz, and Tiancheng Yu.
\newblock Reward-free exploration for reinforcement learning.
\newblock In \emph{International Conference on Machine Learning 37 (ICML)}, volume 119, pages 4870--4879, 2020{\natexlab{a}}.

\bibitem[Jin et~al.(2020{\natexlab{b}})Jin, Yang, Wang, and Jordan]{jin2020provablyefficient}
Chi Jin, Zhuoran Yang, Zhaoran Wang, and Michael~I Jordan.
\newblock Provably efficient reinforcement learning with linear function approximation.
\newblock In \emph{Conference on Learning Theory 33 (COLT)}, volume 125, pages 2137--2143, 2020{\natexlab{b}}.

\bibitem[Jin et~al.(2021{\natexlab{a}})Jin, Liu, and Miryoosefi]{jin2021eluder}
Chi Jin, Qinghua Liu, and Sobhan Miryoosefi.
\newblock Bellman eluder dimension: New rich classes of rl problems, and sample-efficient algorithms.
\newblock In \emph{Advances in Neural Information Processing Systems 34 (NeurIPS)}, pages 13406--13418, 2021{\natexlab{a}}.

\bibitem[Jin et~al.(2021{\natexlab{b}})Jin, Yang, and Wang]{jin2021ispessimism}
Ying Jin, Zhuoran Yang, and Zhaoran Wang.
\newblock Is pessimism provably efficient for offline rl?
\newblock In \emph{International Conference on Machine Learning 38 (ICML)}, volume 139, pages 5084--5096, 2021{\natexlab{b}}.

\bibitem[Jing et~al.(2019)Jing, Ma, Huang, Sun, Yang, Fang, and Liu]{jing2019reinforcementlearningimperfectdemonstrations}
Mingxuan Jing, Xiaojian Ma, Wenbing Huang, Fuchun Sun, Chao Yang, Bin Fang, and Huaping Liu.
\newblock Reinforcement learning from imperfect demonstrations under soft expert guidance, 2019.

\bibitem[Kakade(2003)]{kakade2003onthesample}
Machandranath~Sham Kakade.
\newblock \emph{On the Sample Complexity of Reinforcement Learning}.
\newblock PhD thesis, 2003.

\bibitem[Kaufmann et~al.(2021)Kaufmann, M{\'e}nard, Darwiche~Domingues, Jonsson, Leurent, and Valko]{kaufmann2021adaptiveRFE}
Emilie Kaufmann, Pierre M{\'e}nard, Omar Darwiche~Domingues, Anders Jonsson, Edouard Leurent, and Michal Valko.
\newblock Adaptive reward-free exploration.
\newblock In \emph{International Conference on Algorithmic Learning Theory 32 (ALT)}, volume 132, pages 865--891, 2021.

\bibitem[Kim et~al.(2020)Kim, Gu, Song, Zhao, and Ermon]{kim2020domain}
Kuno Kim, Yihong Gu, Jiaming Song, Shengjia Zhao, and Stefano Ermon.
\newblock Domain adaptive imitation learning.
\newblock In \emph{International Conference on Machine Learning 37 (ICML)}, volume 119, pages 5286--5295, 2020.

\bibitem[Klein et~al.(2012)Klein, Geist, Piot, and Pietquin]{klein2012IRLclassification}
Edouard Klein, Matthieu Geist, Bilal Piot, and Olivier Pietquin.
\newblock Inverse reinforcement learning through structured classification.
\newblock In \emph{Advances in Neural Information Processing Systems 25 (NeurIPS)}, 2012.

\bibitem[Komanduru and Honorio(2019)]{komanduru2019correctness}
Abi Komanduru and Jean Honorio.
\newblock On the correctness and sample complexity of inverse reinforcement learning.
\newblock In \emph{Advances in Neural Information Processing Systems 32 (NeurIPS)}, 2019.

\bibitem[Komanduru and Honorio(2021)]{komanduru2021lowerbound}
Abi Komanduru and Jean Honorio.
\newblock A lower bound for the sample complexity of inverse reinforcement learning.
\newblock In \emph{International Conference on Machine Learning 38 (ICML)}, volume 139, pages 5676--5685, 2021.

\bibitem[Kurenkov et~al.(2019)Kurenkov, Mandlekar, Martin-Martin, Savarese, and Garg]{kurenkov2019acteachbayesianactorcriticmethod}
Andrey Kurenkov, Ajay Mandlekar, Roberto Martin-Martin, Silvio Savarese, and Animesh Garg.
\newblock Ac-teach: A bayesian actor-critic method for policy learning with an ensemble of suboptimal teachers, 2019.

\bibitem[Lazzati and Metelli(2024)]{lazzati2024learningutilitiesdemonstrationsmarkov}
Filippo Lazzati and Alberto~Maria Metelli.
\newblock Learning utilities from demonstrations in markov decision processes, 2024.

\bibitem[Lazzati et~al.(2024{\natexlab{a}})Lazzati, Mutti, and Metelli]{lazzati2024offline}
Filippo Lazzati, Mirco Mutti, and Alberto~Maria Metelli.
\newblock Offline inverse rl: New solution concepts and provably efficient algorithms.
\newblock In \emph{International Conference on Machine Learning 41 (ICML)}, 2024{\natexlab{a}}.

\bibitem[Lazzati et~al.(2024{\natexlab{b}})Lazzati, Mutti, and Metelli]{lazzati2024scaleinverserllarge}
Filippo Lazzati, Mirco Mutti, and Alberto~Maria Metelli.
\newblock How does inverse rl scale to large state spaces? a provably efficient approach, 2024{\natexlab{b}}.

\bibitem[Lindner et~al.(2022)Lindner, Krause, and Ramponi]{lindner2022active}
David Lindner, Andreas Krause, and Giorgia Ramponi.
\newblock Active exploration for inverse reinforcement learning.
\newblock In \emph{Advances in Neural Information Processing Systems 35 (NeurIPS)}, pages 5843--5853, 2022.

\bibitem[Liu et~al.(2023)Liu, Yoneda, Wang, Walter, and Chen]{liu2023activepolicyimprovementmultiple}
Xuefeng Liu, Takuma Yoneda, Chaoqi Wang, Matthew~R. Walter, and Yuxin Chen.
\newblock Active policy improvement from multiple black-box oracles, 2023.

\bibitem[Liu et~al.(2018)Liu, Gupta, Abbeel, and Levine]{liu2018imitation}
YuXuan Liu, Abhishek Gupta, Pieter Abbeel, and Sergey Levine.
\newblock Imitation from observation: Learning to imitate behaviors from raw video via context translation.
\newblock In \emph{IEEE International Conference on Robotics and Automation (ICRA)}, pages 1118--1125, 2018.

\bibitem[Liu et~al.(2022)Liu, Zhang, Fu, Yang, and Wang]{liu2022learning}
Zhihan Liu, Yufeng Zhang, Zuyue Fu, Zhuoran Yang, and Zhaoran Wang.
\newblock Learning from demonstration: Provably efficient adversarial policy imitation with linear function approximation.
\newblock In \emph{International Conference on Machine Learning 39 (ICML)}, volume 162, pages 14094--14138, 2022.

\bibitem[Lopes et~al.(2009)Lopes, Melo, and Montesano]{lopes2009active}
Manuel Lopes, Francisco Melo, and Luis Montesano.
\newblock Active learning for reward estimation in inverse reinforcement learning.
\newblock In \emph{Machine Learning and Knowledge Discovery in Databases (ECML PKDD)}, pages 31--46, 2009.

\bibitem[Menard et~al.(2021)Menard, Domingues, Jonsson, Kaufmann, Leurent, and Valko]{menard2021fast}
Pierre Menard, Omar~Darwiche Domingues, Anders Jonsson, Emilie Kaufmann, Edouard Leurent, and Michal Valko.
\newblock Fast active learning for pure exploration in reinforcement learning.
\newblock In \emph{International Conference on Machine Learning 38 (ICML)}, volume 139, pages 7599--7608, 2021.

\bibitem[Metelli et~al.(2021)Metelli, Ramponi, Concetti, and Restelli]{metelli2021provably}
Alberto~Maria Metelli, Giorgia Ramponi, Alessandro Concetti, and Marcello Restelli.
\newblock Provably efficient learning of transferable rewards.
\newblock In \emph{International Conference on Machine Learning 38 (ICML)}, volume 139, pages 7665--7676, 2021.

\bibitem[Metelli et~al.(2023)Metelli, Lazzati, and Restelli]{metelli2023towards}
Alberto~Maria Metelli, Filippo Lazzati, and Marcello Restelli.
\newblock Towards theoretical understanding of inverse reinforcement learning.
\newblock In \emph{International Conference on Machine Learning 40 (ICML)}, volume 202, pages 24555--24591, 2023.

\bibitem[Michini et~al.(2013)Michini, Cutler, and How]{michini2013scalable}
Bernard Michini, Mark Cutler, and Jonathan~P. How.
\newblock Scalable reward learning from demonstration.
\newblock In \emph{IEEE International Conference on Robotics and Automation (ICRA)}, pages 303--308, 2013.

\bibitem[Neu and Szepesv\'{a}ri(2007)]{neu2007apprenticeship}
Gergely Neu and Csaba Szepesv\'{a}ri.
\newblock Apprenticeship learning using inverse reinforcement learning and gradient methods.
\newblock In \emph{Conference on Uncertainty in Artificial Intelligence 23 (UAI 2007)}, pages 295--302, 2007.

\bibitem[Ng and Russell(2000)]{ng2000algorithms}
Andrew~Y. Ng and Stuart~J. Russell.
\newblock Algorithms for inverse reinforcement learning.
\newblock In \emph{International Conference on Machine Learning 17 (ICML)}, pages 663--670, 2000.

\bibitem[Ornstein(1969)]{ornstein1969existence}
Donald Ornstein.
\newblock On the existence of stationary optimal strategies.
\newblock \emph{Proceedings of the American Mathematical Society}, 20\penalty0 (2):\penalty0 563--569, 1969.

\bibitem[Poiani et~al.(2024)Poiani, Curti, Metelli, and Restelli]{poiani2024inverse}
Riccardo Poiani, Gabriele Curti, Alberto~Maria Metelli, and Marcello Restelli.
\newblock Inverse reinforcement learning with sub-optimal experts, 2024.

\bibitem[Puterman(1994)]{puterman1994markov}
Martin~Lee Puterman.
\newblock \emph{{M}arkov Decision Processes: Discrete Stochastic Dynamic Programming}.
\newblock John Wiley \& Sons, Inc., 1994.

\bibitem[Rajaraman et~al.(2020)Rajaraman, Yang, Jiao, and Ramchandran]{Rajaraman2020TowardTF}
Nived Rajaraman, Lin Yang, Jiantao Jiao, and Kannan Ramchandran.
\newblock Toward the fundamental limits of imitation learning.
\newblock In \emph{Advances in Neural Information Processing Systems 33 (NeurIPS)}, pages 2914--2924, 2020.

\bibitem[Rajaraman et~al.(2021)Rajaraman, Han, Yang, Liu, Jiao, and Ramchandran]{rajaraman2021value}
Nived Rajaraman, Yanjun Han, Lin Yang, Jingbo Liu, Jiantao Jiao, and Kannan Ramchandran.
\newblock On the value of interaction and function approximation in imitation learning.
\newblock In \emph{Advances in Neural Information Processing Systems 34 (NeurIPS)}, volume~34, pages 1325--1336, 2021.

\bibitem[Ramachandran and Amir(2007)]{Ramachandran2007birl}
Deepak Ramachandran and Eyal Amir.
\newblock Bayesian inverse reinforcement learning.
\newblock In \emph{International Joint Conference on Artifical Intelligence 20 (IJCAI)}, pages 2586--2591, 2007.

\bibitem[Rolland et~al.(2022)Rolland, Viano, Sch\"{u}rhoff, Nikolov, and Cevher]{rolland2022identifiabilitymultipleexperts}
Paul Rolland, Luca Viano, Norman Sch\"{u}rhoff, Boris Nikolov, and Volkan Cevher.
\newblock Identifiability and generalizability from multiple experts in inverse reinforcement learning.
\newblock In \emph{Advances in Neural Information Processing Systems 35 (NeurIPS)}, pages 550--564, 2022.

\bibitem[Russell(1998)]{russell1998learning}
Stuart Russell.
\newblock Learning agents for uncertain environments (extended abstract).
\newblock In \emph{Conference on Computational Learning Theory 11 (COLT)}, pages 101--103, 1998.

\bibitem[Shani et~al.(2022)Shani, Zahavy, and Mannor]{Shani2021OnlineAL}
Lior Shani, Tom Zahavy, and Shie Mannor.
\newblock Online apprenticeship learning.
\newblock In \emph{AAAI Conference on Artificial Intelligence 36 (AAAI)}, pages 8240--8248, 2022.

\bibitem[Skalse et~al.(2023)Skalse, Farrugia-Roberts, Russell, Abate, and Gleave]{skalse2023invariance}
Joar Max~Viktor Skalse, Matthew Farrugia-Roberts, Stuart Russell, Alessandro Abate, and Adam Gleave.
\newblock Invariance in policy optimisation and partial identifiability in reward learning.
\newblock In \emph{International Conference on Machine Learning 40 (ICML)}, volume 202, pages 32033--32058, 2023.

\bibitem[Sun et~al.(2019)Sun, Vemula, Boots, and Bagnell]{sun2019provably}
Wen Sun, Anirudh Vemula, Byron Boots, and Drew Bagnell.
\newblock Provably efficient imitation learning from observation alone.
\newblock In \emph{International Conference on Machine Learning 36 (ICML)}, volume~97, pages 6036--6045, 2019.

\bibitem[Swamy et~al.(2022)Swamy, Rajaraman, Peng, Choudhury, Bagnell, Wu, Jiao, and Ramchandran]{swamy2022minimax}
Gokul Swamy, Nived Rajaraman, Matt Peng, Sanjiban Choudhury, J.~Bagnell, Steven~Z. Wu, Jiantao Jiao, and Kannan Ramchandran.
\newblock Minimax optimal online imitation learning via replay estimation.
\newblock In \emph{Advances in Neural Information Processing Systems 35 (NeruIPS)}, volume~35, pages 7077--7088, 2022.

\bibitem[Syed and Schapire(2007)]{syed2007game}
Umar Syed and Robert~E Schapire.
\newblock A game-theoretic approach to apprenticeship learning.
\newblock In \emph{Advances in Neural Information Processing Systems 20 (NeurIPS)}, 2007.

\bibitem[Wagenmaker et~al.(2022)Wagenmaker, Chen, Simchowitz, Du, and Jamieson]{wagenmaker2022noharder}
Andrew~J Wagenmaker, Yifang Chen, Max Simchowitz, Simon Du, and Kevin Jamieson.
\newblock Reward-free {RL} is no harder than reward-aware {RL} in linear {M}arkov decision processes.
\newblock In \emph{International Conference on Machine Learning 39 (ICML)}, volume 162, pages 22430--22456, 2022.

\bibitem[Wang et~al.(2020{\natexlab{a}})Wang, Foster, and Kakade]{Wang2020WhatAT}
Ruosong Wang, Dean~Phillips Foster, and Sham~M. Kakade.
\newblock What are the statistical limits of offline rl with linear function approximation?
\newblock In \emph{International Conference on Learning Representations 8 (ICLR)}, 2020{\natexlab{a}}.

\bibitem[Wang et~al.(2020{\natexlab{b}})Wang, Salakhutdinov, and Yang]{wang2020general}
Ruosong Wang, Ruslan Salakhutdinov, and Lin~F. Yang.
\newblock Reinforcement learning with general value function approximation: provably efficient approach via bounded eluder dimension.
\newblock In \emph{Advances in Neural Information Processing Systems 34 (NeurIPS)}, pages 6123--6135, 2020{\natexlab{b}}.

\bibitem[Wirth et~al.(2017)Wirth, Akrour, Neumann, and F{{\"u}}rnkranz]{wirth2017surveyPbRL}
Christian Wirth, Riad Akrour, Gerhard Neumann, and Johannes F{{\"u}}rnkranz.
\newblock A survey of preference-based reinforcement learning methods.
\newblock \emph{Journal of Machine Learning Research}, 18:\penalty0 1--46, 2017.

\bibitem[Xie et~al.(2021)Xie, Jiang, Wang, Xiong, and Bai]{xie2021bridging}
Tengyang Xie, Nan Jiang, Huan Wang, Caiming Xiong, and Yu~Bai.
\newblock Policy {F}inetuning: Bridging sample-efficient offline and online reinforcement learning.
\newblock In \emph{Advances in Neural Information Processing Systems 34 (NeurIPS)}, pages 27395--27407, 2021.

\bibitem[Xu et~al.(2023)Xu, Li, Yu, and Luo]{Xu2023ProvablyEA}
Tian Xu, Ziniu Li, Yang Yu, and Zhimin Luo.
\newblock Provably efficient adversarial imitation learning with unknown transitions.
\newblock In \emph{Conference on Uncertainty in Artificial Intelligence 39 (UAI)}, volume 216, pages 2367--2378, 2023.

\bibitem[Yang and Wang(2019)]{yang2019sampleoptimal}
Lin Yang and Mengdi Wang.
\newblock Sample-optimal parametric q-learning using linearly additive features.
\newblock In \emph{International Conference on Machine Learning 36 (ICML)}, volume~97, pages 6995--7004, 2019.

\bibitem[Yue et~al.(2024)Yue, Li, and Liu]{yue2024provablyefficientexplorationinverse}
Bo~Yue, Jian Li, and Guiliang Liu.
\newblock Provably efficient exploration in inverse constrained reinforcement learning, 2024.

\bibitem[Zhao et~al.(2024)Zhao, Wang, and Bai]{zhao2023inverse}
Lei Zhao, Mengdi Wang, and Yu~Bai.
\newblock Is inverse reinforcement learning harder than standard reinforcement learning?
\newblock In \emph{International Conference on Machine Learning 41 (ICML)}, 2024.

\bibitem[Zhou et~al.(2021)Zhou, Gu, and Szepesvari]{zhou2021nearly}
Dongruo Zhou, Quanquan Gu, and Csaba Szepesvari.
\newblock Nearly minimax optimal reinforcement learning for linear mixture markov decision processes.
\newblock In \emph{Conference on Learning Theory 34 (COLT)}, pages 4532--4576, 2021.

\bibitem[Ziebart et~al.(2008)Ziebart, Maas, Bagnell, and Dey]{ziebart2008maximum}
Brian~D. Ziebart, Andrew~L. Maas, J.~Andrew Bagnell, and Anind~K. Dey.
\newblock Maximum entropy inverse reinforcement learning.
\newblock In \emph{AAAI Conference on Artificial Intelligence 23 (AAAI)}, pages 1433--1438, 2008.

\bibitem[Ziebart et~al.(2010)Ziebart, Bagnell, and Dey]{ziebart2010MCE}
Brian~D. Ziebart, J.~Andrew Bagnell, and Anind~K. Dey.
\newblock Modeling interaction via the principle of maximum causal entropy.
\newblock In \emph{International Conference on Machine Learning 27 (ICML)}, pages 1255--1262, 2010.

\end{thebibliography}

\end{document}